%% file: main.tex
\documentclass{article} 
\usepackage{iclr2026_conference,times}

\input{math_commands.tex}

\usepackage{hyperref}
\usepackage{url}

\usepackage{booktabs}       
\usepackage{amsfonts}       
\usepackage{nicefrac}       
\usepackage{microtype}      
\usepackage{xcolor}         

\usepackage{tcolorbox}  
\tcbuselibrary{skins}   

\usepackage{graphicx}
\usepackage{booktabs}
\usepackage{multirow}
\usepackage{colortbl}
\usepackage{amsmath}
\usepackage{natbib}
\usepackage{amsthm} 

\usepackage{adjustbox}
\usepackage{mathtools}
\usepackage{arydshln}
\usepackage{enumitem}
\usepackage{wrapfig}
\usepackage{makecell}
\usepackage{array}
\usepackage{amssymb}
\usepackage{subcaption}
\usepackage{xspace}
\newcommand{\ours}{SoLoPO\xspace}

\newtheorem{theorem}{Theorem}
\newtheorem{proposition}{Proposition}
\newtheorem{lemma}{Lemma}

\newtheorem{assumption}{Assumption}

\definecolor{Gray}{gray}{0.9}

\setcitestyle{numbers}
\setcitestyle{square}
\setcitestyle{comma}

\newcommand{\scriptscalefactor}{0.45}
\renewcommand{\bibfont}{\small}

\title{SoLoPO: Unlocking Long-Context Capabilities in
LLMs via Short-to-Long Preference Optimization}
\author{
Huashan Sun$^1$$^{*}$ \quad Shengyi Liao$^1$\thanks{~~Equal contribution.} \quad Yansen Han \quad Yu Bai \quad Yang Gao$^\dagger$\\ \textbf{Cheng Fu$^1$} \quad \textbf{Weizhou Shen}$^1$\quad \textbf{Fanqi Wan}$^1$ \quad \textbf{Ming Yan}$^1$\thanks{~~Corresponding authors} \quad \textbf{Ji Zhang}$^1$\quad \textbf{Fei Huang}$^1$
\\
$^1$ Tongyi Lab, Alibaba Group
\\
\texttt{\{liaoshengyi.lsy,ym119608\}@alibaba-inc.com}\\
\texttt{hanyansen@gmail.com}\quad\texttt{\{hssun,yubai,gyang\}@bit.edu.cn}
}

\iclrfinalcopy
\begin{document}

\maketitle

\begin{abstract}
Despite advances in pretraining with extended context sizes, large language models (LLMs) still face challenges in effectively utilizing real-world long-context information, primarily due to insufficient long-context alignment caused by data quality issues, training inefficiencies, and the lack of well-designed optimization objectives. To address these limitations, we propose a framework named \textbf{S}h\textbf{o}rt-to-\textbf{Lo}ng \textbf{P}reference \textbf{O}ptimization (\textbf{\ours}), decoupling long-context preference optimization (PO) into two components: short-context PO and short-to-long reward alignment (SoLo-RA), supported by both theoretical and empirical evidence. Specifically, short-context PO leverages preference pairs sampled from short contexts to enhance the model's contextual knowledge utilization ability. Meanwhile, SoLo-RA explicitly encourages reward score consistency for the responses when conditioned on both short and long contexts that contain identical task-relevant information. This facilitates transferring the model's ability to handle short contexts into long-context scenarios. \ours is compatible with mainstream preference optimization algorithms, while substantially improving the efficiency of data construction and training processes. Experimental results show that \ours enhances all these algorithms with respect to stronger length and domain generalization abilities across various long-context benchmarks, while achieving notable improvements in both computational and memory efficiency\footnote{Code and data resources are available at \url{https://github.com/shs910/SoLoPO}}.
\end{abstract}

\input{Introduction}

\input{Method}
\input{Experiments}

\input{Related}

\input{Conclusion}
\input{Acknowledgments}
\input{else}
\bibliography{iclr2026_conference}
\bibliographystyle{iclr2026_conference}

\appendix
\input{Appendix}

\end{document}

%% file: math_commands.tex
\usepackage{amsmath,amsfonts,bm}

\def\eqref#1{equation~\ref{#1}}

\def\1{\bm{1}}

\DeclareMathAlphabet{\mathsfit}{\encodingdefault}{\sfdefault}{m}{sl}
\SetMathAlphabet{\mathsfit}{bold}{\encodingdefault}{\sfdefault}{bx}{n}



%% file: Introduction.tex
\section{Introduction}
Long-text modeling is a cornerstone capability of large language models (LLMs)~\citep{Yang2024Qwen25TR,LLama_3,Yang2024MindLLMLL,li-etal-2024-fundamental}.
While the input context size of LLMs has increased dramatically~\citep{Survey_Liu2025ACS,Survey_Dong2023ASO}, studies show that they can effectively utilize only 10–20\% of this capacity primarily due to insufficient long-context alignment~\citep{BABILong,hsieh2024ruler,loong,chen2025longpo}, leaving their potential in long-context scenarios largely untapped.

To address this issue, data augmentation methods~\citep{MDCure,LongReward,bai-etal-2024-longalign,Zhu2025GeneralizingFromShort2Long,LongFaith} leverage advanced LLMs to generate long-dependency instruction-following data for supervised fine-tuning (SFT) and preference optimization (PO). However, as text length increases, these methods suffer from declining reliability and efficiency. Moreover, directly applying short-context training strategies to long-context scenarios may overlook the inherent discrepancies between the two settings, yielding suboptimal performance~\citep{Survey_Dong2023ASO,BART_Lewis2019}. A different approach improves long-text alignment via training objective optimization. \citet{LongCEloss} propose LongCE, which identifies tokens critical for long-text modeling and assign them higher loss weights during SFT. However, this approach incurs extra computation due to multiple forward passes to identify salient tokens. LongPO~\citep{chen2025longpo} leverages responses generated with short contexts as positive examples in long-context direct preference optimization (DPO)~\citep{Rafailov2023DirectPO}. Additionally, a short-to-long constraint is introduced, which optimizes the DPO objective by replacing $\pi_{ref}(y\mid x_{long})$ with $\pi_{ref}(y\mid x_{short})$ to mitigate performance degradation on short-context tasks. However, LongPO is not generalizable to other PO algorithms. Accordingly, long-context alignment poses three primary challenges: (1) difficulties in data construction, (2) inefficient training procedures, and (3) the absence of a suitable optimization objective.

\begin{figure}[t]
    \centering
    \begin{subfigure}[b]{0.6\textwidth}
        \centering
        \includegraphics[height=3.2cm]{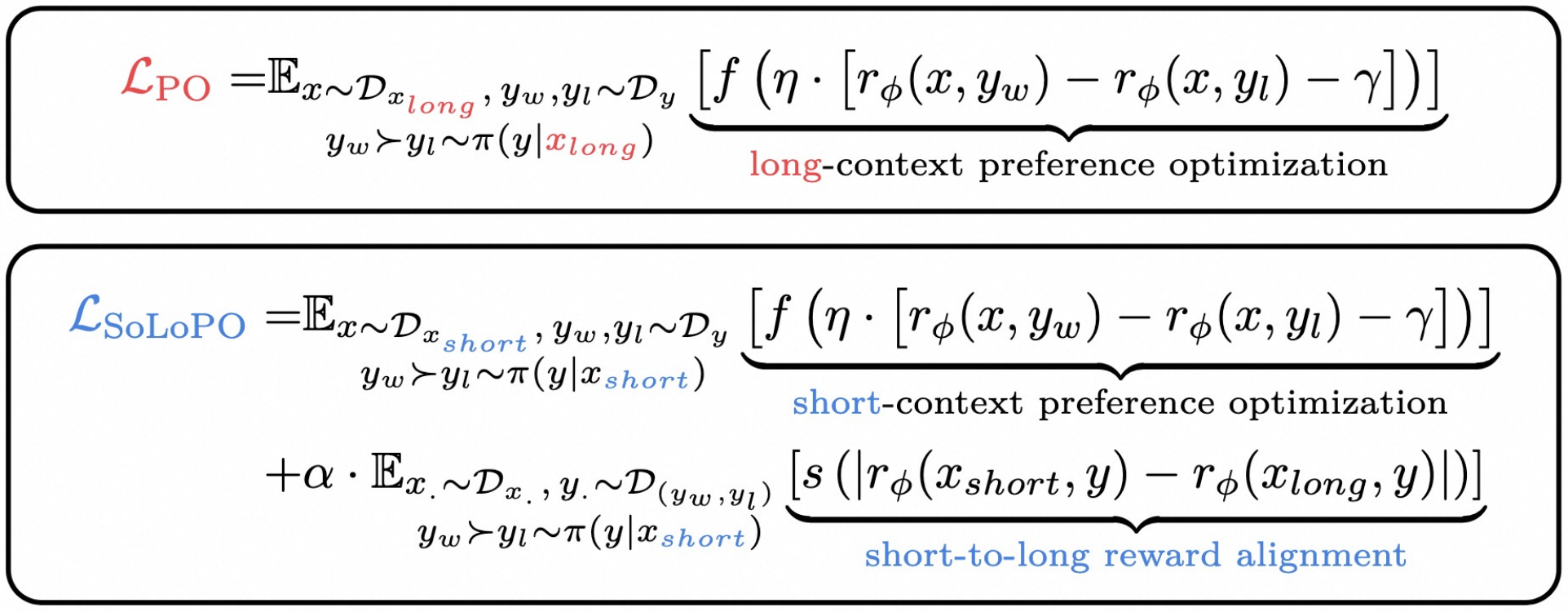}
        \caption{Comparison of objectives between original PO and \ours. $x_{long}$ denotes the original long-context input, and $x_{short}$ denotes the compressed short-context input preserving key task information.}
        \label{fig:obj_compa}
    \end{subfigure}
    \hfill
    \begin{subfigure}[b]{0.39\textwidth}
        \centering
        \includegraphics[height=3.2cm]{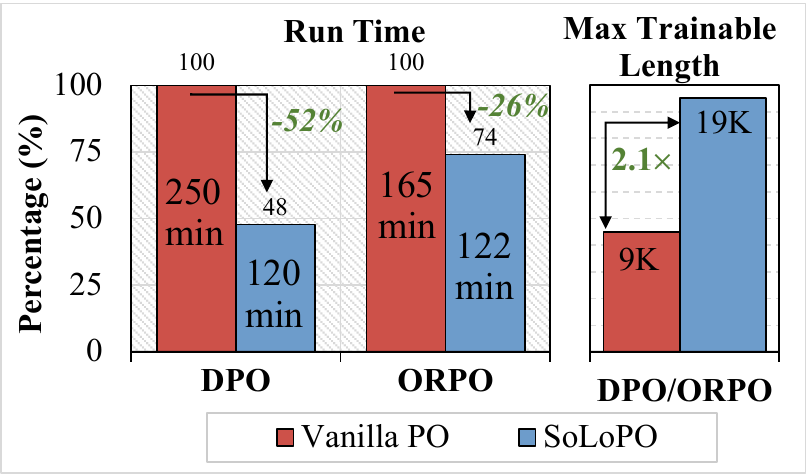}
        \caption{Efficiency of Vanilla PO vs. \ours. SoLoPO greatly cuts the original PO run time and doubles the max trainable length.}
        \label{fig:efficiency}
    \end{subfigure}

    \vspace{0.1cm}
    \begin{subfigure}[b]{1.0\textwidth}
        \centering
        \includegraphics[width=\textwidth]{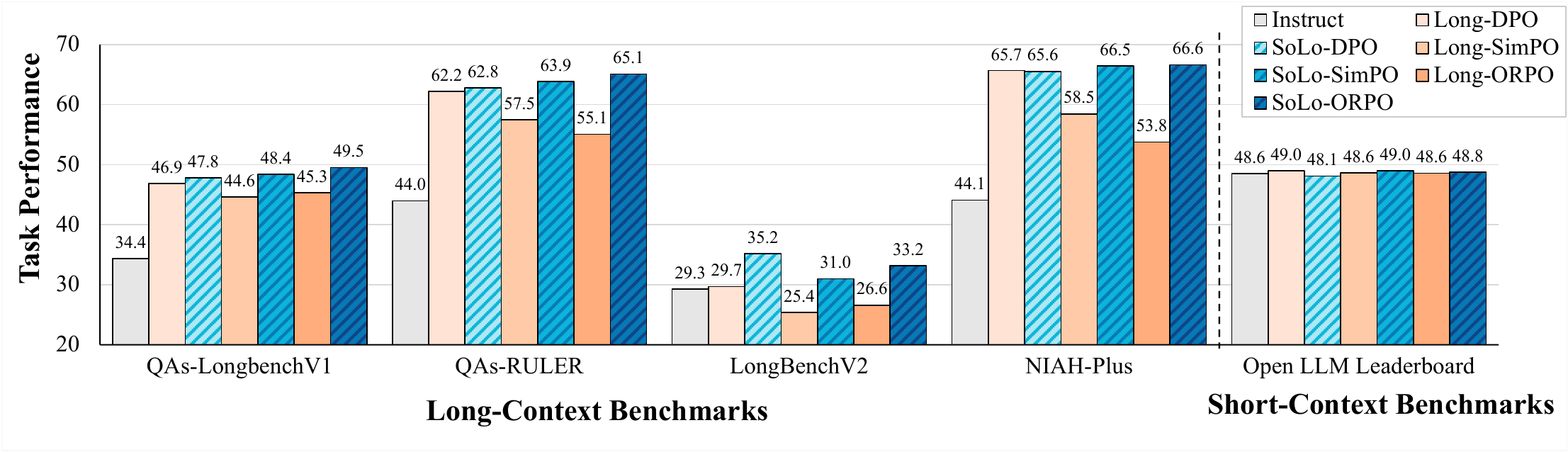}
        \caption{Performance comparison of Qwen2.5-7B-Instruct~\citep{Yang2024Qwen25TR} trained with original PO versus \ours across various long-context and short-context benchmarks. Long-PO refers to original preference optimization on preference pairs sampled from long texts, while SoLo-PO denotes preference optimization within our \ours framework using preference data derived from compressed short texts.}
        \label{fig:overall_perfomance}
    \end{subfigure}
    \caption{Original PO vs. \ours. \textbf{(a)} \ours decouples long-context PO into two components: short-context PO and short-to-long reward alignment, reducing the complexity of preference data construction and minimizing long-text processing during training. \textbf{(b)} Under identical configurations, \ours exhibits superior training efficiency compared to vanilla methods. \textbf{(c)} \ours outperforms the original PO across various long-context benchmarks while maintaining short-context ability.}
    \label{fig:overall}
\end{figure}

In this paper, we introduce \textbf{S}h\textbf{o}rt-to-\textbf{Lo}ng \textbf{P}reference \textbf{O}ptimization (\ours), a general, simple yet effective framework to transfer the powerful understanding ability of short contexts (hereafter termed short-context capability) of LLMs to long-context scenarios. As illustrated in Figure~\ref{fig:obj_compa}, we first theoretically demonstrate that long-context PO~\citep{logo} can be decoupled into two components: short-context PO and short-to-long reward alignment (SoLo-RA). Intuitively, \ours enhances the model's contextual knowledge utilization ability via short-context PO,  while SoLo-RA explicitly encourages the model to align reward scores between outputs conditioned on long and short contexts containing identical task-relevant information, thereby improving its long-context ability. \ours offers three key advantages over existing methods: (1) sampling preference pairs from compressed shortened long contexts improves data quality and construction efficiency; (2) applying SoLo-RA only to the chosen responses reduces the burden of long-text processing during training, leading to more efficient optimization; (3) the optimization objective accounts for connections between short and long contexts, better favoring long-context alignment.

We apply \ours to different PO algorithms, including DPO, SimPO~\citep{meng2024simpo}, and ORPO~\citep{hong-etal-2024-orpo}. As shown in Figure~\ref{fig:overall_perfomance}, benefiting from \ours, Qwen2.5-7B-Instruct trained solely on MuSiQue~\citep{musique-dataset} achieves better performance across various domains and lengths on QA tasks from LongBenchV1~\citep{bai-etal-2024-longbench} and RULER~\citep{hsieh2024ruler}, demonstrating stronger generalization than models trained with vanilla methods. Results on LongBenchV2~\citep{bai2025longbenchv2deeperunderstanding} and the Open-LLM-Leaderboard~\citep{open-llm-leaderboard-v2} further indicate that \ours exhibits promising potential in handling contexts beyond pre-training window length, without compromising short-context capabilities. Further analysis on NIAH-Plus~\citep{zhao-etal-2024-longagent} indicates that \ours’s decoupled approach with explicit SoLo-RA notably improves the contextual knowledge localization capability of LLMs. Moreover, SoLoPO significantly enhances efficiency, enabling $2.1\times$ longer trainable length under ZeRO stage 3~\citep{rajbhandari2020zeromemoryoptimizationstraining} with offloading while cutting run time by 52\% and 26\% for DPO and ORPO at $9K$ length, respectively (Figure~\ref{fig:efficiency}).


Our main contributions can be summarized as follows:
\begin{itemize}[leftmargin=0.5cm]
\vspace{-5pt}
    \item We theoretically show that long-context PO can be decomposed into short-context PO and short-to-long reward alignment, providing new insights for long-context alignment.
    \item We propose \ours, a general framework for long-context PO, which transfers the model’s short-context ability to long-text scenarios while significantly improving training efficiency.
    \item We integrate mainstream preference optimization algorithms into \ours and empirically demonstrate LLMs can have much better performance within this framework.
\end{itemize}

%% file: Method.tex
\section{SoLoPO: Short-to-Long Preference Optimization}
\label{sec:method}
In this section, we first introduce the background of preference optimization (PO), including DPO and the unified framework, generalized preference optimization (GPO)~\citep{gpo_paper} (\S~\ref{sec:method_prelimi}). Then, by theoretically analyzing long-context preference modeling, we show that long-context PO can be decoupled into short-context PO and short-to-long reward alignment (\S~\ref{sec:method_loss_analysis}). Based on this insight, we propose SoLoPO and apply it to various preference optimization algorithms (\S~\ref{sec:method_app}).
\subsection{Preliminaries}
\label{sec:method_prelimi}
\paragraph{Reinforcement learning from human feedback (RHLF).}
RHLF~\citep{rlhf} aligns LLMs with human preferences through a two-stage process, further enhancing the model’s capabilities. This involves training a reward model $r_\phi$ that captures human preferences, followed by regularized policy optimization to align the LLM with the learned reward model, more formally as below
{\small
\begin{equation} \label{rl_obj}
    \max_{\pi_\theta}\mathbb{E}_{x\sim\mathcal{D}, y\sim\pi_\theta(y|x)} \left[r_\phi(x,y)\right] - \beta\mathbb{D}_{KL}\left[\pi_{\theta}(y|x) || \pi_{ref}(y|x)\right],
\end{equation}
}
where  $\pi_\theta$  is the policy model,  $\pi_{ref}$ is the reference policy, typically initialized with the SFT model.
\paragraph{Preference optimization (PO).}

Without explicit reward modeling, DPO~\citep{Rafailov2023DirectPO} reparameterizes the reward function using the optimal policy model and directly models the preference distribution by incorporating the Bradley-Terry ranking loss~\citep{BT_ranking_loss}, enabling a single-stage preference alignment:
{\small
\begin{equation}
    \label{dpo_loss}
    \mathcal{L}_{DPO}(\pi_{\theta};\pi_{ref}) = - \mathbb{E}_{(x, y_w, y_l) \sim \mathcal{D}}\left[\log \sigma\left(\beta\log \frac{\pi_\theta(y_w|x)}{\pi_{ref}(y_w|x)} - \beta\log \frac{\pi_\theta(y_l|x)}{\pi_{ref}(y_l|x)}\right)\right],
\end{equation}
}
here, $(y_w, y_l)$ is a preference pair. Furthermore, \citet{gpo_paper} propose GPO, a unified framework for preference optimization, which allows us to parameterize the optimization objective using a convex function $f(\cdot)$ and hyperparameters $\eta$ and $\gamma$:
{\small
\begin{equation}
    \label{gpo_loss}
    \mathcal{L}(r_\phi, \mathcal{D}) = \mathbb{E}_{(x, y_w, y_l) \sim \mathcal{D}}\left[f\left(\eta \cdot \left(r_\phi(x, y_w) - r_\phi(x, y_l) - \gamma\right)\right)\right].
\end{equation}
}
\subsection{Theoretical Analysis of Long-Context Preference Modeling}
\label{sec:method_loss_analysis}

Recall that a key challenge in long-text alignment lies in the inefficiency of data construction and training. Thus, \textit{can we represent long-context PO via short-context PO, thereby making data collection and training more tractable?} We analyze the upper bound of general long-context PO loss, demonstrating the viability of this approach based on the redundancy hypothesis.

\paragraph{Redundancy hypothesis and compression rate.} Redundancy, pervasive in human language~\citep{Wit2013WhatIL,Sloane1951PredictionAE}, while potentially aiding human comprehension, may adversely affect LLMs~\citep{li-etal-2023-compressing,pan-etal-2024-llmlingua}. Particularly in task-aware scenarios~\citep{huang-etal-2024-fewer,li2024snapkv,xu2024recomp}, for a long context $c_{long}$ and a task instruction $I$, the model only needs to focus on relevant key content $c_{rel}$ while ignoring irrelevant content $c_{irr}$. Therefore, we can use \textit{compression rate}~\citep{compression_rate} denoted as $\rho$, as a unified lens to observe long-context tasks, which refers to the information ratio between $c_{rel}$ and $c_{long}$. Most long-context tasks, such as question answering and information extraction, require only task-relevant excerpts from the source text~\citep{bai-etal-2024-citrus,huang-etal-2024-fewer}, yielding $\rho<100\%$. As a special case, long-context translation~\citep{long_doc_mt_survey,long_doc_mt} has a compression rate $\rho=100\%$. 


\paragraph{Problem setting.} Regarding long-context scenarios, we use $x_{long}\coloneqq[c_{long}; I]$ to represent the input comprising the original long context $c_{long}$ and task instruction $I$. Based on the redundancy hypothesis, $c_{long}$ can be compressed, given the task instruction $I$, into a context $c_{rel}$ that preserves all task-critical information. We denote by {$x_{short}\coloneqq x_{rel}\coloneqq[c_{rel}; I]$} the concatenation of $c_{rel}$ and $I$. For tasks with $\rho<100\%$, $x_{short}$ is typically shorter than $x_{long}$; for $\rho=100\%$, they are identical.
Given a preference dataset $\mathcal{D}_{(x_{long}, y_w, y_l)}$, the objective of long-context PO is to model the preference relation $p(y_w \succ  y_l\mid x_{long})$ by minimizing the preference modeling loss, as defined in Eq. (\ref{gpo_loss}).   

\paragraph{The upper bound of long-context general preference modeling loss.} 
To simplify notation in the following analysis, we define the preference loss in Eq. (\ref{gpo_loss}) for any given tuple $(x_1,x_2,y_1,y_2)$ as:
{\begin{equation} 
\label{small_op_l}
    l_{\eta, \gamma}(x_1, x_2; y_1, y_2) = f(\eta \cdot[r_\phi(x_1, y_1) - r_\phi(x_2, y_2) - \gamma] ).
\end{equation}}
The expectation of Equation (\ref{small_op_l}) can then be expressed as:
{\begin{equation}
    \mathcal{L}_{\eta, \gamma}(\mathcal{D}_{x_1}, \mathcal{D}_{x_2}; \mathcal{D}_{y_1},\mathcal{D}_{y_2}) = \mathbb{E}_{\substack{x_1, x_2 \sim \mathcal{D}_{x_1}, {D}_{x_2} ;y_1, y_2 \sim \mathcal{D}_{y_1}, \mathcal{D}_{y_2}} } [l_{\eta, \gamma}(x_1, x_2; y_1, y_2)].
\end{equation}}

\begin{assumption}[Discrimination of preference order]\label{assumption1}
    Based on the redundancy hypothesis, $x_{long}\coloneqq [(c_{rel}, c_{irr}); I]$ contains more task-irrelevant information compared to $x_{short}\coloneqq[c_{rel}; I]$, which may hinder LLMs' task performance~\citep{li-etal-2023-compressing,pan-etal-2024-llmlingua}. Consequently, distinguishing the order between $y_w$ and $y_l$ given $x_{long}$ is more difficult than given $x_{short}$ ({refer to Appendix~\ref{sec:empirical_evidence_for_assumption_1} for experimental evidence}): 
    \begin{equation}
        p(y_w \succ y_l \mid x_{long}) \leq p(y_w \succ y_l \mid x_{short})
    \end{equation}
    When $x_{short}$ and $x_{long}$ are identical, the equality holds, giving a compression rate $\rho$ of 100\%.
\end{assumption}
Building upon the previous preparations, we establish the theoretical upper bound for the optimization objective over long-context data in Theorem \ref{theorem_relation}. This bound provides formal guarantees that optimizing on short-context data while maintaining robust long-context performance is theoretically feasible. 

\begin{theorem}[Relation between long-context and short-context preference optimization losses]
\label{theorem_relation}
Under assumption~\ref{assumption1}, suppose $f$ is a convex function and satisfies $f(x+\gamma) + f(-x+\gamma) \leq s(|x|)$ for some function $s(\cdot)$ and non-negative constant $\gamma, \eta$. Then the following inequality holds:
\begin{equation}\label{ineq_diff_bound_general}
    \mathcal{L}_{\eta,\gamma}(x_{long})
    \leq \frac{1}{3} [\mathcal{L}_{3\eta, \frac{\gamma}{3}}(x_{short})+\mathbb{E}_{\substack{x_\cdot \sim \mathcal{D}_{x_{\cdot}}; y \sim \mathcal{D}_{y}} } s(|3\eta \cdot (r_\phi(x_{short}, y) - r_\phi(x_{long}, y))|)]
\end{equation}
where $\mathcal{L}_{\eta,\gamma}(x_{text})\coloneqq\mathcal{L}_{\eta,\gamma}(\mathcal{D}_{x_{text}},\mathcal{D}_{x_{text}}; \mathcal{D}_{y_{w} \succ y_{l}|x_{text}},\mathcal{D}_{y_{w} \succ y_{l}|x_{text}})$
\end{theorem}

The complete derivation of Theorem~\ref{theorem_relation} is presented in Appendix~\ref{proof_of_theorem_1}, {where $s(\cdot)$ is introduced with the primary objective of providing a metric to quantify the distance between $r_\phi(x_{short}, y)$ and $r_\phi(x_{long}, y)$. Given that $s(\cdot)$ serves as an upper bound, a tighter instantiation is theoretically preferred; we provide empirical evidence for this claim in Appendix~\ref{sec:exp_for_s}.} Additionally, an extension of Theorem~\ref{theorem_relation} is provided in Appendix~\ref{sec:general_theorem_1}, which potentially holds promise for wider applicability.
\paragraph{Objective Function of Short-to-Long Preference Optimization (SoLoPO).}
Based on the theorem~\ref{theorem_relation}, we can define the general formula of the SoLoPO loss function:
{
\begin{align}
     \mathcal{L}_{SoLoPO} =\mathbb{E}_{\substack{x \sim \mathcal{D}_{x_{short}};y_w, y_l \sim \mathcal{D}_{y}\\y_w\succ y_l\sim \pi_\theta(y\mid x_{short})}}&\underbrace{ \left[ f\left(3\eta \cdot[r_\phi(x, y_w) - r_\phi(x, y_l) - \frac{\gamma}{3}]\right)\right]}_{\text{short-context preference optimization}}\label{short_po} \\
    + \alpha \cdot \mathbb{E}_{\substack{x_{\cdot} \sim \mathcal{D}_{x_\cdot}; y. \sim \mathcal{D}_{(y_w,y_l)}\\y_w\succ y_l\sim \pi_\theta(y\mid x_{short})}}& \underbrace{\left[s(3\eta \cdot|r_\phi(x_{short}, y) - r_\phi(x_{long}, y)|)\right]}_{\text{short-to-long reward alignment}}\label{solo_ra}.
\end{align}
}
Here, $\gamma$, $\eta$ and $f(\cdot)$ are specified by the original PO algorithm, and $s(\cdot)$ satisfies $f(x+\gamma) + f(-x + \gamma) \leq s(|x|)$. $\alpha$ is a hyperparameter balancing the two loss terms. Thus, we theoretically decouple long-context PO into short-context PO and short-to-long reward alignment. {Specifically, we present detailed derivations of the SoLoPO objective from Theorem~\ref{theorem_relation} for two common convergence functions, $f(x)=x^2$ and $f(x)=-\log \sigma(x)$, in Appendices~\ref{sec:solopo_x2} and~\ref{sec:solopo_logsigmiod},  respectively. Table~\ref{tab:SLA_variants} lists further examples of $f(\cdot)$ and their associated $s(\cdot)$.}
Moreover, the analysis in Section~\ref{sec:depth_analysis} provides experimental evidence supporting the validity of this decoupling. {When $\rho=100\%$, $x_{long}$ and $x_{short}$ are identical, rendering SoLoPO equivalent in form to the original PO, with differences confined solely to $\eta$ and $\gamma$}; for example, in long-context machine translation, the entire context is task-relevant.
\paragraph{Short-to-long reward alignment (SoLo-RA).} As shown in Eq. (\ref{solo_ra}), SoLo-RA implies that, under optimal conditions, the reward model $r_{\phi}$ should assign a consistent score to response 
$y$ when conditioned on either $x_{long}$ or $x_{short}$, as long as the input retains all task-relevant information $c_{rel}$.
\paragraph{What does the SoLoPO learn?} 
Unlike short-context tasks such as mathematics~\citep{shao2024deepseekmathpushinglimitsmathematical,li2025pspoeffectiveprocesssupervisedpolicy} that draw upon the LLMs' intrinsic reasoning ability, long-context tasks necessitate a two-step process: first, identifying critical information within the given context, and second, executing the task based on that located information~\citep{li-etal-2024-fundamental}. Consequently, proficient in both contextual knowledge localization and contextual knowledge utilization or reasoning. Compared to vanilla PO algorithms, which lack explicit modeling of the former, SoLoPO's decoupled optimization process is better aligned with these requirements, potentially leading to superior performance due to its distinct focus on: (a more detailed discussion can be found in Appendix~\ref{sec:on_modeling_solopo})
\begin{itemize}[leftmargin=0.5cm]
    \item\textbf{Contextual knowledge localization.} SoLo-RA (Eq. (\ref{solo_ra})) forces the reward model to implicitly predict $\hat x_{short}\sim \hat p(x_{short}|x_{long})$, minimizing divergence between predicted $\hat x_{short}$ and actual $x_{short}$. In preference optimization, since the reward model is the policy model itself, this also improves the policy model's ability to identify task-relevant knowledge within long context.
    \item\textbf{Contextual knowledge reasoning.} Since $x_{short}$ contains all task-relevant information, short-context PO (Eq. (\ref{short_po})) enhances the model's reasoning ability over this contextual knowledge.
\end{itemize}

\paragraph{Non-decoupled short-to-long alignment.} Based on the above discussion, another non-decoupled approach to short-to-long alignment involves directly applying preference pairs sampled from short texts for long-context PO or SFT, which we term \textit{Expand-Long-PO} and \textit{Expand-Long-SFT}, respectively. Experiments in Section~\ref{sec:depth_analysis} show that the decoupled approach yields superior performance.

\subsection{Applications of Short-to-Long Preference Optimization}
\label{sec:method_app}
\paragraph{Chosen-only short-to-long reward alignment (chosen-only SoLo-RA).} Considering that $y_l\sim\pi_\theta(y|x_{short})$ may not fully exploit task-relevant contextual information (\textit{e.g.}, responses of ``No answer"), performing SoLo-RA on $y_l$ might introduce negative effects on model learning. A supporting analysis of this issue is provided in Appendix~\ref{sec:supporting_analysis_for_chosen_only_solo_RA}. Therefore, to further improve training efficiency, we only apply SoLo-RA on $y_w$. Experimental analysis in Section~\ref{sec:exp_res} demonstrates the effectiveness of this approach, which also reduces training resource consumption while maintaining training stability.

\begin{table}[h]\small
    \centering
    \caption{Applications of SoLoPO to mainstream PO algorithms: DPO, SimPO, and ORPO. \textbf{1.} Only the chosen‑only SoLo‑RA is shown; the short‑context PO formulation is identical to the original algorithms. \textbf{2.} For DPO, $\pi_{ref}$ is omitted since it does not involved in differentiation.}
    \vspace{-5pt}
    \begin{tabular}{lll}
    \toprule
    \textbf{Original Method} &\textbf{Reward}& \textbf{Chosen-only SoLo-RA} \\
    \midrule
        DPO~\citep{Rafailov2023DirectPO}&$\beta\log\frac{\pi_r(y_w|x)}{\pi_{ref}(y_w|x)}+\beta\log Z(x)$&$|\beta\log\pi_\theta(y_w|x_{short}) - \beta\log \pi_\theta(y_w|x_{long})|$ \\
         SimPO~\citep{meng2024simpo}&$\frac{\beta}{|y_w|}\log\pi_{\theta}(y_w|x)$&$|\frac{\beta}{|y_w|} \log\pi_{\theta}(y_w|x_{short}) - \frac{\beta}{|y_w|} \log\pi_{\theta}(y_w|x_{long})|$ \\
         ORPO~\citep{hong-etal-2024-orpo}&$\log \frac{\pi_\theta(y_w|x)}{1-\pi_\theta(y_w|x)}$&$|\log \frac{\pi_\theta(y_w|x_{short})}{1 - \pi_\theta(y_w|x_{short})} - \log \frac{\pi_\theta(y_w|x_{long})}{1 - \pi_\theta(y_w|x_{long})}|$\\
    \bottomrule
    \end{tabular}
    \label{tab:s2l_app_obj}
    \vspace{-5pt}
\end{table}

SoLoPO can be applied to various PO algorithms, provided that the corresponding convergence function $f(\cdot)$ and upper bound function $s(\cdot)$ are specified (see Table~\ref{tab:SLA_variants}). We apply SoLoPO to mainstream algorithms, including DPO, SimPO, and ORPO, with their corresponding optimization objectives shown in Table~\ref{tab:s2l_app_obj}. For brevity, we only present the expressions for the chosen-only SoLo-RA, while the objective functions for short-context PO remain consistent with the original methods. For DPO, since $\pi_{ref}(y_w\mid x)$ is constant and not involved in differentiation, we only align $\pi_\theta(y_w\mid x)$. See Appendix~\ref{sec:full_express_of_SoLoPO_app} for the complete derivation and expressions. Unless otherwise stated, SoLoPO refers to its chosen-only SoLo-RA variant in the remainder of this paper.

\paragraph{{How does SoLoPO improve data sampling and training efficiency?}}
As illustrated in Eq. (\ref{short_po}), SoLoPO's preference sampling leverages $x_{short}$, which, due to its shorter length and lower processing complexity compared to $x_{long}$, enables faster and more effective sampling of high-quality preference pairs. Furthermore, by applying chosen-only SoLo-RA, we process $x_{long}$ once and $x_{short}$ twice per training step. This contrasts with vanilla PO needing two $x_{long}$ passes, where $x_{long}$ processing is significantly more costly than $x_{short}$. Thus, SoLoPO substantially boosts training efficiency. As the lengths of  $x_{short}$ and $x_{long}$ become more similar (higher $\rho$), the efficiency gain from SoLoPO diminishes. Further detailed analysis and potential optimization methods are discussed in Appendix~\ref{sec:efficiency_analysis}.

%% file: Experiments.tex
\begin{table}[ht]\small
    \centering
    \caption{Composition of different datasets and corresponding trained models. \textbf{1.} SoLo denotes short-to-long alignment, where preference pairs derived from short contexts are used for long-context alignment. \textbf{2.} ``*" means the corresponding PO method used in \ours. \textbf{3.} $D^{\text{SoLo}}$ is also utilized for training LongPO, which falls under non-decoupled Short-to-Long DPO in our framework.}
    \label{tab:baseline_dataset_models}
    \begin{tabular}{lll}
    \toprule
        \textbf{Method} &\textbf{Dataset}& \textbf{Trained Models}\\
    \midrule
    \multirow{2}{*}{SFT} & $D^{\text{sft}}_{\text{short}}=\{(q,x_{\text{short}},y_{w}^{\text{short}})\}$ &$M^{\text{SFT}}_{\text{short}}$\\
    &$D^{\text{sft}}_{\text{long}}=\{(q,x_{\text{long}},y^{\text{long}}_{w})\}$ & $M^{\text{SFT}}_{\text{long}}$\\
    
    \hline
    \multirow{2}{*}{PO} &$D^{\text{po}}_{\text{short}}=\{(q,x_{\text{short}},y^{\text{short}}_{w},y^{\text{short}}_{l})\}$&$M^{\text{PO}}_{\text{short}}$\\
    &$D^{\text{po}}_{\text{long}}=\{(q,x_{\text{long}},y^{\text{long}}_{w},y^{\text{long}}_{l})\}$ & $M^{\text{PO}}_{\text{long}}$\\
    
    \hline
    \multirow{3}{*}{SoLo} &$D^{\text{sft}}_{\text{expand-long}}=\{(q,x_{\text{long}},y_{w}^{\text{short}})\}$&$M^{\text{SFT}}_{\text{expand-long}}$\\
    &$D^{\text{po}}_{\text{expand-long}}=\{(q,x_{\text{long}},y^{\text{short}}_{w},y^{\text{short}}_{l})\}$ & $M^{\text{PO}}_{\text{expand-long}}$\\
    &$D^{\text{SoLo}}=\{(q,x_{\text{short}},x_{\text{long}},y^{\text{short}}_{w},y^{\text{short}}_{l})\}$ &$M^{(*)}_{\text{SoLo}}$\\
    \bottomrule
    \end{tabular}
    \vspace{-10pt}
\end{table}
\section{Experimental Setup}
\label{sec:exp_setting}
\paragraph{Dataset Construction}
We construct $x_{\text{short}}$ and $x_{\text{long}}$ from the MuSiQue~\citep{musique-dataset} training set using the method in RULER~\citep{hsieh2024ruler}. Specifically, we form a long context by mixing relevant documents with random unrelated ones. On average, the short and long context contain $1.1$K and $7.5$K tokens, respectively. We use the original QA-pairs $(q, a)$ from Musique as the questions and ground truth answers. To obtain preference pairs, we perform sampling with a temperature of 0.85 using the instruction model. For each input ($x_{\text{short}},q, a$), we sample $N=32$ Chain-of-Thought~\citep{wei2022chain_of_thought} outputs and then select the corresponding preference pairs $(y_w^{\text{short}}, y_l^{\text{short}})$ using the sub-em metric. Ultimately, we synthesize 5,000 training samples $D = \{(x_{\text{long}}, x_{\text{short}}, q, a, y_w^{\text{short}}, y_l^{\text{short}})\}$. Additionally, we also sample 5,000 real long-context preference pairs $(y_w^{\text{long}}, y_l^{\text{long}})$ based on $x_{\text{long}}$. The composition of different datasets is shown in Table~\ref{tab:baseline_dataset_models}. Figure~\ref{fig:data_condtruction} shows the pipeline and more details are in Appendix~\ref{sec:data_construction}. 

\paragraph{Baselines and Models} As shown in Table~\ref{tab:baseline_dataset_models}, we compare \ours with other approaches that perform SFT or original PO on different datasets. Additionally, we incorporates results from LongPO~\citep{chen2025longpo}, which optimizes the reward of DPO based on short-to-long KL constraint, replacing ${\pi_{ref}(y\mid x_{\text{long}})}$ with ${\pi_{ref}(y\mid x_{\text{short}})}$. We also introduce results from Qwen2.5-Instruct-32B/72B and Llama3.1-Instruct-70B for comparative analysis. We use Qwen2.5-7B-Instruct~\citep{Yang2024Qwen25TR} and Llama3.1-8B-Instruct~\citep{jiang2023mistral7b} as the backbones for our experiments with per-training context window of $32K$ and $128K$, respectively (hereafter, Qwen2.5-Instruct and Llama3.1-Instruct are referred to as Qwen2.5 and Llama3.1 for brevity). Appendix~\ref{sec:scaling_exp} presents experiments on Qwen2.5‑Instruct‑14B to evaluate the scalability of \ours. More training details are provided in the Appendix~\ref{sec:training_details}

\paragraph{Evaluation benchmarks}
To comprehensively analyze the effectiveness of \ours, we conduct evaluations on both long-context and short-context benchmarks. The long-context benchmarks include: (1) Real-world QA tasks from LongBenchV1~\citep{bai-etal-2024-longbench}, used to evaluate the generalization capability of different methods on multi-document and single-document question answering in real scenarios within a $32K$ context size. (2) Synthetic QA tasks based on RULER~\citep{hsieh2024ruler}, used to evaluate the generalization capability of different methods across various context lengths ($4K/8K/16K/32K$). (3) We further leverage LongBenchV2~\citep{bai2025longbenchv2deeperunderstanding} to analyze \ours's potential on longer and more diverse real-world long-context tasks, and employ NIAH-Plus~\citep{zhao-etal-2024-longagent} to examine different models' context knowledge utilization ability in Section~\ref{sec:depth_analysis}. For short-context benchmarks, we use MMLU-Pro~\citep{wang2024mmluprorobustchallengingmultitask}, MATH~\citep{hendrycks2021measuringmathematicalproblemsolving}, GPQA~\citep{rein2023gpqagraduatelevelgoogleproofqa}, IFEval~\citep{ifeval}, and BBH~\citep{suzgun2022challengingbigbenchtaskschainofthought}, following Open LLM Leaderboard~\citep{open-llm-leaderboard-v2}.
 

Following previous works~\citep{bai-etal-2024-longbench,cot_matters}, we utilize F1-score and multiple-choice accuracy as evaluation metrics, based on task-specific formats. For a fair comparison, we select the best-performing checkpoint on LongBenchV1 within a single training epoch as the final model for cross-benchmark evaluation. See Appendix~\ref{sec:evalaution_details} for a detailed description of evaluation settings and benchmarks.
\section{Experimental Results}
\label{sec:exp_res}
In this section, we present our main experimental results highlighting the effectiveness of the \ours framework across various benchmarks and PO methods (\S~\ref{sec:exp_main_res}). Through comprehensive comparative analysis, we provide deeper insights into the key components of \ours: decoupling and direct chosen-only short-to-long reward alignment and analyze the impact of the reward alignment coefficient $\alpha$ (\S~\ref{sec:depth_analysis}). We then experimentally validate the efficiency advantage of \ours (\S~\ref{sec:exp_eff_adv}).
\subsection{Main Results}
\label{sec:exp_main_res}

\begin{table}[ht]\small
    \centering
    \caption{Performance comparison on QA tasks from LongBenchV1 and RUELR. For LongPO, ``pub" denotes the public checkpoint, while ``reimp" indicates our implementation on $D^\text{SoLo}$. \textbf{Bold} and \underline{underlined} indicate the best and the second-best performance, respectively.}
    \adjustbox{width=\textwidth, valign=m}{
    \begin{tabular}{lccccccccc}
    \toprule
         \multirow{2}{*}{\textbf{Model}}&\multicolumn{3}{c}{\textbf{QAs-LongBenchV1}}&&\multicolumn{5}{c}{\textbf{QAs-RULER}}\\
         \cline{2-4} \cline{6-10}
         &\textbf{S-Doc QA}&\textbf{M-Doc QA}&\textbf{Avg.}&&\textbf{4k}&\textbf{8k}&\textbf{16k}&\textbf{32k}&\textbf{Avg.}\\
    \midrule
    Qwen2.5-72B-Instruct & 37.8  & 61.1 & 49.4&&65.7&64.4&61.2&55.0&61.6 \\
    Qwen2.5-32B-Instruct & 34.1  & 49.8 & 41.9 &&58.4&52.1&46.0&43.9&50.1\\
    Llama3.1-70B-Instruct &28.5 &64.1&46.3&&72.1&68.8&52.4&23.2&54.1\\
    {LongPO-Qwen2.5-7B(reimp)} &34.8 &52.6 &43.7 & & 62.4& 54.4 & 48.9 & 43.1 & 52.2 \\
    \makecell[l]{{LongPO-Qwen2.5-7B\citep{chen2025longpo}(pub)}} & 27.5 & 38.3 & 32.9&  &54.7 & 51.9 & 40.6 & 36.3 & 45.9 \\
    \midrule
    \multicolumn{10}{c}{\textbf{Qwen2.5-7B-Instruct}}\\
    \midrule
 Instruct& 29.4 & 39.4 & 34.4 &  & 53.9 & 50.1 & 37.6 & 34.6 & 44.0 \\
 $M^{\scalebox{\scriptscalefactor}{SFT}}_{\scalebox{\scriptscalefactor}{short}}$& 28.9 & 48.4 & 38.6 &  & 63.8 & 56.7 & 42.3 & 31.8 & 48.7 \\
 $M^{\scalebox{\scriptscalefactor}{SFT}}_{\scalebox{\scriptscalefactor}{long}}$& 34.8 & 55.8 & 45.3&  & 65.9 & 61.4 & 58.4 & 52.4 & 59.5 \\
 \hdashline
$M^{\scalebox{\scriptscalefactor}{DPO}}_{\scalebox{\scriptscalefactor}{short}}$& 34.6 & 51.8 & 43.2&  & \underline{70.9} & 63.3 & 45.3 & 46.9 & 56.6 \\
 $M^{\scalebox{\scriptscalefactor}{DPO}}_{\scalebox{\scriptscalefactor}{long}}$& 35.7 & 58.2 & 46.9&  & \textbf{71.0} & 64.2 & 60.6 & 53.2 & 62.2 \\
\rowcolor{Gray} $M^{\scalebox{0.45}{DPO}}_{\scalebox{\scriptscalefactor}{SoLo}}$& \underline{38.0} & 57.6 & 47.8&  & 66.4 & 64.5 & \underline{62.7} & \underline{57.7} & 62.8 \\
 \hdashline
$M^{\scalebox{\scriptscalefactor}{SimPO}}_{\scalebox{\scriptscalefactor}{short}}$ & 34.7 & 53.6 & 44.1&  & 70.8 & 62.5 & 42.1 & 48.8 & 56.0 \\
 $M^{\scalebox{\scriptscalefactor}{SimPO}}_{\scalebox{\scriptscalefactor}{long}}$& 34.2 & 54.9 & 44.6&  & 69.8 & 64.1 & 49.1 & 47.2 & 57.5 \\
 \rowcolor{Gray}$M^{\scalebox{\scriptscalefactor}{SimPO}}_{\scalebox{\scriptscalefactor}{SoLo}}$& \textbf{38.1} & \underline{58.6} & \underline{48.4}&  & 69.2 & \underline{66.0} & \underline{62.7} & \textbf{57.8} & \underline{63.9} \\
 \hdashline
 $M^{\scalebox{\scriptscalefactor}{ORPO}}_{\scalebox{\scriptscalefactor}{short}}$& 28.9 & 48.4 & 38.6&  & 69.1 & 62.1 & 50.8 & 46.6 & 57.1 \\
 $M^{\scalebox{\scriptscalefactor}{ORPO}}_{\scalebox{\scriptscalefactor}{long}}$& 34.8 & 55.8 & 45.3&  & 64.9 & 59.9 & 50.0 & 45.6 & 55.1 \\
 \rowcolor{Gray}$M^{\scalebox{\scriptscalefactor}{ORPO}}_{\scalebox{\scriptscalefactor}{SoLo}}$& 37.6 & \textbf{61.4} & \textbf{49.5}&  & 70.8 & \textbf{68.3} & \textbf{64.0} & 57.3 & \textbf{65.1} \\
    \midrule
    \multicolumn{10}{c}{\textbf{Llama3.1-8B-Instruct}}\\
    \midrule
    Instruct& 30.3 & 49.3 & 39.8&&58.3&49.2&42.9&35.6&46.5\\
    $M^{\scalebox{\scriptscalefactor}{SFT}}_{\scalebox{\scriptscalefactor}{short}}$& 33.0 & \underline{56.2} & 44.6&& \textbf{65.0} &\underline{61.0}&58.5&52.2& \underline{59.2}  \\
    $M^{\scalebox{\scriptscalefactor}{SFT}}_{\scalebox{\scriptscalefactor}{long}}$& 35.0 & \textbf{57.3} & \underline{46.1}& &63.7& 59.0&57.5&\underline{53.5}&58.4 \\
    \hdashline
    $M^{\scalebox{\scriptscalefactor}{ORPO}}_{\scalebox{\scriptscalefactor}{short}}$& 33.1 & 55.1 & 44.1 &&64.0&59.3&\underline{59.7}&50.4&58.4\\

    $M^{\scalebox{\scriptscalefactor}{ORPO}}_{\scalebox{\scriptscalefactor}{long}}$& \textbf{35.4} & 55.4 & 45.4& &63.2&60.1&58.9&53.2&58.9\\


    \rowcolor{Gray}$M^{\scalebox{\scriptscalefactor}{ORPO}}_{\scalebox{\scriptscalefactor}{SoLo}}$& \underline{35.2} & \textbf{57.3} & \textbf{46.3} &&\underline{64.4}&\textbf{62.5}& \textbf{60.1} & \textbf{58.2} & \textbf{61.3}\\
    \bottomrule
    \end{tabular}
    }
    \label{tab:main_results}
\end{table}

\paragraph{\ours effectively enhances long-context capabilities within pre-training windows.}
As illustrated in Table~\ref{tab:main_results}, \ours achieves substantial performance gains and strong generalization, {outperforming the original PO algorithm in 28 out of 32 settings}. Qwen2.5-7B, trained solely on the Musique~\citep{musique-dataset} dataset, achieves a score comparable to Qwen2.5-72B on LongBenchV1. Compared to various original algorithms (DPO, SimPO, ORPO), it attains performance improvements of 0.9, 4.3, and 10.9 points, respectively. Furthermore, \ours exhibits consistently superior performance across varying context lengths on RULER, with only a performance drop observed at the length of $4K$ for DPO and SimPO. Similarly, we conduct experiments with SoLo-ORPO, the best-performing approach, on Llama3.1-8B, and the results further validate our claims. Given that Llama3.1-8B has a context size of $128K$, a $32K$ test length may already lie within its effective range (\textit{i.e.} 25\%~\citep{hsieh2024ruler}), and our experimental data has a maximum length of $8K$; therefore, the gains are smaller compared to Qwen2.5-7B. Specifically, SoLo-ORPO attains performance gains of 4.2 vs. 0.9 on LongBenchV1 and 10.0 vs. 2.4 on RULER over Long-ORPO, for Qwen2.5-7B and Llama3.1-8B, respectively. See Appendix~\ref{sec:scaling_exp} for scalability experiments on Qwen2.5-14B.
\begin{table}[th]\tiny
    \vspace{-5pt}
    \centering
    \caption{Performance comparison of different models on {LongBenchV2} and {Open LLM Leaderboard}. \textcolor{red}{Red} values indicate performance degradation on short-context tasks compared to the Instruct model.}
    \adjustbox{width=\textwidth, valign=m}{
    \begin{tabular}{lcccccccccccc}
    \toprule
    \multirow{2}{*}{\textbf{Model}}&\multicolumn{6}{c}{\textbf{LongBenchV2}}&\multicolumn{6}{c}{\textbf{Open LLM Leaderboard}}\\
    \cmidrule(lr){2-7} \cmidrule(lr){8-12} 
    & \textbf{Overall} & \textbf{Easy} & \textbf{Hard} & \textbf{$<$32k} & \textbf{32k-128k} & \textbf{$>$128k}&\textbf{MMLU-Pro}&\textbf{IFEval}&\textbf{BBH}&\textbf{MATH}&\textbf{GPQA}&\textbf{Avg.}\\
    \midrule
    \multicolumn{13}{c}{\textbf{Qwen2.5-7B-Instruct}}\\
    \midrule
        Instruct&29.3$({\pm0.7})$&30.9&28.3&36.9&24.6&26.1& 44.63 & 74.22 & 55.25 & 36.86 & 31.88 & 48.56
        \\
        {LongPO(pub)}&33.3$(\pm0.5)$&35.0&32.0&40.5&30.0&27.8&44.69&76.49&\textcolor{red}{53.94}&\textcolor{red}{32.32}&\textcolor{red}{31.87}&\textcolor{red}{47.86}\\
        {LongPO(reimp)}&29.6$(\pm1.5)$&32.2&28.0&36.7&26.7&23.7&44.80&\textcolor{red}{73.86}&\textcolor{red}{55.07}&\textcolor{red}{34.81}&31.91&\textcolor{red}{48.08}\\
        \hdashline
        $M^{\scalebox{\scriptscalefactor}{SFT}}_{\scalebox{\scriptscalefactor}{short}}$&\underline{30.8}$({\pm}0.9)$&{33.6}&\underline{29.1}&\underline{39.2}&25.7&27.2& 44.81 & \textcolor{red}{72.18} & \textcolor{red}{55.15} & \textcolor{red}{36.71} & 32.12 & \textcolor{red}{48.19} \\
         $M^{\scalebox{\scriptscalefactor}{SFT}}_{\scalebox{\scriptscalefactor}{long}}$& 30.0$({\pm1.4})$&32.0&28.8&35.7&\underline{25.9}&\underline{28.7}&44.74 & \textcolor{red}{71.46} & 55.35 & \textcolor{red}{36.55} & \textcolor{red}{31.45} & \textcolor{red}{47.90} \\
        $M^{\scalebox{\scriptscalefactor}{ORPO}}_{\scalebox{\scriptscalefactor}{short}}$&{29.3}$({\pm1.1})$&\underline{34.4}&{26.2}&{35.0}&25.3&27.8& 44.78 & 74.70 & 55.26 & 38.21 & \textcolor{red}{30.95}& 48.78 \\
         $M^{\scalebox{\scriptscalefactor}{ORPO}}_{\scalebox{\scriptscalefactor}{long}}$&{26.6}$({\pm1.2})$&{30.1}&{24.5}&{33.8}&{22.3}&{23.3}& 44.64 & \textcolor{red}{73.61} & 55.37 & 37.16 & 32.12 & 48.58 \\
         \rowcolor{Gray}$M^{\scalebox{\scriptscalefactor}{ORPO}}_{\scalebox{\scriptscalefactor}{SoLo}}$&\textbf{33.2}$({\pm1.0})$&\textbf{36.3}&\textbf{31.2}&\textbf{39.7}&\textbf{28.8}&\textbf{30.9}& 44.83 & 75.18 & \textcolor{red}{55.23} & 37.16 & \textcolor{red}{31.46 } & 48.77 \\
         \hdashline
         $M^{\scalebox{\scriptscalefactor}{DPO}}_{\scalebox{\scriptscalefactor}{short}}$&{25.5}$(\pm1.0)$&{27.2}&{24.5}&{29.7}&{23.5}&{22.8}&44.91&75.06&\textcolor{red}{54.99}&39.12&\textcolor{red}{31.21}&49.06\\
         $M^{\scalebox{\scriptscalefactor}{DPO}}_{\scalebox{\scriptscalefactor}{long}}$&29.7$(\pm0.7)$&34.3&{26.9}&{35.9}&25.6&27.6&44.91&75.30&\textcolor{red}{55.06}&37.99&\textcolor{red}{31.80}&49.01\\
         \rowcolor{Gray}$M^{\scalebox{\scriptscalefactor}{DPO}}_{\scalebox{\scriptscalefactor}{SoLo}}$&\textbf{35.2}$(\pm1.2)$&\textbf{37.5}&\textbf{33.8}&\textbf{39.3}&\textbf{31.8}&\textbf{35.0}&44.66&\textcolor{red}{73.98}&\textcolor{red}{54.78}&\textcolor{red}{35.57}&\textcolor{red}{31.63}&\textcolor{red}{48.12}\\
         \hdashline
         $M^{\scalebox{\scriptscalefactor}{SimPO}}_{\scalebox{\scriptscalefactor}{short}}$&{24.6}$(\pm1.5)$&{27.1}&{23.1}&{29.5}&{21.2}&{23.1}&44.97&\textcolor{red}{73.50}&\textcolor{red}{54.94}&38.60&\textcolor{red}{31.46}&48.69\\
         $M^{\scalebox{\scriptscalefactor}{SimPO}}_{\scalebox{\scriptscalefactor}{long}}$&{25.4}$(\pm0.3)$&{25.4}&{25.4}&{33.0}&{20.2}&{23.3}&44.74&\textcolor{red}{73.50}&55.29&37.61&32.05&48.64\\
         \rowcolor{Gray}$M^{\scalebox{\scriptscalefactor}{SimPO}}_{\scalebox{\scriptscalefactor}{SoLo}}$&\textbf{31.0}$(\pm1.3)$&\textbf{34.1}&\textbf{29.1}&\textbf{37.5}&\textbf{25.7}&\textbf{30.6}&44.78&75.90&\textcolor{red}{54.89}&37.76&\textcolor{red}{31.54}&48.97\\
    \midrule
    \multicolumn{13}{c}{\textbf{Llama3.1-8B-Instruct}}\\
    \midrule
    Instruct&32.5$(\pm1.0)$&35.5&\underline{30.7}&\underline{40.7}&27.7&28.5&37.79 & 62.23 & 50.98 & 15.25 & 31.71 & 39.59 \\
    $M^{\scalebox{\scriptscalefactor}{SFT}}_{\scalebox{\scriptscalefactor}{short}}$&31.8$(\pm1.7)$&34.5&30.0&38.3&28.6&27.0&\textcolor{red}{36.37} & \textcolor{red}{60.31} & \textcolor{red}{50.23} & 17.90 & 31.79 & \textcolor{red}{39.32 } \\
     $M^{\scalebox{\scriptscalefactor}{SFT}}_{\scalebox{\scriptscalefactor}{long}}$& 30.9$(\pm1.3)$&\underline{36.8}&27.3&36.0&29.6&24.8& \textcolor{red}{37.04} & \textcolor{red}{60.67} & \textcolor{red}{49.64} & 16.16 & 32.38 & \textcolor{red}{39.17} \\
    $M^{\scalebox{\scriptscalefactor}{ORPO}}_{\scalebox{\scriptscalefactor}{short}}$&\underline{33.7}$(\pm0.3)$&\underline{36.8}&31.8&\textbf{41.2}&\underline{29.8}&\underline{28.7}& \textcolor{red}{37.26} & \textcolor{red}{61.15} & \textcolor{red}{50.95} & \textcolor{red}{14.88} & 32.13 & \textcolor{red}{39.27} \\
     $M^{\scalebox{\scriptscalefactor}{ORPO}}_{\scalebox{\scriptscalefactor}{long}}$ &32.1$(\pm1.1)$&34.5&30.6&39.7&27.8&28.0 & \textcolor{red}{37.43} & \textcolor{red}{60.43} & \textcolor{red}{50.86 } & 15.86 & \textcolor{red}{31.63} & \textcolor{red}{39.24}\\
    \rowcolor{Gray}$M^{\scalebox{\scriptscalefactor}{ORPO}}_{\scalebox{\scriptscalefactor}{SoLo}}$&\textbf{34.7}$(\pm0.9)$&\textbf{37.5}&\textbf{33.0}&40.0&\textbf{32.1}&\textbf{31.2}& \textcolor{red}{37.60 }& 63.43 & \textcolor{red}{50.18 }& 15.63 & \textcolor{red}{31.38} & 39.64 \\
    \bottomrule
    \end{tabular}
    }
    \label{tab:lbV2_short_ben_results}
    \vspace{-10pt}
\end{table}


\paragraph{\ours shows better length generalization beyond the pre-training window.} We test Qwen2.5-7B (w/ YARN~\citep{peng2024yarn}) and Llama3.1-8B on LongBenchV2 with results presented in Table~\ref{tab:lbV2_short_ben_results}. For Qwen2.5-7B (w/ YARN), \ours consistently outperforms the original PO algorithms across varying difficulty levels and context lengths. This highlights the superior generalization ability of models trained with \ours, demonstrating its promise in handling longer and more diverse real-world long-context tasks. For Llama3.1-8B, SoLo-ORPO shows improved performance across all evaluation dimensions, except for a slight degradation on tasks with input length $<$32k words. While short-text training inherently aids length generalization~\citep{gao-etal-2025-train,From_Short_to_Long}, SoLoPO’s advantage likely stems from SoLo-RA, which explicitly enhances contextual knowledge localization, as discussed in Appendix~\ref{sec:on_modeling_solopo}. Ablation experiments on NIAH-Plus in Appendix~\ref{sec:ablation_solo_ra} further support this claim. {Additionally, for DPO, SoLo-DPO outperforms LongPO(pub), despite the latter is trained on a larger volume of longer data. This may stem from LongPO's direct assignment of $y \sim \pi(y|x_{\text{long}})$ as $y_l$ and $y \sim \pi(y|x_{\text{short}})$ as $y_w$ without ensuring $y_w \succ y_l$. In contrast, our method constructs preference data from ground truth, ensuring correctness and quality.}


\paragraph{\ours maintains short-context performance.}

As shown in Table~\ref{tab:lbV2_short_ben_results} on the Open LLM Leaderboard, {\ours maintains short-context capabilities relative to the Instruct model, with only a slight decrease on DPO. This trade-off is acceptable as SoLoPO simultaneously enhances the model's long-context understanding ability and training efficiency}. See Appendix~\ref{sec:short_context_stability} for the supporting analysis of short-context stability in \ours framework.

\subsection{In-Depth Exploration of \ours}
\label{sec:depth_analysis}
\begin{figure}[h]
    \centering
    \begin{subfigure}[b]{0.49\textwidth}
        \centering
        \includegraphics[width=\textwidth]{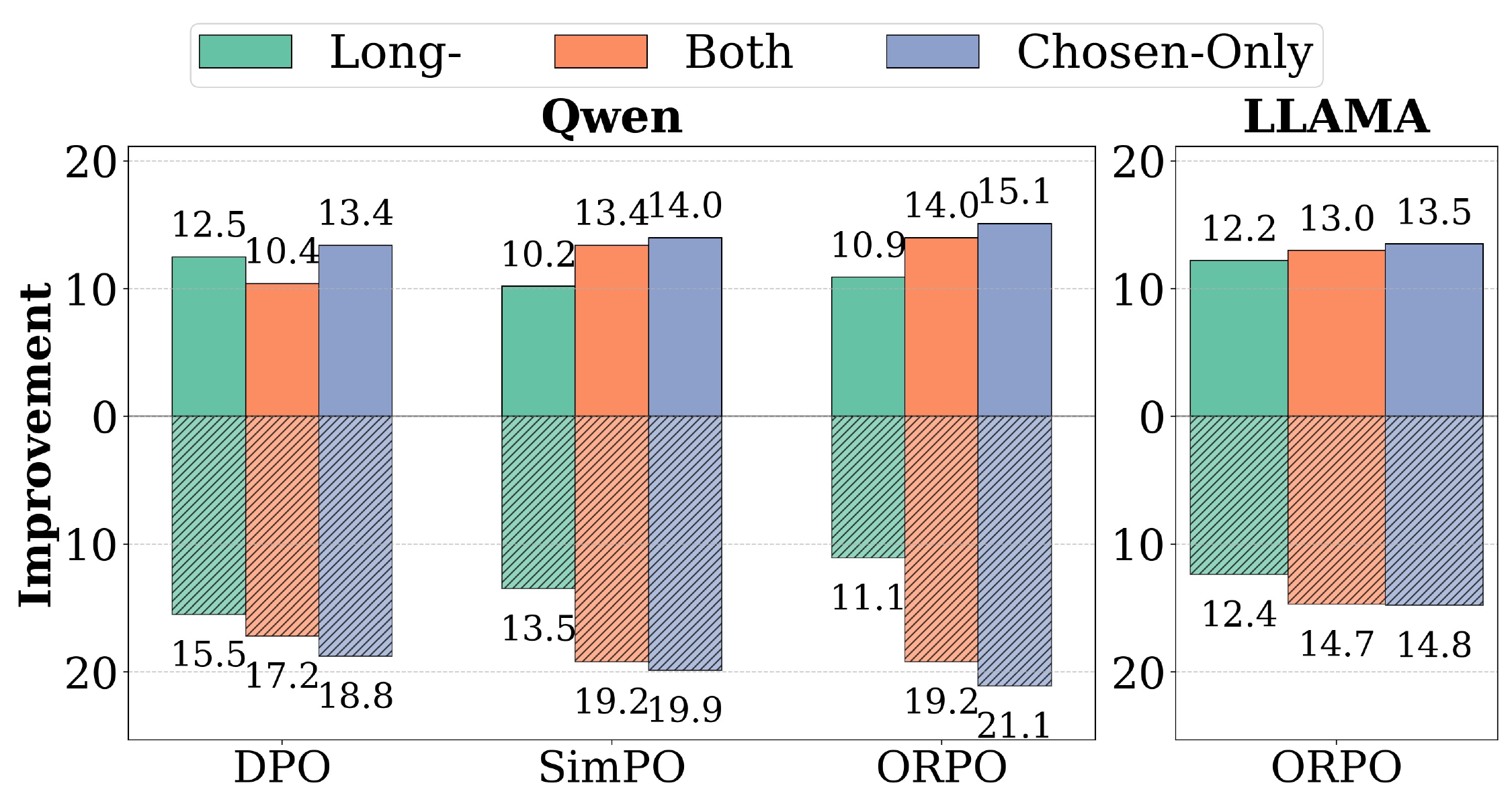} 
        \caption{Comparison of different SoLo-RA approaches. }
        \label{fig:perfor_both_chosen}
     \end{subfigure}%
     %
    \hspace{5pt}
    \begin{subfigure}[b]{0.49\textwidth}
        \centering
        \includegraphics[width=\linewidth]{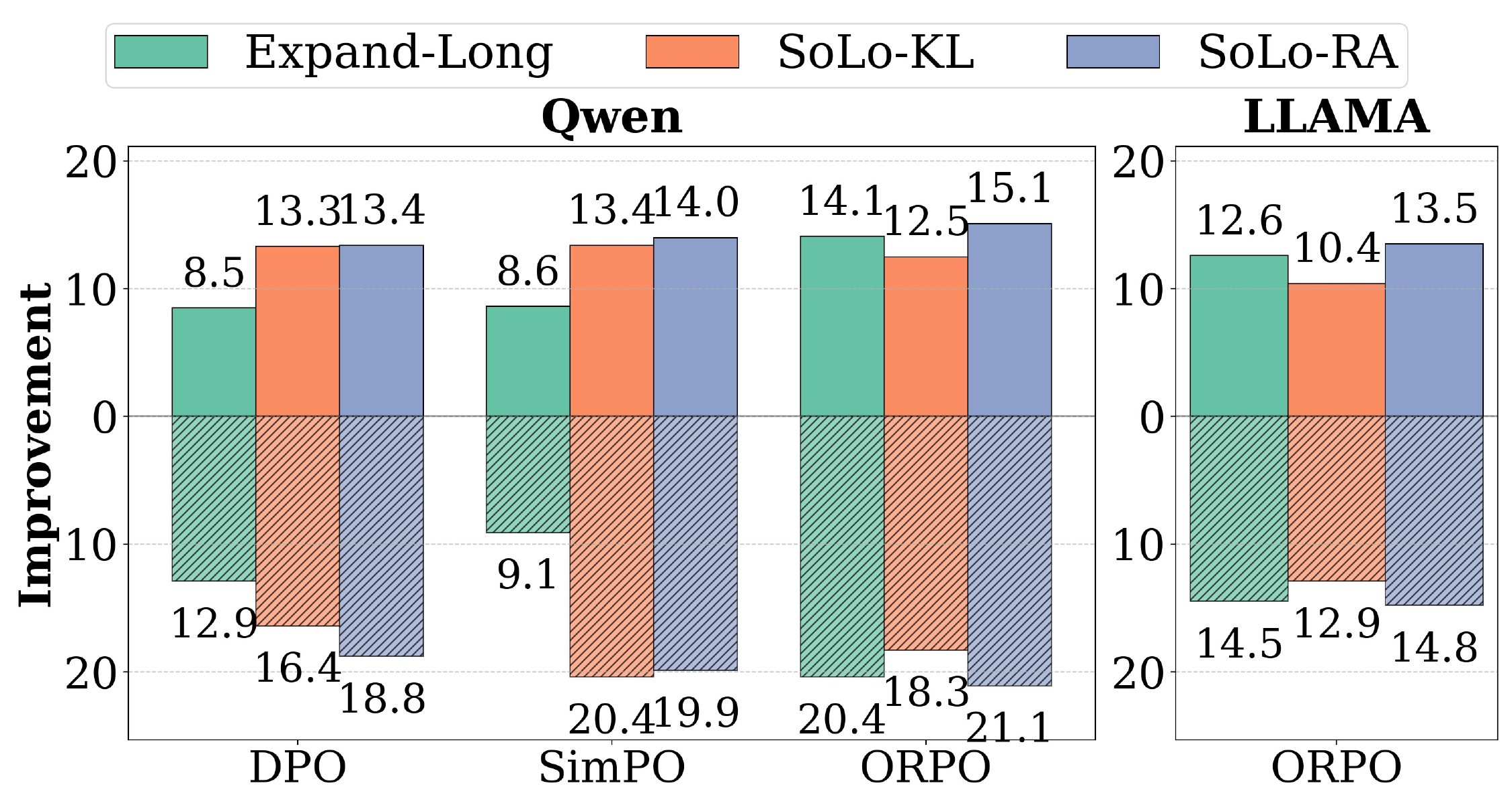} 
        \caption{\ours vs. Approximation variants.}
        \label{fig:perfor_solo_variant}
    \end{subfigure}
    \caption{Performance improvements of different short-to-long preference optimization frameworks based on various PO algorithms over Qwen2.5-7B on LongBenchV1 (top) and RULER (bottom).}
    \label{fig:dep_analy}
    \vspace{-10pt}
\end{figure}
\paragraph{Empirical evidence for SoLoPO and superior performance of chosen-only SoLo-RA.}

We investigate the impact of \textit{chosen-only} SoLo-RA versus applying SoLo-RA jointly to both $y_w$ and $y_l$ (\textit{both}), as defined in original SoLoPO (Eq.(\ref{solo_ra})). As shown in Figure~\ref{fig:perfor_both_chosen}, SoLoPO with \textit{both} SoLo-RA consistently outperforms Long-PO, except for a slight drop on DPO, supporting the validity of its decoupling strategy theoretically established in Section~\ref{sec:method_loss_analysis}. Moreover, the \textit{chosen-only} SoLo-RA surpasses the \textit{both} approach across diverse algorithms and models. Appendix~\ref{sec:trainin_dy} reveals that, compared with the \textit{both} setting, the \textit{chosen-only} version yields larger reward margins (Figure~\ref{fig:s2l_orpo_long_margins}) and assigns lower prediction probability to $y_l$ (Figure~\ref{fig:s2l_orpo_long_rejected}). These results suggest that \textit{chosen-only} SoLo-RA achieves stronger fitting capacity and more stable convergence, leading to better performance.

\paragraph{Direct reward alignment matters.}
To validate the effectiveness of SoLo-RA, we compare it with an approximation we termed short-to-long KL divergence (SoLo-KL): 
{\small
\begin{equation}
\alpha\cdot\mathbb{E}_{\substack{x_{.}\sim \mathcal{D}x_{.};y_w\sim\pi_\theta(y|x_{\text{short}})}}|\log\pi_\theta(y_w\mid x_{\text{short}})-\log\pi_\theta(y_w\mid x_{\text{long}})|.
\end{equation}
}
Here, we employ the absolute function to ensure non-negative training loss. From the reward expressions in Table~\ref{tab:s2l_app_obj}, one can observe that SoLo-KL also promotes the convergence between $r_\phi(x_{\text{short}},y_w)$ and $r_\phi(x_{\text{long}},y_w)$. For DPO and SimPO, SoLo-KL is equivalent to SoLo-RA, as the reward coefficients $\beta$ can be integrated into the coefficient $\alpha$ in Eq. (\ref{solo_ra}). As shown in Figure~\ref{fig:perfor_solo_variant}, for DPO and SimPO, the performance of SoLo-KL and SoLo-RA is comparable, with minor differences attributed to the slight variations in $\alpha$. However, for ORPO, SoLo-RA significantly outperforms SoLo-KL on both benchmarks, demonstrating the effectiveness of direct reward alignment.

\paragraph{Decoupling long-context alignment yields better results.}

As shown in Figure~\ref{fig:perfor_solo_variant}, \ours outperforms Expand-Long-PO, a non-decoupled approach across different PO algorithms. We posit
\begin{wraptable}{r}{6.5cm}\small
    \vspace{-5pt}
    \centering
    \caption{Performance gains of various ORPO over Instruct model on \textit{NIAH-Plus}~\citep{zhao-etal-2024-longagent}.}
    \begin{tabular}{lccc}
    \toprule
         \textbf{Model}&\textbf{S-Doc QA}&\textbf{M-Doc QA}&\textbf{AVG.}\\
    \midrule
        \multicolumn{4}{c}{\textbf{Qwen2.5-7B-Instruct}}\\
        \midrule
         $M^{\scalebox{\scriptscalefactor}{ORPO}}_{\scalebox{\scriptscalefactor}{expand-long}}$&23.94&17.29&20.62\\
         \rowcolor{Gray}$M^{\scalebox{\scriptscalefactor}{ORPO}}_{\scalebox{\scriptscalefactor}{SoLo}}$&\textbf{25.98}&\textbf{18.83}&\textbf{22.41}\\
         \midrule
         \multicolumn{4}{c}{\textbf{LLama3.1-8B-Instruct}}\\
         \midrule
         $M^{\scalebox{\scriptscalefactor}{ORPO}}_{\scalebox{\scriptscalefactor}{expand-long}}$&11.65&9.77&10.71\\
         \rowcolor{Gray}$M^{\scalebox{\scriptscalefactor}{ORPO}}_{\scalebox{\scriptscalefactor}{SoLo}}$&\textbf{11.82}&\textbf{20.16}&\textbf{15.99}\\
    \bottomrule
    \end{tabular}
    \label{tab:naih}
    \vspace{-10pt}
\end{wraptable}
that the decoupled approach, through SoLo-RA, explicitly improves the model's contextual knowledge utilization ability by enabling direct comparison between short and long contexts,  while explicitly strengthening the model's perception of rewards and preferences. We further test the contextual knowledge utilization ability of different models on \textit{NIAH-Plus}~\citep{zhao-etal-2024-longagent} to validate the above claim. As shown in Table~\ref{tab:naih}, \ours achieves consistently greater improvements over Expand-Long-ORPO, confirming the rationality and effectiveness of our decoupled approach. Similar trends can be observed in the results of DPO and SimPO presented in Table \ref{tab:full_naih}. {Additionally, Figure~\ref{fig:niah_all} in Appendix~\ref{sec:detailed_bench_res} presents heatmaps of model performance on \textit{NIAH-Plus} when trained with different ORPO variants. Compared to Long-ORPO or Expand-Long ORPO, SoLo-ORPO significantly enhances the model's ability to retrieve information across various depths and context lengths in both single-hop and multi-hop settings, which further supports our claim.}


\paragraph{Impact of reward alignment coefficient $\alpha$ in \ours}
To evaluate the influence of $\alpha$ in Eq. (\ref{solo_ra}), we progressively adjust $\alpha$ in SoLo-ORPO across two distinct foundation models. As shown in Figure~\ref{fig:beta_s2l}, all architectures exhibit characteristic response curves with clear peaks in performance metrics, exceeding Long-ORPO in most settings. See Appendix~\ref{sec:impact_for_dpo_simpo} for analysis of DPO and SimPO.

\begin{figure}[h]
    \centering
    \begin{minipage}[t]{0.41\textwidth}
      \centering
      \includegraphics[width=\linewidth]{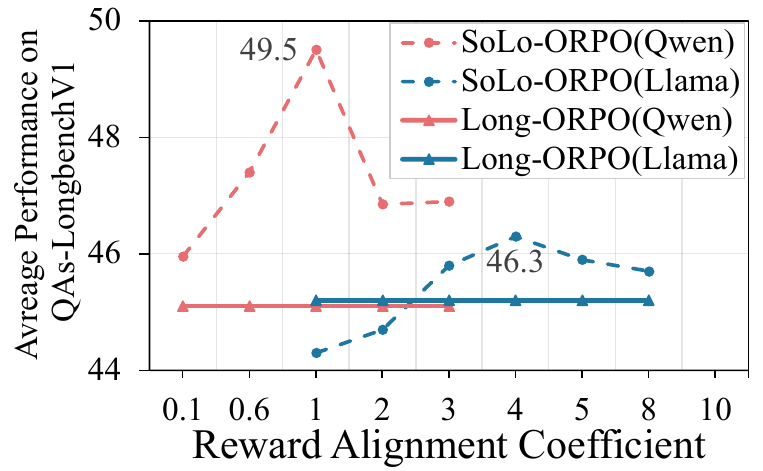}
      \caption{Performance w/ different $\alpha$ in SoLo-ORPO.}
      \label{fig:beta_s2l}
    \end{minipage}%
    ~ 
    \begin{minipage}[t]{0.56\textwidth}
      \centering
      \includegraphics[width=\linewidth]{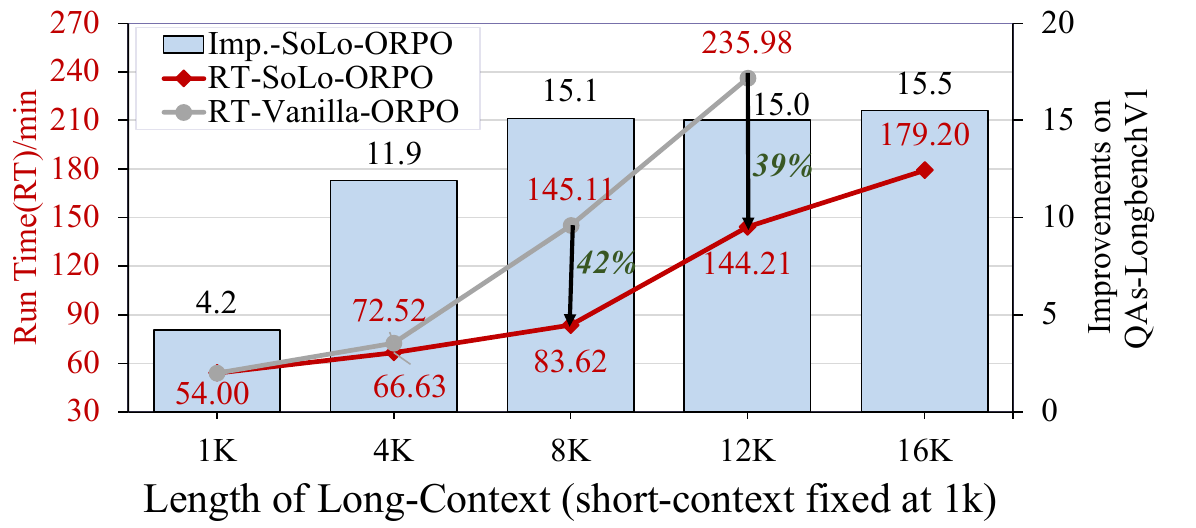}
      \caption{Run time (RT) and performance gains (Imp.) under varying lengths of $x_{\text{long}}$, with $x_{\text{short}}$ fixed at $1K$.}
      \label{fig:performance_vs_run_time}
    \end{minipage}
    \vspace{-10pt}
\end{figure}

\subsection{Efficiency Advantage of \ours}
\label{sec:exp_eff_adv}
The chosen-only SoLo-RA reduces the processing of long texts during training, thereby improving overall efficiency. As illustrated in Figure~\ref{fig:performance_vs_run_time}, we fix the length of $x_{\text{short}}$ at $1K$ and investigate how varying lengths of $x_{\text{long}}$ affect the performance gains and {the efficiency gains in training time} for SoLo-ORPO in the Qwen2.5-7B setting using 4$\times$A100 GPUs. Results show that as the length of $x_{\text{long}}$ increases, SoLo-ORPO achieves significant efficiency gains over the vanilla ORPO, cutting run time by 42\% and 39\% for $x_{\text{long}}$'s lengths of $8K$ and $16K$, respectively. Moreover, for Qwen2.5-7B with a $32K$ pretrained context size, setting the length of $x_{\text{short}}$ and $x_{\text{long}}$ to $1K$ and $8K$, respectively, yields a favorable trade-off between model performance and computational efficiency. Notably, as shown in Figure~\ref{fig:efficiency}, with only ZeRO stage 3 and offloading enabled, \ours supports trainable length up to $19K$ tokens, while vanilla methods are limited to $9K$. See Appendix~\ref{sec:efficiency_analysis} for more experimental details and discussions.

%% file: Related.tex
\section{Related Work}

\paragraph{Long-Context Data Augmentation} This approach leverages advanced LLMs to synthesize high-quality, long-dependency instruction-following data, used with PO or SFT for long-context alignment~\citep{MDCure,LongReward,Generalizing_Short2Long,LongFaith,bai-etal-2024-longalign}. However, as text length increases, it becomes less reliable and efficient~\citep{logo}. Moreover, short-context alignment methods may underperform in long-context settings~\citep{Survey_Dong2023ASO,LongCEloss,BART_Lewis2019}. 
\paragraph{Long-Context Alignment Objective Optimization} \citet{LongCEloss} propose LongPPL and LongCE loss to identify key tokens in long-text modeling and increase the loss weights for these critical tokens, thereby improving the effectiveness of long-context SFT. LongPO~\citep{chen2025longpo} searches for preference pairs based on short texts and applies them to long-context DPO training to achieve the non-decoupled short-to-long alignment discussed in our paper. Additionally, LongPO introduces a short-to-long constraint, replacing ${\pi_{ref}(y\mid x_{\text{long}})}$ with ${\pi_{ref}(y\mid x_{\text{short}})}$, thereby maintaining the short-context ability. However, LongPO focuses on context size expansion while its optimization is restricted to DPO. Both LongPO and LongCE suffer from limited training efficiency due to their reliance on long text processing, with LongCE incurring additional overhead associated with critical token detection.

\ours introduces a theory-based decoupling strategy for long-context preference optimization, enabling more effective modeling of contextual knowledge localization and reasoning. The \textit{chosen-only} SoLo-RA variant improves performance while facilitating data construction and training efficiency. Moreover, integrating SoLoPO with long-context data augmentation may further improve its alignment performance. A more comprehensive review of related work is provided in Appendix~\ref{sec:detailed_related_work}.

%% file: Conclusion.tex
\section{Conclusion}

In this work, we propose \ours, a general framework for long-context preference optimization (PO). Our method decouples long-context PO into short-context PO and short-to-long reward alignment, supported by both theoretical and empirical evidence. Experimental results demonstrate that the \textit{chosen-only} variant of \ours consistently outperforms vanilla PO methods and enhances the model's generalization ability in handling long contexts across diverse domains and lengths, while significantly improving both data and training efficiency. SoLoPO highlights the importance of the connection between short and long texts, paving the way for more effective long-context alignment.

{Our findings open up several promising avenues for future investigation, such as enhancing training efficiency for fully context-relevant tasks and exploring how the core ideas of \ours can inform the optimization of long-output generation tasks. A more detailed discussion is provided in Appendix~\ref{sec:limitations}.}

%% file: Acknowledgments.tex
\section*{Acknowledgments}
Yang Gao was supported by the Major Research Plan of the National Natural Science Foundation of China (Grant No. 92370110) and the Joint Funds of the National Natural Science Foundation of China (Grant No. U21B2009).

%% file: else.tex
\section*{Ethics statement}
This work focuses on optimizing preference learning objectives for large language models (LLMs) in long-context scenarios. It does not introduce unexpected ethical risks beyond those commonly considered in standard NLP research. Although LLMs are trained on large amounts of Internet text that may contain harmful content, our study targets long-context understanding rather than direct deployment, which greatly reduces the risk of propagating biased information. All models and datasets used in our experiments are open-sourced and publicly available, ensuring transparency and minimizing potential ethical concerns.
\section*{Reproducibility statement}
We have provided comprehensive details of our method (Section~\ref{sec:method_loss_analysis} and Appendix~\ref{sec:math_derivation}), data synthesis pipeline (Section~\ref{sec:data_construction}), model training settings (Appendix~\ref{sec:training_details}), and evaluation benchmarks and configurations (Appendix~\ref{sec:evalaution_details}) in the main paper. All datasets and models used in this work are openly available, with direct access links provided in our paper and the supplementary materials. To further facilitate reproduction of our results, we release the complete source code, data examples, and step-by-step usage instructions in the supplementary materials. These resources are intended to enable other researchers to fully replicate and verify our experiments.

%% file: Appendix.tex
\section{The Use of Large Language Models}
We explicitly disclose that large language models (LLMs) were used solely for the following purposes: (1) \textbf{Writing refinement} – limited to minor grammar correction, wording improvement, and stylistic polishing of the manuscript text; (2) \textbf{Data generation} – specifically, the preference data required for our experiments were sampled from the corresponding open-source Instruct models, following the procedures described in the paper. 

LLMs were not used for idea conception, methodological design, result analysis, or any other substantive scientific contribution to this work. All research ideas, methodological innovations, experimental executions, analyses, and conclusions were conceived, implemented, and validated entirely by the authors.

\section{Related Work}
\label{sec:detailed_related_work}
Numerous studies focus on extending the limited pre-training context windows of LLMs to support longer inputs, including post-training on long-context corpora~\citep{DE_128k,chen2024longlora,Extending_llama3}, designing novel architectures~\citep{gu2022efficiently,lu2025mobamixtureblockattention,gu2024mamba}, or modifying positional encodings~\citep{peng2024yarn,BiPE,cream,zhu2024pose}. However, researches reveal that the capability within the pretraining window of current LLMs has not been fully activated, resulting in suboptimal performance on long-context tasks~\citep{hsieh2024ruler,BABILong,belyi-etal-2025-luna,zhang2024longcite}. To address this challenge, existing approaches primarily focus on two aspects: data augmentation and training objective optimization.
\paragraph{Long-Context Alignment based on Data Augmentation}
Most research~\citep{MDCure,LongReward,Generalizing_Short2Long,LongFaith} synthesizes high-quality, long-dependency instruction-following data for supervised fine-tuning or offline preference optimization. Instruction synthesis~\citep{MDCure,bai-etal-2024-longalign} directly leverages one or multiple real long documents to prompt advanced LLMs to generate diverse instructions and responses for long-text alignment. Context synthesis~\citep{Generalizing_Short2Long}, on the contrary, is built on real QA data by having models synthesize background contexts based on questions, then randomly concatenates multiple synthetic contexts to create long-context instruction-following data. Although data augmentation demonstrates some effectiveness, directly applying short-context training methods to the long-context setting may overlook differences between short and long contexts, thus failing to fully activate the model's potential capabilities~\citep{Survey_Dong2023ASO,BART_Lewis2019}.  While \ours incorporates the connection between short and long contexts into its training objective, it can be combined with data enhancement techniques, which has the potential to further enhance model performance.
\paragraph{Long-Context Alignment based on Training Objective Optimization}
Some works explore optimizing training objectives to further enhance long-context capabilities in LLMs. \citet{LongCEloss} propose LongCE loss to identify key tokens in long-text modeling and increase the loss weights for these critical tokens, thereby improving the effectiveness of long-context SFT. LOGO~\citep{logo} employs multiple negative samples and adapts the SimPO objective to minimize the probability of generating various dis-preference instances. LongPO~\citep{chen2025longpo} searches for preference pairs based on short texts and applies them to long-context DPO training to achieve the non-decoupled short-to-long alignment discussed in our paper. Additionally, LongPO introduces a short-to-long constraint, utilizing the output distribution of short texts on the reference model as a reference during long-context DPO training (replacing ${\pi_{ref}(y\mid x_{long})}$ with ${\pi_{ref}(y\mid x_{short})}$), thereby maintaining short-context capabilities. Although LOGO and LongPO adopt similar data construction strategies to ours, they fundamentally differ in that they do not decouple the optimization objectives. As a result, these methods fall into the category of non-decoupled short-to-long alignment discussed in our work. Moreover, LOGO, LongPO and LongCE suffer from limited training efficiency due to their reliance on long text processing, with LongCE incurring extra overhead from critical token detection.

\ours introduces a theoretically grounded framework for long-context preference optimization. Specifically, \ours explicitly models the connection between short and long contexts, decoupling long-text preference optimization into short-text preference optimization and short-to-long reward alignment. The \textit{chosen-only} variant of \ours not only improves the model's long-context ability, but also significantly enhances the efficiency of both data construction and training procedure.
\section{Limitations \& Future Work}
\label{sec:limitations}

\paragraph{Training Efficiency Enhancements} For tasks where compression rate is 100\%, such as long-context  machine translation, \ours is equivalent to the original PO algorithm, offering no gain in training efficiency. Given that redundancy also exists in the hidden states of LLMs~\citep{huang-etal-2024-fewer,pan-etal-2024-llmlingua,li-etal-2023-compressing}, future research could extend token-level compression to hidden-state-level compression, potentially by combining our approach with KV-cache compression techniques~\citep{li2024snapkv,li-etal-2023-compressing}. Such an extension would better support a wider variety of long-text applications. Moreover, although \ours leverages the chosen-only SoLoRA strategy, it remains necessary to process long sequences, which can lead to efficiency bottlenecks when dealing with large-scale datasets. Future work could explore the decoupling strategy of \ours in combination with data pruning techniques~\citep{qin2024infobatch,marion2023moreinvestigatingdatapruning}, aiming to appropriately reduce the processing of long-context inputs and thereby improve training efficiency. 

\paragraph{Toward Extended Theoretical Analysis} \ours is primarily designed for long-context \textit{input} scenarios, and therefore does not directly address the challenges of long-text \textit{generation}~\citep{bai2025longwriter}. Extending our theoretical analysis to long-text generation settings represents a natural and important direction for future work, which would further broaden the applicability of \ours. {The core principle of SoLoPO posits conditional equivalence between short and long inputs. We believe that this concept similarly extends to long-output generation tasks, as exemplified below:
\begin{itemize}[leftmargin=0.5cm]
\vspace{-5pt}
    \item Long Chain-of-Thought (CoT)~\citep{tang-etal-2025-unlocking}: A long CoT that ultimately yields a correct final answer is equivalent to a concise CoT achieving the same result from a task completion perspective.
    \item Text Refinement: A stylistically refined and thus longer text can be deemed semantically equivalent to a plainer, shorter version, as long as the core semantic meaning is retained.
    \item Story Generation~\citep{gurung2025learning}: Longer story outputs are functionally equivalent to shorter versions if both fulfill a user-specified narrative arc or core plot, despite offering greater descriptive depth.
\end{itemize}}

Moreover, as observed from Equation (\ref{solo_ra}), original preference optimization (short-context PO) focuses on discrepancies in the output space, whereas our proposed SoLo-RA emphasizes relationships in the input space. This raises an intriguing question: \textit{Can the decoupled preference modeling approach underlying \ours be generalized to other tasks where modeling input-side connections plays a critical role?} (other context-aware tasks, such as complex instruction following~\citep{IOPO} and context-faithful alignment~\citep{context_align}.) Investigating this direction may yield new insights into the design of more expressive and flexible preference optimization frameworks.

\paragraph{More experimental analysis} Due to resource limitations, our current experiments primarily focus on efficiently activating capabilities within the model's pretrianing context window ($32K$). Future work should further evaluate the effectiveness of \ours on even longer context and larger model scales to fully understand its capabilities. Additionally, \ours introduces two hyperparameters—compression ratio $c$ and reward alignment coefficient $\alpha$—which require manual tuning. Future work could explore automated methods for determining optimal values for these parameters. Moreover, while we believe that \ours could support self-evolving mechanisms for progressive context window expansion~\citep{chen2025longpo,logo}, this remains to be validated via more comprehensive analysis.
\paragraph{Higher-Quality Data Synthesis} While our current approach to constructing the short-to-long dataset is simple yet effective, it suffers from limited realism and diversity. Future work could explore integrating \ours with existing data augmentation techniques to synthesize more realistic long-context instruction-following data, such as instruction or context synthesis grounded in authentic data sources~\citep{MDCure,Generalizing_Short2Long,LongFaith,bai-etal-2024-longalign}. Moreover, extending the dataset to cover a broader range of long-text scenarios—such as long-document summarization, long-in-context learning, and long-form dialogue understanding—could provide a more comprehensive improvement of models' long-context processing capability. 

\section{Dataset Construction}
\label{sec:data_construction}
In this section, we present the methodology for the short-to-long preference dataset construction. As noted in our preliminary experiments and related studies~\citep{Generalizing_Short2Long,gao2025how}, training on shorter texts can still yield improvements in performance over longer contexts. Inspired by these findings, we heuristically set the average length of short contexts ($x_{short}$) to approximately $1K$ tokens and long contexts ($x_{long}$) to around $8K$ tokens, which corresponds to 25\% of the pretraining context window ($32K$) of Qwen2.5-7B-Instruct. The overall data construction pipeline is illustrated in Figure~\ref{fig:data_condtruction}.

\begin{figure}[ht]
    \centering
    \includegraphics[width=0.8\textwidth]{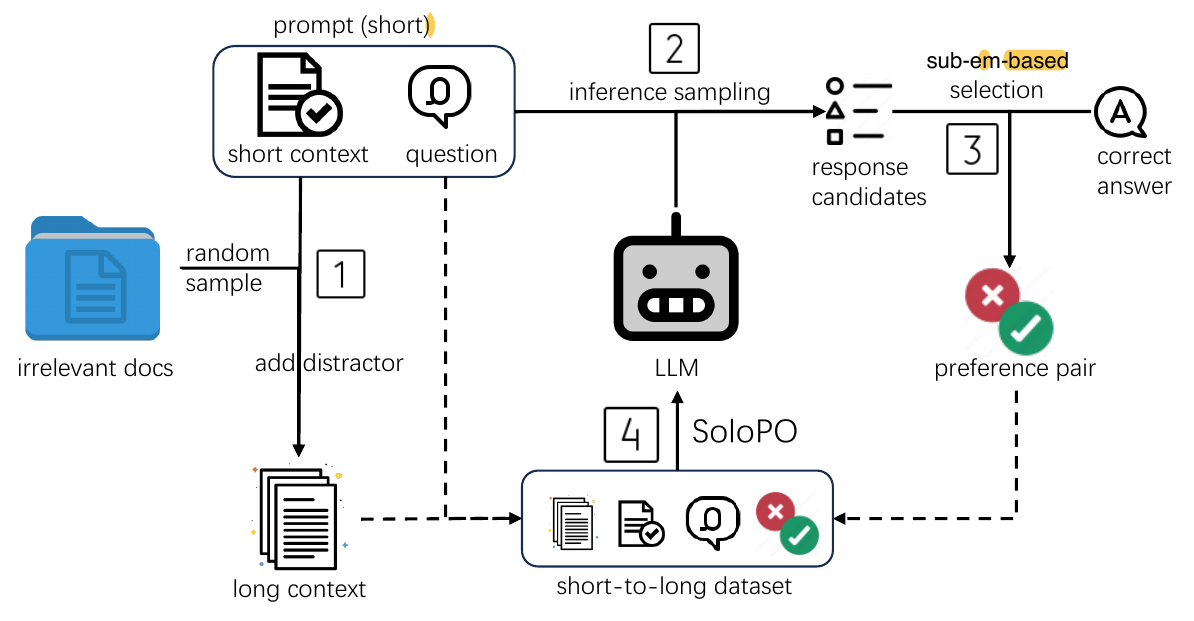}
    \caption{Illustration of the construction pipeline for the short-to-long dataset. (1) Irrelevant documents are randomly sampled and concatenated with the original short input to form long contexts. (2) Multiple candidate responses are generated based on the short context and question via the instruct model. (3) Preference pairs are curated using a sub-em\protect\footnotemark based selection guided by ground-truth answers. (4) The final short-to-long dataset, composed of short contexts, long contexts, questions, and preference pairs, is used for training LLM with \ours.}
    \label{fig:data_condtruction}
    \vspace{-10pt}
\end{figure}
\footnotetext{Alternative evaluation methods, such as LLM-as-Judge, may be employed provided they can differentiate preference pairs.}
\paragraph{Context Synthesis} We follow the strategy proposed by RULER~\citep{hsieh2024ruler} to generate synthetic contexts based on the MuSiQue dataset. Specifically, given an original sample $(q,a, D_{supporting})$, where $q$ and $a$ denote the question and ground-truth answer respectively, and $D_{supporting}$ is the set of supporting evidence documents, we synthesize a context of target length L as follows: we randomly sample a set of irrelevant documents $D_{irrelevant}$, such that the total token count of $D_{supporting} \cup D_{irrelevant}$ approximates $L$. The token count is calculated using the tokenizer of the corresponding instruct model. In this setup, both short and long contexts contain the necessary information to answer the question; however, the long contexts include more distractor content, thereby simulating the redundancy in natural language. 
\paragraph{Preference Pairs Construction} To evaluate the effectiveness of short-to-long alignment compared with original long-context alignment, we construct preference pairs based on both short and long contexts using instruct models. Specifically, for each input context and question, we first generate $N=32$ Chain-of-Thought responses using a sampling temperature of $0.85$ to encourage response diversity. We then apply the sub-em method to identify the chosen and rejected responses within the generated samples. A final preference pair is formed by randomly selecting one from the chosen and the rejected candidates, respectively. Samples where all responses are correct/incorrect are discarded. 

To enhance the validity of the comparative analysis, we prioritize constructing training examples from the intersection of samples that can be successfully constructed using both short and long contexts. This ensures that differences in model performance stem primarily from alignment strategies rather than data distribution. The final size of each train set is 5000, as detials shown in Table~\ref{tab:baseline_dataset_models} and an example is shown in Table~\ref{tab:data_example}. The prompt template used for data construction is illustrated in Figure~\ref{fig:prompt_template}.

\begin{table}[ht]\small
    \vspace{-5pt}
    \centering
    \caption{An example preference pair sampled from Qwen2.5-7B-Instruct using the short context ($x_{short}$). Certain reasoning details have been omitted and denoted by ``…" for conciseness.}
    \begin{adjustbox}{max width=\textwidth}
    \begin{tabular}{ll}    
    \toprule
         \textbf{question}&\texttt{When was the institute that owned The Collegian founded?}  \\
         \textbf{answer}&\texttt{1960}\\
         \hline
         \textbf{chosen}&\texttt{\makecell[l]{To find when the institute that owned The Collegian was founded, \\let's follow these steps:\\1. Identify the owner of The Collegian:...\\2. Find the founding year of Houston Baptist University: ...\\The answer is: 1960}}\\
         \hline
         \textbf{rejected}&\texttt{\makecell[l]{To answer the question, we need to identify which institute \\"The Collegian" is associated with and then find its founding date.\\However, the provided passages do not explicitly link The Collegian ...\\no founding date for the university is given in the passages provided.\\The answer is: No answer.}}\\
    \bottomrule
    \end{tabular}
    \end{adjustbox}
    \label{tab:data_example}
\end{table}
\begin{figure}[ht]
    \centering
    \includegraphics[width=0.9\textwidth]{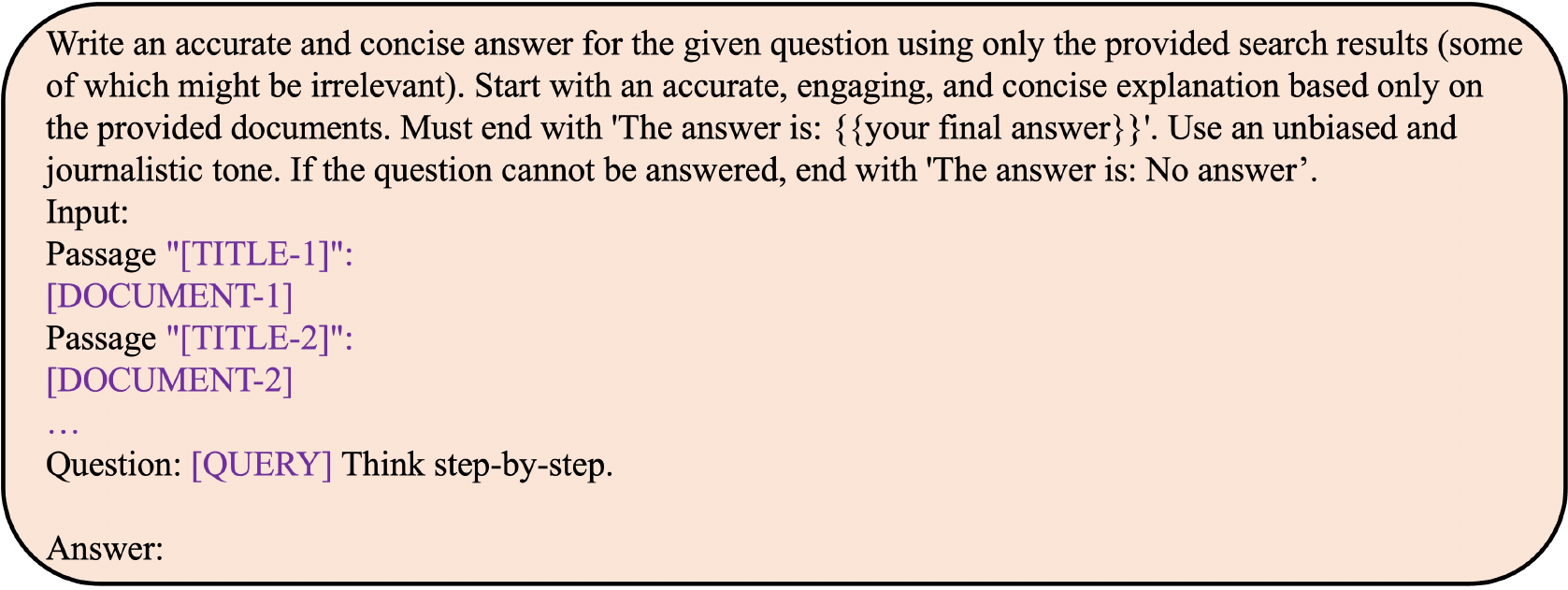}
    \caption{Prompt template used for data construction and training, adapted from~\citet{coc}}
    \label{fig:prompt_template}
\end{figure}
\section{Implementation Details}
\label{sec:imple_details}
In this section, we describe the implementation details of our experiments, including the configurations for model training and evaluation.
\subsection{Model Training Configuration}
\label{sec:training_details}
\paragraph{General training settings} We train our model using LLaMAFactory~\citep{zheng-etal-2024-llamafactory} for data with a maximum length $\le8K$. To enhance training efficiency and better utilize GPU memory, we employ Flash-Attention 2~\citep{dao2023flashattention2fasterattentionbetter} and DeepSpeed ZeRO stage 3 with offloading~\citep{rajbhandari2020zeromemoryoptimizationstraining} strategies. All models are fully trained on four NVIDIA A100 80GB GPUs in bf16 precision. We use the AdamW optimizer~\citep{adamw} together with a cosine learning rate scheduler. The \texttt{warmup\_ratio} is set to 0.1 and the \texttt{total batch size} is 64. For a fair comparison, for each method we select the checkpoint that achieves the best performance on the QA tasks in LongBench-V1 during a single training epoch\footnote{Results show that, on our dataset, all training methods achieve their best performance within a single epoch.} as the final model to be evaluated across different benchmarks. 
\paragraph{Supervised fine-tuning (SFT)} For SFT, the maximum learning rate is set to $1\times10^{-5}$.
\paragraph{Original preference optimization} We apply DPO, SimPO, and ORPO with a smaller maximum learning rate $1\times10^{-6}$. Following the original methods~\citep{Rafailov2023DirectPO,meng2024simpo,hong-etal-2024-orpo}, for DPO and ORPO, we set the $\beta=0.1$, and for SimPO, we set $\beta=2.0$ and $\gamma=1.4$.
\paragraph{Short-to-Long preference optimization (SoLo-PO)} For SoLo-PO, we maintain the same training parameters as in the original method. Specifically, SoLo-DPO, SoLo-SimPO, and SoLo-ORPO achieve optimal performance on LongBenchV1 with the reward alignment coefficient $\alpha$ in Equation (\ref{solo_ra}) set as 3, 1, and 1, respectively, for Qwen2.5-7B, while SoLo-ORPO yields better performance with $\alpha=4$ for LLaMA3.1-8B. 
\paragraph{LongPO} We train Qwen2.5-7B-Instruct on our constructed short-to-long dataset with the publicly available LongPO codebase\footnote{\url{https://github.com/DAMO-NLP-SG/LongPO}}, using the default hyperparameter settings, including the optimizer, learning rate and the learning rate scheduler, except for the batch size, which is set to 64 to match our setup. In addition, we also evaluate the publicly released model checkpoint\footnote{\url{https://huggingface.co/DAMO-NLP-SG/Qwen2.5-7B-LongPO-128K}} for direct comparison.
\subsection{Evaluation Details}
\label{sec:evalaution_details}
\subsubsection{Evaluation Benchmarks}
To comprehensively evaluate the capabilities of \ours, we conduct experiments across a diverse set of benchmark datasets, as follows:
\paragraph{QA tasks from LongBenchV1 and RULER} Given that our training data is derived from the multi-hop QA dataset MuSiQue~\citep{musique-dataset}, with a maximum sequence length of $8K$ tokens, we primarily assess \ours's performance on long-context QA tasks within a $32K$ context size. Specifically, we use the QA tasks from LongBenchV1~\citep{bai-etal-2024-longbench} to evaluate the model’s generalization across various real-world domains. These include single-document QA tasks such as Qasper~\citep{qasper-dataset}, NarrativeQA~\citep{narrativeqa-dataset}, and MultiFieldQA-En~\citep{bai-etal-2024-longbench}, as well as multi-document QA tasks including HotpotQA~\citep{hotpotqa-dataset}, MuSiQue~\citep{musique-dataset}, and 2WikiMQA~\citep{2wikimqa_dataset}. Additionally, we incorporate synthetic QA tasks from RULER—SquadQA~\citep{squad-dataset}, HotpotQA~\citep{hotpotqa-dataset}, and MuSiQue~\citep{musique-dataset}—at varying context lengths ($4K$, $8K$, $16K$, and $32K$ tokens)—to further analyze the model's length extrapolation abilities.
\paragraph{LongBenchV2} To explore the potential of \ours in more diverse and longer-context scenarios, we further evaluate it on the full suite of tasks in LongBenchV2~\citep{bai2025longbenchv2deeperunderstanding}. This benchmark covers a wide range of long-context tasks, including question answering, abstractive summarization, and in-context learning, with input lengths spanning below $32K$ words, between $32K$ and $128K$ words, and beyond $128K$ words.
\paragraph{Open LLM Leaderboard} Prior works~\citep{short_forgetting,longred,chen2025longpo} note that aligning models for long-context tasks may forget their short-context capabilities. To assess short-context performance retention of different methods, we adopt evaluations from the Open LLM Leaderboard\footnote{\url{https://huggingface.co/spaces/open-llm-leaderboard/open_llm_leaderboard}}~\citep{open-llm-leaderboard-v2}. These include widely used tasks such as MMLU-Pro~\citep{wang2024mmluprorobustchallengingmultitask}, MATH~\citep{hendrycks2021measuringmathematicalproblemsolving}, GPQA~\citep{rein2023gpqagraduatelevelgoogleproofqa}, IFEval~\citep{ifeval}, and BBH~\citep{suzgun2022challengingbigbenchtaskschainofthought}, which valuate general knowledge, mathematical reasoning, scientific (chemistry, biology, physics) knowledge, instruction following, and complex reasoning, respectively. 
\paragraph{NIAH-Plus} As described in Section~\ref{sec:depth_analysis}, to better understand how different training strategies affect the contextual knowledge utilization capability, we employ the NIAH-Plus~\citep{zhao-etal-2024-longagent} benchmark. This needle-in-a-haystack QA benchmark includes both single-document and multi-document settings, and is designed to directly probe a model’s capacity for context-aware retrieval and multi-step reasoning.

\subsubsection{Evaluation Settings}
\paragraph{Prompts for Evaluation} For QA tasks in LongBenchV1 and RULER, we use the same prompt template as employed during data construction and model training, which is illustrated in Figure~\ref{fig:prompt_template}. For all other benchmarks mentioned in this paper, we adopt their publicly released prompts. Specifically, as shown in Figure~\ref{fig:prompt_template_v2}, for LongBenchV2, we employ a single-stage chain-of-thought prompting strategy to generate answers directly, differing from the official two-stage evaluation protocol. For the Open LLM Leaderboard and NIAH-Plus benchmarks, we follow the default prompts used in the official implementation code (lm-evaluation-harness\footnote{\url{https://github.com/EleutherAI/lm-evaluation-harness}} and NIAHaystack-PLUS\footnote{\url{https://github.com/zuucan/NeedleInAHaystack-PLUS}} repositories).
\begin{figure}[t]
    \centering
    \includegraphics[width=0.9\textwidth]{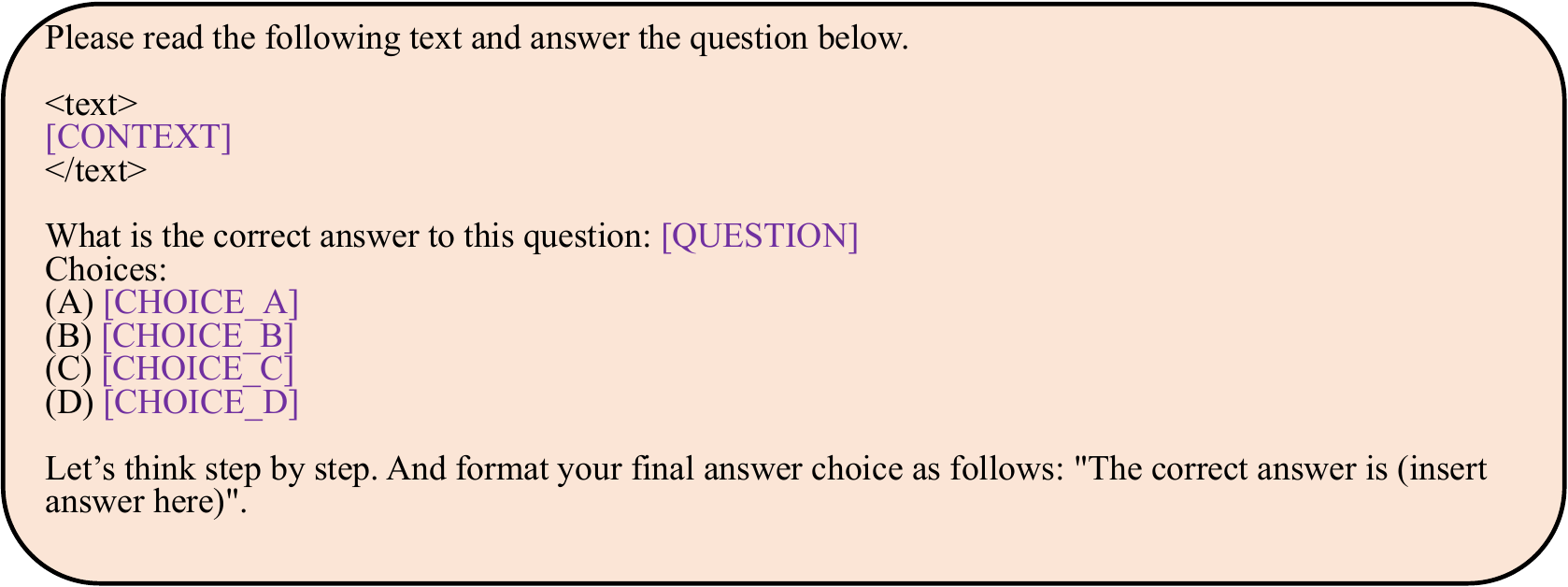}
    \caption{Prompt template used for LongBenchV2 evaluation, adapted from~\citet{bai2025longbenchv2deeperunderstanding}}
    \label{fig:prompt_template_v2}
\end{figure}
\paragraph{Decoding hyperparameters} We use greedy decoding for evaluation on LongBenchV1, RULER, and NIAH-Plus. For other benchmarks, we follow the official decoding settings with temperature values of 0.1 and 0 for LongBenchV2 and the Open LLM Leaderboard, respectively. The error analysis for LongbenchV2 is provided in Appendix~\ref{sec:std}.
\section{Standard Deviation of LongBenchV2}
\label{sec:std}
\begin{figure}[h]
    \centering\includegraphics[width=1\linewidth]{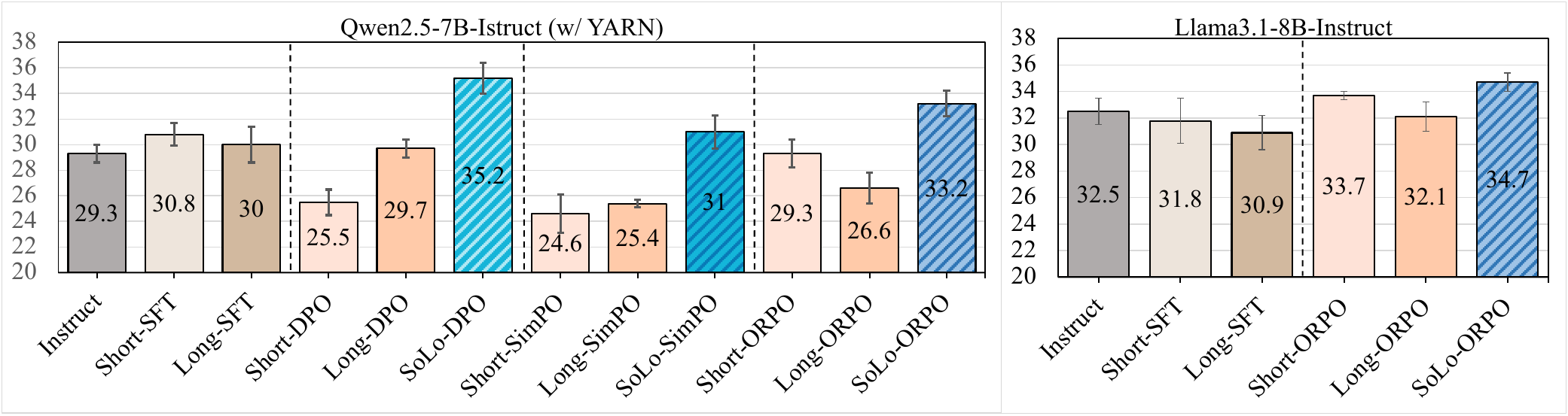}
    \caption{Overall Performance on LongbenchV2. \textbf{1.} We report the average score with standard deviation across 5 evaluation runs for each model. \textbf{2.} All of these metrics are reasonable.}
    \label{fig:error_bar}
\end{figure}
As mentioned in Appendix~\ref{sec:evalaution_details}, we use greedy decoding (temperature $=0$) for all benchmarks except LongbenchV2, where the temperature is set to $0.1$. Here, we report the standard deviation of LongbenchV2 in Figure~\ref{fig:error_bar} and Table~\ref{tab:lbV2_short_ben_results}, where all of these metrics are reasonable.

\section{More Detailed Benchmark Results}
\label{sec:detailed_bench_res}
We present more detailed results for our experiments presented in Section~\ref{sec:exp_res}. 
\paragraph{Detailed results on QA tasks from LongbenchV1 and RULER} Since the training dataset we use, MuSiQue, is designed for multi-hop QA tasks, in Section~\ref{sec:exp_main_res}, we primarily evaluate various methods on the QA tasks from LongBenchV1 and RULER. This allows us to examine model performance on real-world scenarios and across varying input lengths. Results are presented in Table~\ref{tab:detailed_res_of_lbv1} and Table~\ref{tab:detailed_res_of_ruler}.

\begin{table}[ht]
    \centering
    \caption{Detailed Results of QA tasks from LongBenchV1}
    \adjustbox{max width=\textwidth}{
    \begin{tabular}{cccccccccc}
    \toprule
    \multirow{2}{*}{\textbf{Model}}& \multicolumn{4}{c}{\textbf{Single-Doc QA}}& \multicolumn{4}{c}{\textbf{Multi-Doc QA}} & \multirow{2}{*}{\textbf{AVG.}} \\
    \cmidrule(lr){2-5} \cmidrule(lr){6-9} 
    &\textbf{NarrativeQA}&\textbf{Qasper}&\textbf{MultiFieldQA-En}&\textbf{Avg.}&\textbf{HotpotQA}&\textbf{2WikiMQA}&\textbf{MuSiQue}&\textbf{Avg.}&\\
        \midrule
        \multicolumn{10}{c}{\textbf{Qwen2.5-7B-Instruct}} \\
        \midrule
        Instruct&15.8 & 31.3 & 41.0 & 29.4 & 39.6 & 48.9 & 29.7 & 39.4 & 34.4 \\
        $M^{\scalebox{\scriptscalefactor}{SFT}}_{\scalebox{\scriptscalefactor}{short}}$&13.4 & 34.0 & 39.3 & 28.9 & 47.3 & 63.3 & 34.5 & 48.4 & 38.6 \\
        $M^{\scalebox{\scriptscalefactor}{SFT}}_{\scalebox{\scriptscalefactor}{long}}$&21.3 & 39.4 & 43.7 & 34.8 & 55.4 & 65.7 & 46.5 & 55.8 & 45.3 \\
        \hdashline
        $M^{\scalebox{\scriptscalefactor}{DPO}}_{\scalebox{\scriptscalefactor}{short}}$&17.2 & 41.3 & 45.4 & 34.6 & 50.9 & 68.0 & 36.4 & 51.8 & 43.2 \\
        $M^{\scalebox{\scriptscalefactor}{DPO}}_{\scalebox{\scriptscalefactor}{long}}$&22.6 & 41.1 & 43.4 & 35.7 & 60.2 & 66.3 & 48.0 & 58.2 & 46.9 \\
        \rowcolor{Gray}$M^{\scalebox{\scriptscalefactor}{SoLo}}_{\scalebox{\scriptscalefactor}{DPO}}$&\textbf{26.1} & \textbf{43.0} & 44.8 & \underline{38.0} & 58.8 & 63.3 & \textbf{50.7} & 57.6 & 47.8 \\
        \hdashline
        $M^{\scalebox{\scriptscalefactor}{SimPO}}_{\scalebox{\scriptscalefactor}{short}}$&20.3 & {41.4} & 42.5 & 34.7 & 55.5 & 67.1 & 38.1 & 53.6 & 44.1 \\
        $M^{\scalebox{\scriptscalefactor}{SimPO}}_{\scalebox{\scriptscalefactor}{long}}$&16.8 & 41.0 & 44.8 & 34.2 & 56.3 & 64.9 & 43.6 & 54.9 & 44.6 \\
        \rowcolor{Gray}$M^{\scalebox{\scriptscalefactor}{SoLo}}_{\scalebox{\scriptscalefactor}{SimPO}}$&\underline{25.5} & \underline{42.8} & \underline{46.0} & \textbf{38.1} & \underline{61.5} & \underline{68.4} & 46.0 & \underline{58.6} & \underline{48.4} \\
        \hdashline
        $M^{\scalebox{\scriptscalefactor}{ORPO}}_{\scalebox{\scriptscalefactor}{short}}$&13.4 & 34.0 & 39.3 & 28.9 & 47.3 & 63.3 & 34.5 & 48.4 & 38.6 \\
        $M^{\scalebox{\scriptscalefactor}{ORPO}}_{\scalebox{\scriptscalefactor}{long}}$&21.3 & 39.4 & 43.7 & 34.8 & 55.4 & 65.7 & 46.5 & 55.8 & 45.3 \\
        \rowcolor{Gray}$M^{\scalebox{\scriptscalefactor}{SoLo}}_{\scalebox{\scriptscalefactor}{ORPO}}$&{24.9} & {41.4} & \textbf{46.6} & {37.6} & \textbf{64.3} & \textbf{71.7} & \underline{48.1} & \textbf{61.4} & \textbf{49.5} \\
        \midrule
        \multicolumn{10}{c}{\textbf{LLama3.1-8B-Instruct}} \\
        \midrule
        Instruct&12.3 & 37.3 & 41.5 & 30.4 & 54.8 & 59.8 & 33.4 & 49.3 & 39.8 \\
        $M^{\scalebox{\scriptscalefactor}{SFT}}_{\scalebox{\scriptscalefactor}{short}}$&19.1 & 37.1 & 42.9 & 33.0 & \textbf{59.0} & 65.4 & 44.4 & \underline{56.3} & 44.6 \\
        $M^{\scalebox{\scriptscalefactor}{SFT}}_{\scalebox{\scriptscalefactor}{long}}$&\underline{20.3} & 39.6 & \underline{45.1} & 35.0 & 56.6 & \underline{68.3} & \underline{47.1} & \textbf{57.3} & \underline{46.1}\ \\
        \hdashline
        $M^{\scalebox{\scriptscalefactor}{ORPO}}_{\scalebox{\scriptscalefactor}{short}}$&17.1 & 38.8 & 43.3 & 33.1 & 55.4 & 67.1 & 42.7 & 55.1 & 44.1 \\
        $M^{\scalebox{\scriptscalefactor}{ORPO}}_{\scalebox{\scriptscalefactor}{long}}$&19.6&\textbf{40.1}&\textbf{46.7}&\textbf{35.4}&\underline{57.2}&\underline{68.3}&40.8&55.4&45.4\\
        \rowcolor{Gray}$M^{\scalebox{\scriptscalefactor}{SoLo}}_{\scalebox{\scriptscalefactor}{ORPO}}$&\textbf{21.5}&\underline{40.0}&44.2&\underline{35.2}&54.8&\textbf{69.0}&\textbf{48.2}&\textbf{57.3}&\textbf{46.3}\\
        \bottomrule
    \end{tabular}
    }
    \label{tab:detailed_res_of_lbv1}
\end{table}

\begin{table}[ht]
    \centering
    \caption{Detailed Results of QA tasks from RULER}
    \adjustbox{max width=\textwidth}{
    \begin{tabular}{l *{16}{c}}
        \toprule
        \multirow{2}{*}{\textbf{Model}} & \multicolumn{5}{c}{\textbf{SquadQA(OOD)}} & \multicolumn{5}{c}{\textbf{HotpotQA(OOD)}} & \multicolumn{5}{c}{\textbf{MuSiQue (InDomain)}} & \multirow{2}{*}{\textbf{AVG.}} \\
        \cmidrule(lr){2-6} \cmidrule(lr){7-11} \cmidrule(lr){12-16}
              & \textbf{4k} & \textbf{8k} & \textbf{16k} & \textbf{32k} & \textbf{avg.} & \textbf{4k.} & \textbf{8k} & \textbf{16k} & \textbf{32k} & \textbf{avg.} & \textbf{4k} & \textbf{8k} & \textbf{16k} & \textbf{32k} & \textbf{avg.}\\
        \midrule
        \multicolumn{17}{c}{\textbf{Qwen2.5-7B-Instruct}} \\
        \midrule
        Instruct&61.3 & 49.5 & 37.0 & 35.2 & 45.8 & 59.6 & 61.4 & 51.3 & 47.7 & 55.0 & 40.9 & 39.4 & 24.6 & 20.8 & 31.4 & 44.0 \\
        $M^{\scalebox{\scriptscalefactor}{SFT}}_{\scalebox{\scriptscalefactor}{short}}$&72.3 & 57.5 & 41.4 & 26.5 & 49.4 & 68.2 & 66.0 & 52.6 & 44.8 & 57.9 & 50.8 & 46.5 & 33.0 & 24.2 & 38.6 & 48.7 \\
        $M^{\scalebox{\scriptscalefactor}{SFT}}_{\scalebox{\scriptscalefactor}{long}}$&70.9 & 64.4 & 61.8 & 54.1 & 62.8 & 71.0 & 67.7 & 68.4 & 66.7 & 68.5 & 55.9 & 52.1 & 44.9 & 36.5 & 47.3 & 59.5 \\
        \hdashline
        $M^{\scalebox{\scriptscalefactor}{DPO}}_{\scalebox{\scriptscalefactor}{short}}$&\textbf{76.2} & 64.0 & 41.9 & 47.8 & 57.5 & 74.4 & 72.2 & 58.9 & 59.7 & 66.3 & 62.1 & 53.7 & 35.0 & 33.1 & 46.0 & 56.6 \\
        $M^{\scalebox{\scriptscalefactor}{DPO}}_{\scalebox{\scriptscalefactor}{long}}$&75.6 & 66.7 & 62.5 & 52.7 & 64.4 & \underline{74.6} & 69.8 & {70.9} & 67.0 & 70.6 & \underline{62.7} & 56.3 & 48.5 & 39.8 & {51.8} & 62.2 \\
        \rowcolor{Gray}$M^{\scalebox{\scriptscalefactor}{SoLo}}_{\scalebox{\scriptscalefactor}{DPO}}$&{70.7} & {64.4} & {64.8} & \textbf{62.0} & {65.5} & {70.5} & {73.2} & \underline{71.5} & \underline{67.3} & {70.6} & 58.2 & {55.8} & \underline{51.7} & \underline{44.0} & \underline{52.4} & {62.8} \\
        \hdashline
        $M^{\scalebox{\scriptscalefactor}{SimPO}}_{\scalebox{\scriptscalefactor}{short}}$&\underline{75.8} & 62.7 & 39.4 & 50.0 & 57.0 & \textbf{74.8} & 70.1 & 53.8 & 62.6 & 65.3 & 61.7 & 54.9 & 33.2 & 33.8 & 45.9 & 56.0 \\
        $M^{\scalebox{\scriptscalefactor}{SimPO}}_{\scalebox{\scriptscalefactor}{long}}$&74.0 & 64.0 & 45.0 & 46.4 & 57.4 & 74.1 & 73.8 & 62.7 & 61.1 & 67.9 & 61.2 & 54.7 & 39.6 & 34.0 & 47.4 & 57.5 \\
        \rowcolor{Gray}$M^{\scalebox{\scriptscalefactor}{SoLo}}_{\scalebox{\scriptscalefactor}{SimPO}}$&75.2 & \underline{67.6} & \textbf{66.9} & \underline{60.6} & \textbf{67.6} & 74.3 & \underline{73.9} & \textbf{71.8} & \underline{69.9} & \textbf{72.5} & 58.1 & \underline{56.5} & {49.3} & {42.9} & 51.7 & \underline{63.9} \\
        \hdashline
        $M^{\scalebox{\scriptscalefactor}{ORPO}}_{\scalebox{\scriptscalefactor}{short}}$&75.2 & 64.3 & 48.8 & 45.8 & 58.5 & 73.3 & 69.8 & 62.8 & 61.5 & 66.9 & 59.0 & 52.1 & 40.8 & 32.5 & 46.1 & 57.1 \\
        $M^{\scalebox{\scriptscalefactor}{ORPO}}_{\scalebox{\scriptscalefactor}{long}}$&70.5 & 61.2 & 48.4 & 45.9 & 56.5 & 68.7 & 69.0 & 63.4 & 59.5 & 65.2 & 55.4 & 49.5 & 38.1 & 31.4 & 43.6 & 55.1 \\
        \rowcolor{Gray}$M^{\scalebox{\scriptscalefactor}{SoLo}}_{\scalebox{\scriptscalefactor}{ORPO}}$&74.8 & \textbf{67.8} & \underline{65.6} & 59.2 & \underline{66.9} & 73.3 & \textbf{74.8} & 70.7 & 67.2 & \underline{71.5} & \textbf{64.2} & \textbf{62.4} & \textbf{55.7} & \textbf{45.6} & \textbf{57.0} & \textbf{65.1} \\
        \midrule
        \multicolumn{17}{c}{\textbf{LLama3.1-8B-Instruct}} \\
        \midrule
        Instruct& 67.4 & 53.8 & 44.9 & 41.7 & 52.0 & 68.1 & 61.4 & 57.3 & 49.9 & 59.2 & 39.4 & 32.5 & 26.3 & 15.1 & 28.3 & 46.5 \\
        $M^{\scalebox{\scriptscalefactor}{SFT}}_{\scalebox{\scriptscalefactor}{short}}$& \textbf{70.7} & \textbf{66.0} & 61.7 & 55.4 & 63.5 & 67.6 & \textbf{67.8} & \underline{66.2} & 62.5 & \underline{66.0} & \textbf{56.7} & \underline{49.1} & \underline{47.5} & 38.6 & \underline{48.0} & \underline{59.2} \\
        $M^{\scalebox{\scriptscalefactor}{SFT}}_{\scalebox{\scriptscalefactor}{long}}$&69.4 & 63.6 & 64.6 & 58.9 & \underline{64.1} & 66.6 & 64.6 & 63.7 & 62.0 & 64.2 & 55.0 & 48.9 & 44.3 & \underline{39.5} & 46.9 & 58.4 \\
        \hdashline
        $M^{\scalebox{\scriptscalefactor}{ORPO}}_{\scalebox{\scriptscalefactor}{short}}$&\underline{70.0} & \underline{65.5} & \textbf{67.7} & 52.0 & 63.8 & 67.9 & 64.8 & 65.6 & \underline{63.0} & 65.3 & 54.2 & 47.7 & 45.7 & 36.2 & 46.0 & 58.4 \\
        $M^{\scalebox{\scriptscalefactor}{ORPO}}_{\scalebox{\scriptscalefactor}{long}}$&68.5 & 65.1 & \underline{64.9} & \underline{59.2} & \textbf{64.4} & \underline{68.5} & 66.7 & \textbf{66.7} & 61.4 & 65.8 & 52.7 & 48.5 & 45.1 & 38.9 & 46.3 & 58.9 \\
        \rowcolor{Gray}$M^{\scalebox{\scriptscalefactor}{SoLo}}_{\scalebox{\scriptscalefactor}{ORPO}}$&67.0 & 63.5 & 64.5 & \textbf{59.8} & 63.7 & \textbf{70.3} & \underline{67.0} & 66.0 & \textbf{66.8} & \textbf{67.5} & \underline{55.8} & \textbf{57.1} & \textbf{49.8} & \textbf{48.1} & \textbf{52.7} & \textbf{61.3} \\
        \bottomrule
    \end{tabular}
    }
    \label{tab:detailed_res_of_ruler}
\end{table}

\begin{table}[th]\small
    \centering
    \caption{Comparative performance of \ours and Expand-Long-PO on \textit{NIAH-Plus} (within $128K$ context size). As shown in \textcolor{blue!50}{($\uparrow$)}, \ours consistently outperforms Expand-Long-PO, validating our decoupled short-to-long preference optimization approach. }
    \begin{tabular}{llll}
    \toprule
         \textbf{Model}&\textbf{Single-document QA}&\textbf{Multi-document  QA}&\textbf{AVG.}\\
    \midrule
    \multicolumn{4}{c}{\textbf{Qwen-2.5-7B-Instruct}}\\
    \midrule
    Instruct&35.66&52.63&44.14\\
    \hdashline
    $M^{\scalebox{\scriptscalefactor}{DPO}}_{\scalebox{\scriptscalefactor}{expand-long}}$&55.98&68.02&62.00\\
    \rowcolor{Gray}$M^{\scalebox{\scriptscalefactor}{DPO}}_{\scalebox{\scriptscalefactor}{SoLo}}$&59.35 (\textcolor{blue!50}{$\uparrow$ 3.37})&\underline{71.76} (\textcolor{blue!50}{$\uparrow$ 3.74})&65.56 (\textcolor{blue!60}{$\uparrow$ 3.56})\\
    \hdashline
    $M^{\scalebox{\scriptscalefactor}{SimPO}}_{\scalebox{\scriptscalefactor}{expand-long}}$&51.81&53.61&52.71\\
    \rowcolor{Gray}$M^{\scalebox{\scriptscalefactor}{SimPO}}_{\scalebox{\scriptscalefactor}{SoLo}}$&\underline{60.85} (\textcolor{blue!50}{$\uparrow$ 9.04})&\textbf{72.05} (\textcolor{blue!50}{$\uparrow$ 18.44})&\underline{66.45} (\textcolor{blue!60}{$\uparrow$ 13.74})\\
    \hdashline
         $M^{\scalebox{\scriptscalefactor}{ORPO}}_{\scalebox{\scriptscalefactor}{expand-long}}$&59.60&69.92&64.76\\
         \rowcolor{Gray}$M^{\scalebox{\scriptscalefactor}{ORPO}}_{\scalebox{\scriptscalefactor}{SoLo}}$&\textbf{61.64} (\textcolor{blue!50}{$\uparrow$ 2.04})&71.46 (\textcolor{blue!50}{$\uparrow$ 1.54})&\textbf{66.55} (\textcolor{blue!50}{$\uparrow$ 1.79})\\
    \midrule
    \multicolumn{4}{c}{\textbf{LLama3.1-8B-Instruct}}\\
    \midrule
    Instruct&32.04	&32.8&32.42\\
    \hdashline
    $M^{\scalebox{\scriptscalefactor}{ORPO}}_{\scalebox{\scriptscalefactor}{expand-long}}$&43.69&42.57&43.13\\
    \rowcolor{Gray}$M^{\scalebox{\scriptscalefactor}{ORPO}}_{\scalebox{\scriptscalefactor}{SoLo}}$&\textbf{43.86} (\textcolor{blue!50}{$\uparrow$ 0.17})&\textbf{52.96} (\textcolor{blue!50}{$\uparrow$ 10.39})&\textbf{48.41} (\textcolor{blue!50}{$\uparrow$ 5.28})\\
    \bottomrule
    \end{tabular}
    \label{tab:full_naih}
\end{table}

\begin{figure}[htbp]
    \centering
    
    \begin{subfigure}[b]{1\textwidth}
        \centering
        \includegraphics[width=\textwidth]{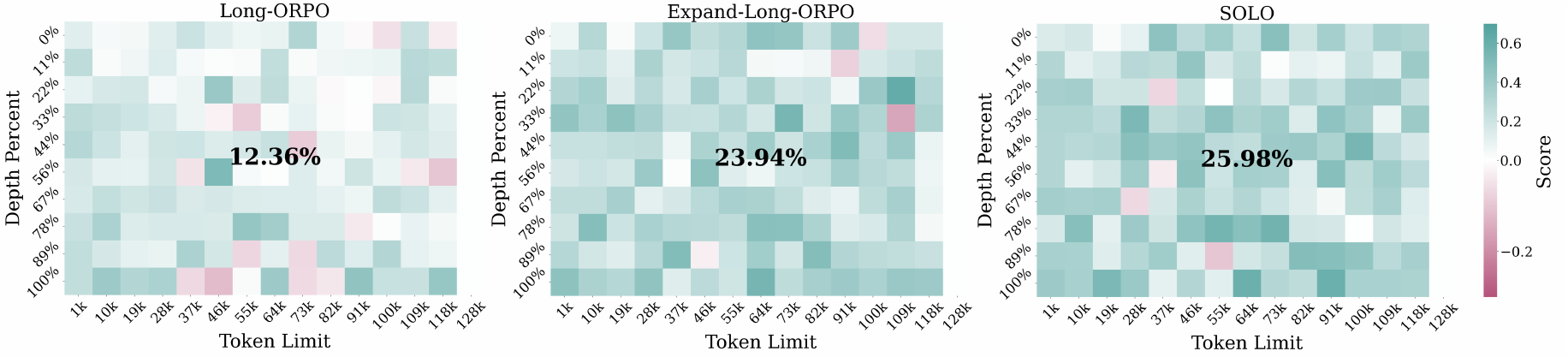}
        \caption{Single-document QA setting.}
        \label{fig:niah_all_single}
    \end{subfigure}
    
    \vspace{1em} 
    
    \begin{subfigure}[b]{1\textwidth}
        \centering
        \includegraphics[width=\textwidth]{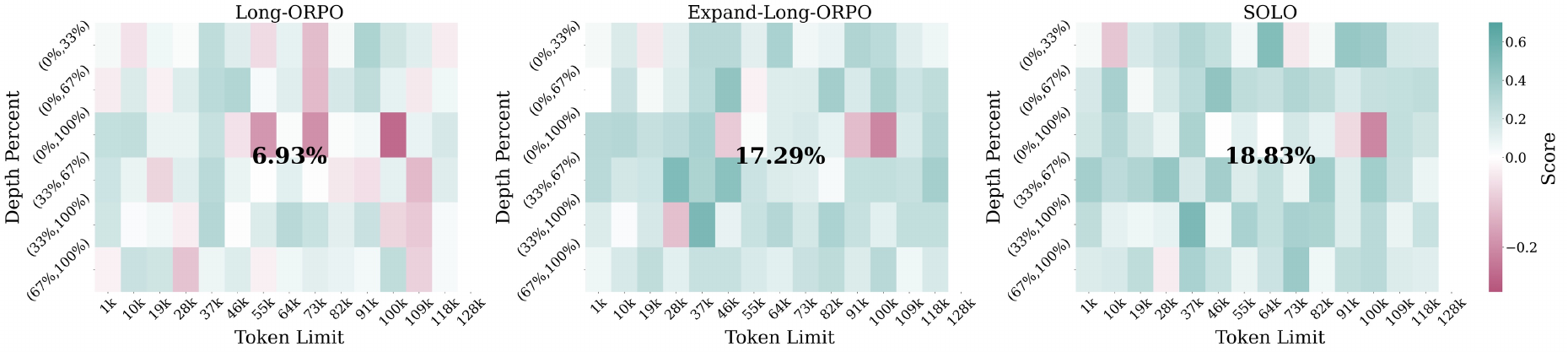}
        \caption{Multi-document QA setting.}
        \label{fig:niah_all_multi}
    \end{subfigure}
    
    \caption{Comparison of performance improvements achieved by various ORPO methods relative to Qwen2.5-7B on the \textit{NIAH-Plus}~\citep{zhao-etal-2024-longagent}. Expand-Long-ORPO and \ours demonstrate significantly greater improvements than Long-ORPO, highlighting the effectiveness of short-to-long alignment. Moreover, \ours provides greater gains than Expand-Long-ORPO, validating the effectiveness of our decoupled approach.}
    \label{fig:niah_all}
\end{figure}

\paragraph{More detailed results on NIAH-Plus} In Section~\ref{sec:depth_analysis}, to validate whether the decoupled approach may more effectively enhance the model's ability to locate contextual knowledge, we evaluate Expand-Long-PO (a non-decoupled approach) and \ours (our decoupled approach) on NIAH-Plus with full results presented in Table~\ref{tab:full_naih}. \ours consistently outperforms Expand-Long-PO across all evaluation scenarios and preference optimization algorithms, demonstrating the effectiveness of our decoupling-based short-to-long preference optimization approach. Figure~\ref{fig:niah_all} presents heatmaps of the performance gains of various ORPO
over Qwen2.5-7B. {SoLo-ORPO, powered by the SoLoPO's decoupling strategy, surpasses Long-ORPO and Expand-Long ORPO on NIAH-Plus. It achieves superior information retrieval and utilization across various depths and context lengths in both single-hop and multi-hop settings, fundamentally improving contextual knowledge localization.}

\section{More Experimental Analysis}
\label{sec:detailsed_exp_analysis}
\subsection{More analysis about SoLo-DPO/SimPO on LongbenchV2}
\label{sec:more_analysis_on_dpo}
Table~\ref{tab:lbV2_short_ben_results} presents results of Qwen2.5-7B (w/ YARN~\citep{peng2024yarn}) on LongBenchV2. For Qwen2.5-7B (w/ YARN), \ours consistently outperforms original PO methods across all evaluation context lengths and difficulty levels. Specifically, (1) SoLo-ORPO also surpasses vanilla PO across all dimensions; (2) SoLo-DPO achieves the best overall performance, particularly on contexts $\ge32$ and hard samples, likely due to the reference model ensuring better generalization; (3) SoLo-SimPO shows relatively weaker performance, possibly due to its reward model relies on normalized prediction probabilities, which can underperform on long-context evaluations like perplexity observed by~\citet{LongCEloss}.

\subsection{Training Dynamics of Different SoLo-RA Approaches (\textit{chosen-only} vs. \textit{both})}
\label{sec:trainin_dy}
\begin{figure}[th]
\begin{subfigure}[t]{0.48\textwidth}
\raggedleft
\raisebox{0.8pt}{\includegraphics[width=\textwidth]{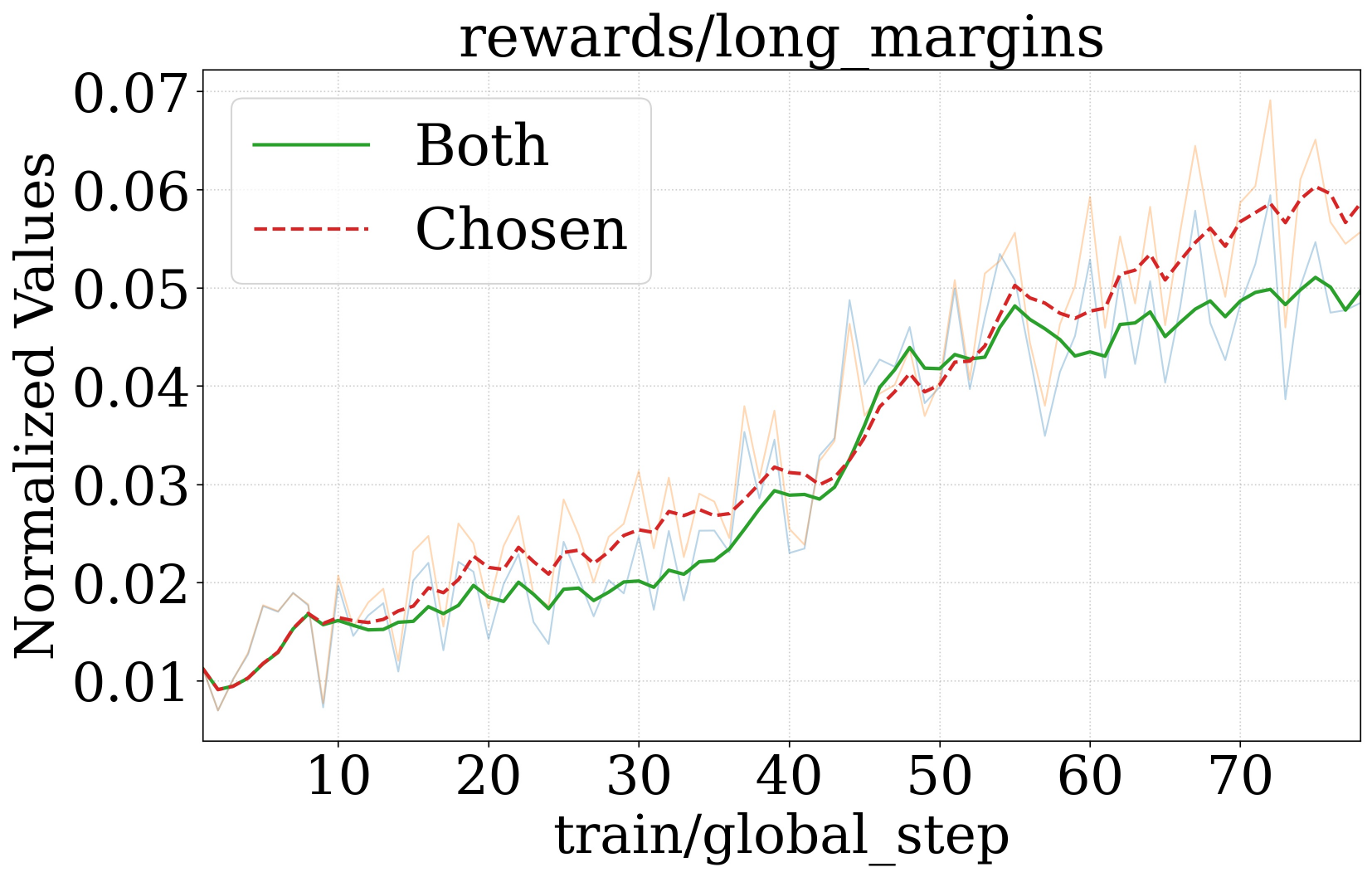}}
\caption{$r_{orpo}(x_{long},y_w)-r_{orpo}(x_{long},y_l)$. \label{fig:s2l_orpo_long_margins}}
\end{subfigure}%
~
\begin{subfigure}[t]{0.48\textwidth}
\raggedright
\raisebox{2pt}
{\includegraphics[width=\textwidth]{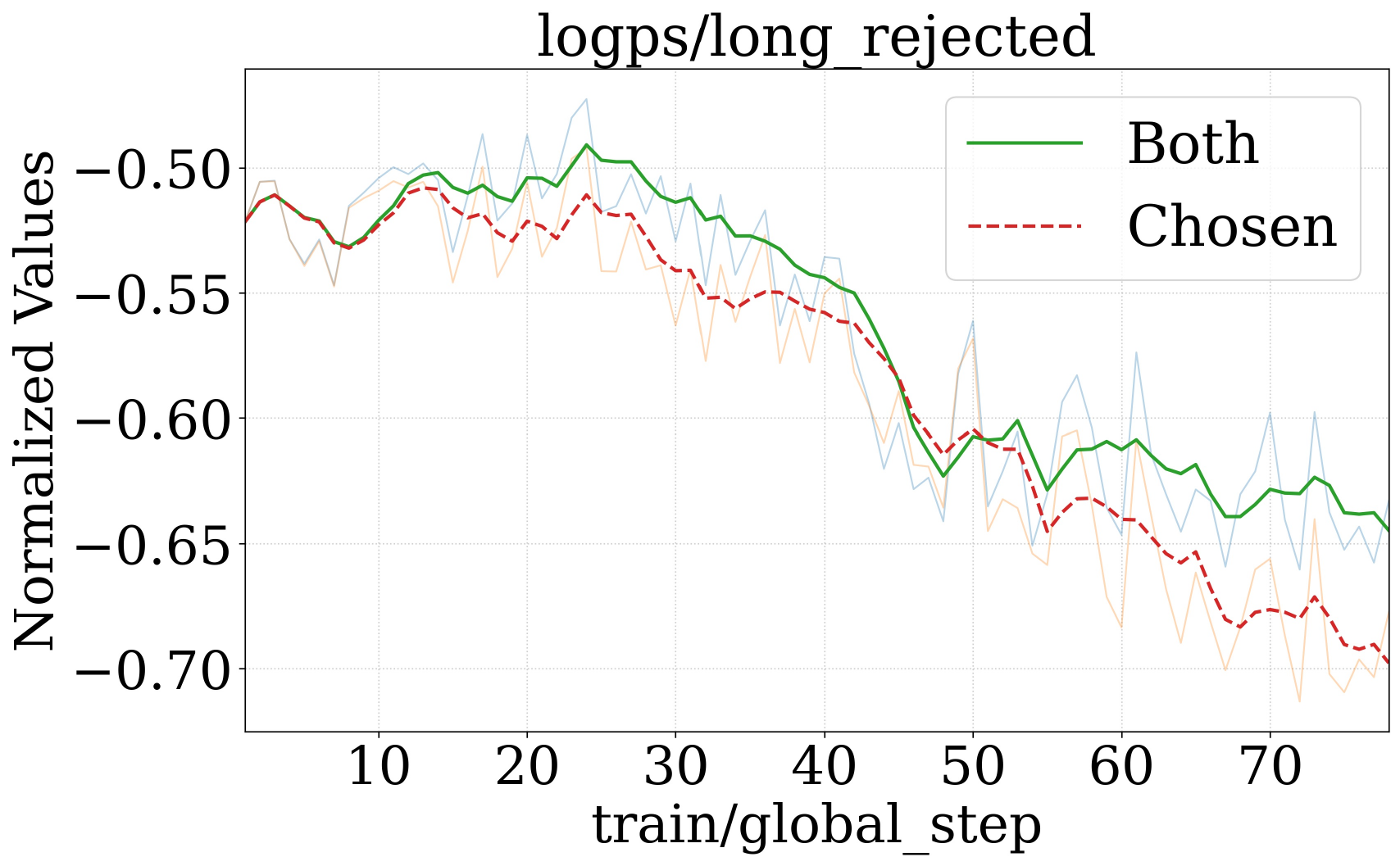}}
\vspace{-1.0em}
\caption{Log prob. of $p_\theta(y_l\mid x_{long})$. \label{fig:s2l_orpo_long_rejected}}
\end{subfigure}
\caption{Changes of reward margins and log prob. of rejected response during SoLo-ORPO training given long contexts. Compared to the \textit{both} approach,  \textit{chosen-only} short-to-long reward alignment achieves (a) larger reward margins, ensuring higher reward accuracy, while (b) simultaneously applying sufficient penalties to rejected responses, reducing their prediction probability. The chosen-only approach exhibits more precise reward modeling, ultimately achieving better alignment outcomes.}
\label{fig:s2l_chosen_and_both} 
\end{figure}

As demonstrated in Figure~\ref{fig:s2l_orpo_long_margins}, the margin curves of the \textit{both} SoLo-RA consistently lie beneath those of the \textit{chosen-only} approach, with its ultimately converged margin values being substantially lower. This observation indicates that the \textit{chosen-only} SoLo-RA exhibits superior fitting capability along the $x_{long}$ dimension. \ref{fig:s2l_orpo_long_rejected} reveals that during the initial optimization phase, the \textit{both} SoLo-RA induces an anomalous transient increase in the probability of $p_\theta(y_l\mid x_{long})$, which should theoretically follow a monotonic decreasing trend. This suboptimal convergence behavior ultimately leads to significantly inferior optimization outcomes compared to the \textit{chosen-only} approach as shown in Figure~\ref{fig:perfor_both_chosen}.
\subsection{Ablation Experiments on SoLo-RA}
\label{sec:ablation_solo_ra}
\begin{wraptable}{r}{6.5cm}\small
    \centering
    \caption{Ablation studies for SoLo-RA on \textit{NIAH-Plus} (within $128K$ context size). }
    \begin{tabular}{llll}
    \toprule
         \textbf{Model}&\textbf{S-Doc QA}&\textbf{M-Doc QA}&\textbf{AVG.}\\
    \midrule
    \multicolumn{4}{c}{\textbf{Qwen-2.5-7B-Instruct}}\\
    \midrule
    Instruct&35.66&52.63&44.14\\
    \hdashline
    $M^{\scalebox{\scriptscalefactor}{DPO}}_{\scalebox{\scriptscalefactor}{short}}$&51.25&64.70&57.98\\
    $M^{\scalebox{\scriptscalefactor}{DPO}}_{\scalebox{\scriptscalefactor}{expand-long}}$&55.98&68.02&62.00\\
    \rowcolor{Gray}$M^{\scalebox{\scriptscalefactor}{DPO}}_{\scalebox{\scriptscalefactor}{SoLo}}$&\textbf{59.35} &\textbf{71.76} &\textbf{65.56} \\
    \hdashline
    $M^{\scalebox{\scriptscalefactor}{SimPO}}_{\scalebox{\scriptscalefactor}{short}}$&50.84&63.59&57.22\\
    $M^{\scalebox{\scriptscalefactor}{SimPO}}_{\scalebox{\scriptscalefactor}{expand-long}}$&51.81&53.61&52.71\\
    \rowcolor{Gray}$M^{\scalebox{\scriptscalefactor}{SimPO}}_{\scalebox{\scriptscalefactor}{SoLo}}$&\textbf{60.85} &\textbf{72.05} &\textbf{66.45} \\
    \hdashline
    $M^{\scalebox{\scriptscalefactor}{ORPO}}_{\scalebox{\scriptscalefactor}{short}}$&45.02&59.17&52.10\\
         $M^{\scalebox{\scriptscalefactor}{ORPO}}_{\scalebox{\scriptscalefactor}{expand-long}}$&59.60&69.92&64.76\\
         \rowcolor{Gray}$M^{\scalebox{\scriptscalefactor}{ORPO}}_{\scalebox{\scriptscalefactor}{SoLo}}$&\textbf{61.64} &\textbf{71.46}&\textbf{66.55} \\
    \midrule
    \multicolumn{4}{c}{\textbf{LLama3.1-8B-Instruct}}\\
    \midrule
    Instruct&32.04	&32.80&32.42\\
    \hdashline
    $M^{\scalebox{\scriptscalefactor}{ORPO}}_{\scalebox{\scriptscalefactor}{short}}$&37.20&37.77&37.49\\
    $M^{\scalebox{\scriptscalefactor}{ORPO}}_{\scalebox{\scriptscalefactor}{expand-long}}$&43.69&42.57&43.13\\
    \rowcolor{Gray}$M^{\scalebox{\scriptscalefactor}{ORPO}}_{\scalebox{\scriptscalefactor}{SoLo}}$&\textbf{43.86}&\textbf{52.96}&\textbf{48.41} \\
    \bottomrule
    \end{tabular}
    \label{tab:ablation_naih}
    \vspace{-40pt}
\end{wraptable} 
\label{ablation_study_on_solo_ra}
In Section~\ref{sec:method_loss_analysis}, we hypothesize that SoLo-RA enhances the model’s contextual knowledge localization ability (see Appendix~\ref{sec:on_modeling_solopo} for a detailed discussion). In this section, we conduct ablation studies on NIAH-Plus to further validate this hypothesis.

To conduct the ablation analysis of SoLo-RA, we focus on the performance of models trained with Short-PO, Expand-PO, and SoLoPO on NIAH-Plus, for the following reasons: 
\begin{itemize}[leftmargin=0.5cm]
    \item Removing SoLo-RA from SoLoPO’s optimization objective yields short-PO; however, this also eliminates the $x_{long}$ from the training data, making a direct comparison unable to disentangle the gains from data length on generalization~\cite{From_Short_to_Long}. 
    \item Expand-PO and SoLoPO are trained on identical data except for the absence of $x_{short}$ in Expand-PO, allowing their comparison to minimize potential confounding effects from training data length.
\end{itemize}

As shown in Table~\ref{tab:ablation_naih}, SoLo-PO consistently outperforms both Short-PO and Expand-PO in contextual knowledge localization tasks with a 128K context length across different PO algorithms and base models. These results provide further evidence that SoLo-RA more effectively strengthens the capacity of the model to localize contextual knowledge.

\subsection{Impact of Reward Alignment Coefficient $\alpha$ in SoLo-DPO and SoLo-SimPO}
\label{sec:impact_for_dpo_simpo}
Figure~\ref{fig:reward_alignment_alpha_dpo_simpo} shows the performance of SoLo-DPO and SoLo-SimPO under varying $\alpha$ in the Qwen2.5-7B setting. The results indicate that the optimal $\alpha$ values are 3 for SoLo-DPO and 1 for SoLo-SimPO, and in most cases SoLoPO outperforms Long-PO, suggesting its stability across different $\alpha$ values.
\begin{figure}[ht]
    \centering
    \includegraphics[width=0.5\linewidth]{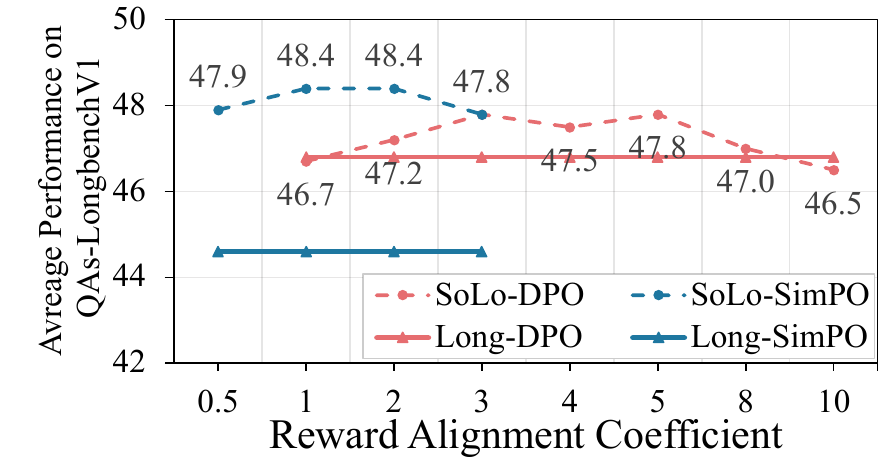}
    \caption{Performance with different $\alpha$ in SoLo-DPO and SoLo-SimPO in the Qwen2.5-7B setting. The optimal values of $\alpha$ for SoLo-DPO and SoLo-SimPO are 3 and 1, respectively.}
    \label{fig:reward_alignment_alpha_dpo_simpo}
\end{figure}
\subsection{Efficiency Analysis of \ours}
\label{sec:efficiency_analysis}
Thanks to the reduced overhead in handling long texts, \ours offers notable advantages over the original algorithms in both computational efficiency and memory usage. In this section, we present detailed empirical comparisons of runtime efficiency and provide a theoretical analysis of the computational speedup. 
\paragraph{Training Implementation Details}
The experimental setup for Figure~\ref{fig:efficiency} adheres to the same training configuration described in Section~\ref{sec:training_details}. For experiments in Section~\ref{sec:exp_eff_adv}, we employ various optimization strategies of ZeRO~\citep{rajbhandari2020zeromemoryoptimizationstraining} to maximize GPU memory utilization while enabling training on longer sequences. The specific training configurations are detailed in Table~\ref{tab:training_config_for_eff_exp}.
\begin{table}[h]\small
    \centering
    \caption{Implementation details and run time of SoLo-ORPO and vanilla ORPO under varying lengths of $x_{long}$. \textbf{1.} the default configuration is Zero stage 3 without offloading \textbf{2.} "2-Stage Forward" indicates sequential forward passes for $(y_w,x_{long})$ and $(y_l,x_{long})$, as opposed to the default concatenated forward strategy. \textbf{3.} "OOM" denotes CUDA Out-of-Memory errors. \textbf{4.} \ours significantly improves the training efficiency of the vanilla ORPO.}
    \begin{tabular}{lcccl}
    \toprule
    \textbf{Length} &\textbf{Method}& \textbf{Offloading} & \textbf{2-Stage Forward for $x_{long}$}&\textbf{Runtime (min) $\downarrow$} \\
    \midrule
         $1K$&vanilla &&&54.00\\
         \hdashline
         \multirow{2}{*}{$4K$}&vanilla&&&72.52\\
            &SoLo&&&66.63 (\textcolor{blue!60}{$\downarrow$ 8\%})\\
        \hdashline
        \multirow{3}{*}{$8K$}&vanilla &\checkmark&&145.11\\
            &SoLo&\checkmark&&109.42 (\textcolor{blue!60}{$\downarrow$ 25\%})\\
            &SoLo&&&83.62 (\textcolor{blue!60}{$\downarrow$ 42\%})\\
        \hdashline
        \multirow{2}{*}{$12K$}&vanilla &\checkmark&\checkmark&235.98\\
            &SoLo&\checkmark&&144.21 (\textcolor{blue!60}{$\downarrow$ 39\%})\\
        \hdashline
        \multirow{2}{*}{$16K$}&vanilla &\checkmark&\checkmark&OOM\\
        &SoLo&\checkmark&&179.20\\
        \hdashline
        \multirow{2}{*}{$19K$}&vanilla &\checkmark&\checkmark&OOM\\
        &SoLo&\checkmark&&205.98\\
    \bottomrule
    \end{tabular}
    \label{tab:training_config_for_eff_exp}
\end{table}
\begin{itemize}[leftmargin=0.5cm]
    \item\textbf{\ours} When the lengths of $x_{long}$ ranging from $4K$ to $16K$ tokens, we employ a two-stage forward mechanism within LLaMAFactory~\citep{zheng-etal-2024-llamafactory} to perform \ours training. Specifically, for the short-context PO, we apply the \texttt{concatenated\_forward}~\footnote{\url{https://github.com/hiyouga/LLaMA-Factory/blob/main/src/llamafactory/train/dpo/trainer.py\#L179}} function directly on $(y_w,x_{short})$ and $(y_l,x_{short})$ to obtain $\log \pi_\theta(y_w\mid x_{short})$ and $\log \pi_\theta(y_l\mid x_{short})$. Subsequently, we conduct a separate forward pass over $(y_w,x_{long})$ to compute $\log \pi_\theta(y_l\mid x_{long})$. Finally, the \ours loss is calculated based on the corresponding loss function in Table~\ref{tab:s2l_app_obj}.
    \item\textbf{Vanilla PO} When the lengths of $x_{long}$ is less than or equal to $8K$ tokens, $(y_w,x_{long})$ and $(y_l,x_{long})$ can be efficiently processed using the \texttt{concatenated\_forward} function, allowing for straightforward loss computation. However, at $x_{long}$ lengths of $12K$ tokens, the use of \texttt{concatenated\_forward} leads to out-of-memory (OOM) errors. Thus, we adopt a 2-stage forward approach—processing $(y_w,x_{long})$ and $(y_l,x_{long})$ sequentially. For sequences as long as $16K$ tokens, even this serialized method becomes infeasible, necessitating the use of sequence parallelism techniques to enable training. For even longer $(x_{long})$ exceeding $16K$ tokens, only a sequence parallelism~\citep{jacobs2023deepspeedulyssesoptimizationsenabling} training strategy becomes feasible. While these alternative approaches mitigate memory constraints, they inevitably increase overall training time.
\end{itemize}

\paragraph{Computation speedup analysis of \ours}
Following \citet{compression_rate}, we define the compression rate $ \rho \in (0, 1] $ as the ratio of the length of $ x_{short} $  to that of $ x_{long} $ . Specifically, if the length of $ x_{long} $ is denoted by $ N $, then the length of $ x_{short} $ is given by $ \rho N $. We focus on the computation incurred by the policy model during training and ignore reference models. Since Transformer operations scale quadratically with sequence length ($\mathrm{FLOPs} \propto n^2$), the total training computation for vanilla PO and \ours can be expressed as:
\begin{equation}
    \mathrm{FLOPs}_{\mathrm{\textit{PO}}} = 2N^2,
    \quad
    \mathrm{FLOPs}_{\mathrm{\textit{\ours}}} = (2\rho ^2 + 1)N^2.
\end{equation}
Consequently, the computation speedup ratio of \ours over vanilla PO can be expressed as:
\begin{equation}
\mathrm{Speedup}(c) = \frac{\mathrm{FLOPs}_{\mathrm{\textit{PO}}}}{\mathrm{FLOPs}_{\mathrm{\textit{\ours}}}} = \frac{2}{2\rho^2 + 1}, \quad \rho \in (0, 1].
\end{equation}
This indicates that \ours achieves computational efficiency gains when $ \rho < \frac{1}{\sqrt{2}} \approx 0.707 $, \textit{i.e.}, when $x_{short}$ is less than approximately 70.7\% of the length of $x_{long}$.
\paragraph{Potential Approaches for Optimizing Training Efficiency on High Compression Rate Tasks} For tasks with high compression ratios, such as long-context machine translation ( 
$\rho$=100\%), the linguistic redundancy hypothesis underlying SoLoPO no longer holds. In this scenario, SoLoPO degrades to vanilla PO and offers no training efficiency gains. It is noteworthy that standard attention mechanisms in LLMs exhibit redundancy, which enables the application of KV compression~\citep{li2024snapkv} or sparse attention~\cite{xiao2024efficient} techniques. We hypothesize that this model-inherent redundancy can be leveraged for representational compression (e.g., gist-token~\citep{zhang2025long}). Consequently, for tasks lacking linguistic redundancy, the short-to-song paradigm could potentially be applied based on the model-inherent redundancy, thereby reducing training time and memory consumption for long-context scenarios.

\begin{table}[h]
    \centering
    \caption{Performance of Qwen2.5-Instruct-14B trained with different methods on LongbenchV1-QAs and RULER-QAs. Across benchmarks, \ours consistently outperforms vanilla PO algorithms.}
    \begin{adjustbox}{max width=\textwidth}
    \begin{tabular}{lcccccccccc}
    \toprule
         \multirow{2}{*}{\textbf{Model}}&\multicolumn{3}{c}{\textbf{QAs-LongBenchV1}}&&\multicolumn{5}{c}{\textbf{QAs-RULER}}&\textbf{Run Time}\\
         \cline{2-4} \cline{6-10}
         &\textbf{S-Doc QA}&\textbf{M-Doc QA}&\textbf{Avg.}&&\textbf{4k}&\textbf{8k}&\textbf{16k}&\textbf{32k}&\textbf{Avg.}&{\textbf{/min($\downarrow$)}}\\
         \midrule
         72B-Instruct & 37.80  & 61.10 & 49.40&&65.70&64.4&61.20&55.0&61.60&-\\
14B-Instruct & 34.40 & 52.07 & 43.30& & 56.60 & 52.27 & 49.69 & 43.86 & 50.60&- \\
\midrule
$M^{\scalebox{\scriptscalefactor}{SFT}}_{\scalebox{\scriptscalefactor}{short}}$ & 36.20 & 59.13 & 47.70& & 64.47 & 59.46 & 55.67 & 46.31 & 56.48&- \\
$M^{\scalebox{\scriptscalefactor}{SFT}}_{\scalebox{\scriptscalefactor}{long}}$  & 34.53 & 61.07 & 47.80& & 65.72 & 61.44 & 52.56 & 43.83 & 55.89&- \\
\midrule
$M^{\scalebox{\scriptscalefactor}{DPO}}_{\scalebox{\scriptscalefactor}{short}}$ & \textbf{37.83} & 65.33 & 51.60& & 71.32 & 65.79 & 63.12 & 55.97 & 64.05&- \\
$M^{\scalebox{\scriptscalefactor}{DPO}}_{\scalebox{\scriptscalefactor}{long}}$  & 36.63 & 66.63 & 51.60 & &\textbf{73.37} & 69.49 & 67.69 & 59.83 & 67.60&249 \\
\rowcolor{Gray}$M^{\scalebox{\scriptscalefactor}{DPO}}_{\scalebox{\scriptscalefactor}{SoLo}}$ 
                                   & \underline{37.80} & \textbf{67.20} & \textbf{53.40} & &\underline{72.95} & \textbf{70.53} & \textbf{68.52} & \textbf{62.03} & \textbf{68.51}&\textbf{172} \\
\midrule
$M^{\scalebox{\scriptscalefactor}{SimPO}}_{\scalebox{\scriptscalefactor}{short}}$ & 37.47 & 64.30 & 50.90& & 71.43 & 65.87 & 63.21 & 55.69 & 64.05&- \\
$M^{\scalebox{\scriptscalefactor}{SimPO}}_{\scalebox{\scriptscalefactor}{long}}$  & 36.77 & \textbf{66.10} & 51.40& & 70.33 & 66.24 & 63.27 & 56.91 & 64.19 &248\\
\rowcolor{Gray}${M^{\scalebox{\scriptscalefactor}{SimPO}}_{\scalebox{\scriptscalefactor}{SoLo}}}$ 
                                      & \textbf{38.40} & \underline{65.67} & \textbf{52.00} & &\textbf{71.79} & \textbf{67.85} & \textbf{66.52} & \textbf{60.47} & \textbf{66.66}&\textbf{169}\\
\midrule
$M^{\scalebox{\scriptscalefactor}{ORPO}}_{\scalebox{\scriptscalefactor}{short}}$ & 39.37 & 63.80 & 51.60& & 68.63 & 63.85 & 61.47 & 52.96 & 61.73&- \\
$M^{\scalebox{\scriptscalefactor}{ORPO}}_{\scalebox{\scriptscalefactor}{long}}$  & 38.80 & 62.73 & 50.80 && 68.69 & 64.19 & 62.16 & 52.83 & 61.97&248 \\
\rowcolor{Gray}${M^{\scalebox{\scriptscalefactor}{ORPO}}_{\scalebox{\scriptscalefactor}{SoLo}}}$ 
                                       & \textbf{39.80} & \textbf{65.70} & \textbf{52.80} & &\textbf{72.45} & \textbf{67.85} & \textbf{65.10} & \textbf{60.78} & \textbf{66.54}&\textbf{170} \\
\bottomrule
    \end{tabular}
\end{adjustbox}
    \label{tab:main_res_on_14b}
\end{table}
\subsection{On the Scalability of \ours to Larger Models}
\label{sec:scaling_exp}
To examine the scalability of \ours to larger models, we conduct experiments on Qwen2.5-Instruct-14B with different preference optimization methods, constrained by available computational resources. The evaluation is performed on the primary benchmarks used in this paper, LongbenchV1-QAs and RULER-QAs. The training and evaluation settings follow those of Qwen2.5-Instruct-7B, except that the parameter $\alpha$ in SoLoPO is set to $1$ and experiments are run on 8$\times$~H20 GPUs with \texttt{LLaMA-Factory(0.9.1)}. As shown in Table~\ref{tab:main_res_on_14b}, \ours consistently outperforms the original PO algorithms across different benchmarks while significantly improving training efficiency.

\section{Theoretical Derivation and Supporting Analysis of SoLoPO}
\label{sec:math_derivation}
In this section, we formulate the theoretical foundation for the short-to-long preference optimization (\ours) method through a novel reward loss function, and develop a distance metric condition for applying the framework for generalized distance metrics. Furthermore, we systematically extend this framework to mainstream preference optimization paradigms, including Direct Preference Optimization (DPO), Simple Preference Optimization (SimPO), and so on, demonstrating its methodological generality.

\subsection{Proof of Lemma \ref{lemma1}}
\begin{lemma} \label{lemma1}
If the function $f(\cdot)$ of equation (\ref{small_op_l}) is convex, then the following inequality holds true: 
    \begin{equation}\label{lemma1_shs}
    \begin{aligned}
         l_{\eta, \gamma}(x_{long}, x_{long}; y_w, y_l) \leq& \frac{1}{3} [l_{3\eta, \frac{\gamma}{3}}(x_{long}, x_{short}; y_w, y_w) \\
         &+ l_{3\eta, \frac{\gamma}{3}}(x_{short},x_{short}; y_w, y_l) + l_{3\eta, \frac{\gamma}{3}}(x_{short}, x_{long}; y_l, y_l)]
    \end{aligned}
    \end{equation}
\end{lemma}


\begin{proof}
\begin{align*}
    &l_{\eta, \gamma}(x_{long}, x_{long}; y_w, y_l) \\
    &= f(\eta \cdot[r_\phi(x_{long}, y_w) - r_\phi(x_{long}, y_l) - \gamma] )\\
    &= f(\eta \cdot[r_\phi(x_{long}, y_w) - r_\phi(x_{short}, y_w) \\
    & \qquad\quad + r_\phi(x_{short}, y_w) - r_\phi(x_{short}, y_w) \\
    &\qquad\quad + r_\phi(x_{short}, y_w)- r_\phi(x_{long}, y_l) - \gamma] )\\
    &= f(\eta \cdot [\Delta_1 + \Delta_2 + \Delta_3 - \gamma]) \\
    &= f(\frac{1}{3}[\eta \cdot(3\Delta_1 - \gamma + 3\Delta_2 - \gamma +3\Delta_3 - \gamma)])\\
    &\leq \frac{1}{3}[f(\eta\cdot(3\Delta_1 - \gamma)) + f(\eta\cdot(3\Delta_2 - \gamma)) + f(\eta\cdot(3\Delta_3 - \gamma))] \quad \text{by Jensen's Inequality} \\
    &= \frac{1}{3} [l_{3\eta, \frac{\gamma}{3}}(x_{long}, x_{short}; y_w, y_w) + l_{3\eta, \frac{\gamma}{3}}(x_{short},x_{short}; y_w, y_l) + l_{3\eta, \frac{\gamma}{3}}(x_{short}, x_{long}; y_l, y_l)]
\end{align*}
where 
\begin{align*}
    \Delta_1 &= r_\phi(x_{long}, y_w) - r_\phi(x_{short}, y_w) \\
    \Delta_2 &= r_\phi(x_{short}, y_w) - r_\phi(x_{short}, y_l) \\
    \Delta_3 &= r_\phi(x_{short}, y_l) - r_\phi(x_{long}, y_l) \\
\end{align*}
\end{proof}

\subsection{Proof of Theorem \ref{theorem_relation}}
\label{proof_of_theorem_1}
\begin{proof}
    Directly applying expectation $\mathbb{E}_{\substack{x_\cdot \sim \mathcal{D}_{x_{\cdot}}; y_w, y_l \sim \mathcal{D}_{y} \\ y_w \succ y_l | x_{long}} }$ to inequality \ref{lemma1_shs} and using assumption~\ref{assumption1}, we can obtain the following inequality:
    \begin{align}\label{ineq_diff_bound1}
    \begin{aligned}
    &\mathcal{L}_{\eta, \gamma}(\mathcal{D}_{x_{long}}, \mathcal{D}_{x_{long}}; \mathcal{D}_{y_{w} \succ y_{l}|x_{long}},\mathcal{D}_{y_{w} \succ y_{l}|x_{long}}) \\
    \leq& \tfrac{1}{3} \left[ 
    \begin{aligned}
    &\quad\ \mathcal{L}_{3\eta, \tfrac{\gamma}{3}}( \mathcal{D}_{x_{long}}, \mathcal{D}_{x_{short}}; \mathcal{D}_{y_{w}|x_{short}}, \mathcal{D}_{y_{w}|x_{short}}) \\[4pt]
    &+\, \mathcal{L}_{3\eta, \frac{\gamma}{3}}(\mathcal{D}_{x_{short}},\mathcal{D}_{x_{short}}; \mathcal{D}_{y_{w} \succ y_{l}|x_{short}},\mathcal{D}_{y_{w} \succ y_{l}|x_{short}}) \\[4pt]
    &+\, \mathcal{L}_{3\eta, \frac{\gamma}{3}}(\mathcal{D}_{x_{short}}, \mathcal{D}_{x_{long}}; \mathcal{D}_{y_{l}|x_{short}},\mathcal{D}_{y_{l}|x_{short}})]
    \end{aligned} \right].
    \end{aligned}
    \end{align}
    We prove only the second term of inequality~\ref{ineq_diff_bound1}, as the proofs for the remaining terms follow in the same manner.
    \begin{align}
        &\mathbb{E}_{\substack{x_\cdot \sim \mathcal{D}_{x_{\cdot}}; y_w, y_l \sim \mathcal{D}_{y},\ y_w \succ y_l | x_{long} }}[ l_{3\eta,\frac{\gamma}{3}}(x_{short}, x_{short}; y_w, y_l)]\\
        =&\mathbb{E}_{\substack{x_\cdot \sim \mathcal{D}_{x_{\cdot}}; y_w, y_l \sim \mathcal{D}_{y} }}[P(y_w \succ y_l | x_{long}) l_{3\eta,\frac{\gamma}{3}}(x_{short},x_{short}; y_w, y_l)]\\
        =& \mathbb{E}_{\substack{x_\cdot \sim \mathcal{D}_{x_{\cdot}}; y_w, y_l \sim \mathcal{D}_{y} }}[\frac{P(y_w \succ y_l | x_{long})}{P(y_w \succ y_l | x_{short})}P(y_w \succ y_l | x_{short}) l_{3\eta,\frac{\gamma}{3}}(x_{short}; y_w, y_l)] \\
        \leq& \mathbb{E}_{\substack{x_\cdot \sim \mathcal{D}_{x_{\cdot}}; y_w, y_l \sim \mathcal{D}_{y} }}[P(y_w \succ y_l | x_{short}) l_{3\eta,\frac{\gamma}{3}}(x_{short}; y_w, y_l)]\\
        =&\mathcal{L}_{3\eta, \frac{\gamma}{3}}(\mathcal{D}_{x_{short}},\mathcal{D}_{x_{short}}; \mathcal{D}_{y_{w} \succ y_{l}|x_{short}},\mathcal{D}_{y_{w} \succ y_{l}|x_{short}})
    \end{align}
    
    Now, we prove the inequality \ref{ineq_diff_bound_general}. We only need to consider the sum of the first term and the third term:
   \begin{align}
        &\mathcal{L}_{3\eta, \tfrac{\gamma}{3}}( \mathcal{D}_{x_{long}}, \mathcal{D}_{x_{short}}; \mathcal{D}_{y_{w}|x_{short}}, \mathcal{D}_{y_{w}|x_{short}}) + \mathcal{L}_{3\eta, \frac{\gamma}{3}}(\mathcal{D}_{x_{short}}, \mathcal{D}_{x_{long}}; \mathcal{D}_{y_{l}|x_{short}},\mathcal{D}_{y_{l}|x_{short}}) \\
        =& \mathbb{E}_{\substack{x_\cdot \sim \mathcal{D}_{x_{\cdot}}; \\ y_w, y_l \sim \mathcal{D}_{y}} } [P(y_w\succ y_l | x_{short}) (l_{3\eta, \frac{\gamma}{3}}(x_{long}, x_{short}; y_w, y_w) + l_{3\eta, \frac{\gamma}{3}}(x_{short}, x_{long}; y_l, y_l))] \\
        \leq& \mathbb{E}_{\substack{x_\cdot \sim \mathcal{D}_{x_{\cdot}}; \\ y_w, y_l \sim \mathcal{D}_{y}} } [l_{3\eta, \frac{\gamma}{3}}(x_{long}, x_{short}; y_w, y_w) + l_{3\eta, \frac{\gamma}{3}}(x_{short}, x_{long}; y_l, y_l)] \\
        =& \mathbb{E}_{\substack{x_\cdot \sim \mathcal{D}_{x_{\cdot}}; \\ y_w, y_l \sim \mathcal{D}_{y}} } [f(\eta \cdot (3\Delta_1(y_w) - \gamma)) + f(\eta\cdot (3\Delta_3(y_l) - \gamma))] \\
        =&\mathbb{E}_{\substack{x_\cdot \sim \mathcal{D}_{x_{\cdot}}; \\ y_w, y_l \sim \mathcal{D}_{y}} } [f(\eta \cdot (3\Delta_1(y_w) - \gamma)) + f(\eta\cdot (3\Delta_3(y_l) - \gamma))] \\
        =& \mathbb{E}_{\substack{x_\cdot \sim \mathcal{D}_{x_{\cdot}}; y\sim \mathcal{D}_{y}} } [f(\eta \cdot (3\Delta_1(y) - \gamma)) + f(\eta\cdot (3\Delta_3(y) - \gamma))] \\
        =& \mathbb{E}_{\substack{x_\cdot \sim \mathcal{D}_{x_{\cdot}}; y\sim \mathcal{D}_{y}} } [f(\eta \cdot (3\Delta_1(y) - \gamma)) + f(\eta\cdot (-3\Delta_1(y) - \gamma))] \\
        =& \mathbb{E}_{\substack{x_\cdot \sim \mathcal{D}_{x_{\cdot}}; y\sim \mathcal{D}_{y}} } [f(3\eta\Delta_1(y) - \eta\gamma)) + f(-3\eta\Delta_1(y) - \eta\gamma))] \\
        \leq& \mathbb{E}_{\substack{x_\cdot \sim \mathcal{D}_{x_{\cdot}}; y \sim \mathcal{D}_{y}} } s(|3\eta \cdot (r_\phi(x_{short}, y) - r_\phi(x_{long}, y))|)
    \end{align}
    where $s(\cdot)$ satisfies $f(z+\gamma) + f(-z + \gamma) \leq s(|z|)$ and
    \begin{align}
        \Delta_1(y) &= r_\phi(x_{long}, y) - r_\phi(x_{short}, y) \\
        \Delta_3(y) &= r_\phi(x_{short}, y) - r_\phi(x_{long}, y).
    \end{align}
\end{proof}

The introduction of $s(\cdot)$ is to unify the two terms, $r_\phi(x_{long}, y)$ and $r_\phi(x_{short}, y)$. into a single expression. This unified expression serves to \textbf{quantify the distance} between $r_\phi(x_{long}, y)$ and $r_\phi(x_{short}, y)$. Building upon this foundation, we further generalized this concept in Appendix~\ref{sec:general_theorem_1}, leading to a more generalized theorem, which may provide valuable insights and inspire future research directions. Given that $s(\cdot)$ serves as an upper bound, a tighter instantiation is theoretically preferred; we provide empirical evidence for this claim in Appendix~\ref{sec:exp_for_s}.
\subsection{\ours for $f(x) = x^2$}
\label{sec:solopo_x2}
\begin{proposition}
Following the notation of Theorem~\ref{theorem_relation}, if we take $f(x) = x^2$, the inequality becomes:
\begin{equation}
    \mathcal{L}_{\eta, \gamma}(x_{long}) \leq \frac{1}{3} \mathcal{L}_{3\eta, \frac{\gamma}{3}}(x_{short}) +  3\eta^2\cdot \mathbb{E}_{x_{\cdot} \sim \mathcal{D}_{x_\cdot}; y \sim \mathcal{D}_{y}}|r_\phi(x_{short}, y) - r_\phi(x_{long}, y)|^2 + \frac{2}{3}\gamma^2
\end{equation}
\end{proposition}
\begin{proof}
    \begin{align*}
        f(x + \gamma) + f(-x + \gamma) &= (x+\gamma)^2 + (-x + \gamma)^2 = 2x^2 + 2\gamma^2
    \end{align*}
    Therefore, by using the Theorem \ref{theorem_relation}, we prove this proposition.
\end{proof}

\subsection{\ours for $f(x) = -\log \sigma(x)$}
\label{sec:solopo_logsigmiod}
\begin{proposition}\label{prop:specific}
Following the notation of Theorem~\ref{theorem_relation}, if we take  $f(x) = -\log \sigma(x)$, then the inequality can be:
\begin{equation}
    \mathcal{L}_{\eta,\gamma}(x_{long}) \leq \frac{1}{3} \mathcal{L}_{3\eta,\frac{\gamma}{3}}(x_{short}) +  \eta\cdot\mathbb{E}_{x_{\cdot} \sim \mathcal{D}_{x_\cdot}; y \sim \mathcal{D}_{y}}|r_\phi(x_{short}, y) - r_\phi(x_{long}, y)| + \frac{2}{3}\log(1 + e^{\eta\cdot \gamma})
\end{equation}
\end{proposition}
\begin{proof}
    \begin{align}
        &\mathcal{L}_{3\eta, \tfrac{\gamma}{3}}( \mathcal{D}_{x_{long}}, \mathcal{D}_{x_{short}}; \mathcal{D}_{y_{w}|x_{short}}, \mathcal{D}_{y_{w}|x_{short}}) + \mathcal{L}_{3\eta, \frac{\gamma}{3}}(\mathcal{D}_{x_{short}}, \mathcal{D}_{x_{long}}; \mathcal{D}_{y_{l}|x_{short}},\mathcal{D}_{y_{l}|x_{short}}) \\
        \leq& \mathbb{E}_{\substack{x_\cdot \sim \mathcal{D}_{x_{\cdot}}; \\ y_w, y_l \sim \mathcal{D}_{y}} } [l_{3\eta, \frac{\gamma}{3}}(x_{long}, x_{short}; y, y) + l_{3\eta, \frac{\gamma}{3}}(x_{short}, x_{long}; y, y)] \\
    \end{align}
    For brevity, we omit the expectation in the following derivation.
    \begin{align}
        &l_{3\eta, \frac{\gamma}{3}}(x_{long}, x_{short};, y_w) + l_{3\eta, \frac{\gamma}{3}}(x_{short}, x_{long}; y, y)\\
    =& f(3\eta\cdot(r_\phi(x_{long}, y) - r_\phi(x_{short}, y)) - \eta\cdot \gamma) + f(3\eta\cdot(r_\phi(x_{short}, y) - r_\phi(x_{long}, y)) - \eta\cdot \gamma) \\
    &\text{we denote $r_\phi(x_{long}, y), r_\phi(x_{short}, y)$  as $r_l$, $r_s$} \\
    =&f(3\eta\cdot(r_l - r_s) - \eta\cdot\gamma) + f(3\eta\cdot(r_s - r_l) - \eta\cdot \gamma)\\
    =& -\log\sigma(3\eta\cdot r_l - 3\eta\cdot r_s - \eta\cdot\gamma) - \log\sigma(3\eta\cdot r_s - 3\eta\cdot r_l - \eta\cdot\gamma) \\
    =& - \log\frac{1}{1 + \exp{\{-(3\eta\cdot r_l - 3\eta\cdot r_s - \eta\cdot\gamma)\}}} - \log\frac{1}{1 + \exp{\{-(3\eta\cdot r_s - 3\eta\cdot r_l - \eta\cdot\gamma)\}}}\\
    =& -\log \frac{e^{3\eta\cdot r_l}}{e^{3\eta\cdot r_s+\eta\cdot\gamma} + e^{3\eta\cdot r_l}} - \log \frac{e^{3\eta\cdot r_s}}{e^{3\eta\cdot r_s} + e^{3\eta\cdot r_l+\eta\cdot\gamma}} \\
    =& -3\eta\cdot r_l - 3\eta\cdot r_s + \log(e^{3\eta\cdot r_s + \eta\cdot \gamma} + e^{3\eta\cdot r_l}) + \log(e^{3\eta\cdot r_s} + e^{3\eta\cdot r_l + \eta\cdot \gamma}) \\
    &\stackrel{(a)}{\leq} 3\mathbb{E}_{\substack{x_\cdot \sim \mathcal{D}_{x_{\cdot}}; \\ y \sim \mathcal{D}_{y}}}|r_l - r_s| + 2\log(1+e^{3\gamma})
    \end{align}

    In the following, we prove the inequality (a).
    
    If $r_l \leq r_s$, then
    \begin{align}
        & -3\eta\cdot r_l - 3\eta\cdot r_s + \log(e^{3\eta\cdot r_s + \eta\cdot \gamma} + e^{3\eta\cdot r_l}) + \log(e^{3\eta\cdot r_s} + e^{3\eta\cdot r_l + \eta\cdot \gamma}) \\
        &\leq -3\eta\cdot r_l - 3\eta\cdot r_s + \log(e^{3\eta\cdot r_s + \eta\cdot \gamma} + e^{3\eta\cdot r_s}) + \log(e^{3\eta\cdot r_s} + e^{3\eta\cdot r_s + \eta\cdot \gamma})] \\
        &= -3\eta\cdot r_l - 3\eta\cdot r_s + 6\eta\cdot r_s + 2\log(1+e^{\eta\cdot \gamma})] \\
        &= 3\eta\cdot r_s - 3\eta\cdot r_l + 2\log(1+e^{\eta\cdot \gamma})
    \end{align}
    By symmetry, we can easily obtain that if $r_s \leq r_l$, then
    \begin{equation}
       -3\eta\cdot r_l - 3\eta\cdot r_s + \log(e^{3\eta\cdot r_s + \eta\cdot \gamma} + e^{3\eta\cdot r_l}) + \log(e^{3\eta\cdot r_s} + e^{3\eta\cdot r_l + \eta\cdot \gamma}) \leq 3\eta\cdot r_l - 3\eta\cdot r_s + 2\log(1+e^{\eta\cdot \gamma})
    \end{equation} 
    Therefore,
    \begin{equation}
        -3\eta\cdot r_l - 3\eta\cdot r_s + \log(e^{3\eta\cdot r_s + \eta\cdot \gamma} + e^{3\eta\cdot r_l}) + \log(e^{3\eta\cdot r_s} + e^{3\eta\cdot r_l + \eta\cdot \gamma}) \leq 3\eta\cdot |r_l - r_s| + 2\log(1+e^{\eta\cdot \gamma})
    \end{equation}
\end{proof}

\subsection{Generalization of Theorem \ref{theorem_relation}}
\label{sec:general_theorem_1}
\begin{theorem}\label{new_distance_shs}
Let  $D_{p}(x_1, x_2) = (\mathbb{E}_{ y \sim \mathcal{D}_{y}}|r_\phi(x_{1}, y) - r_\phi(x_{2}, y)|^p)^{\frac{1}{p}}$. If any divergence $D(x_1, x_2)$ satisfies 
\begin{equation}\label{ineq_distance_shs}
     D_{1}(x_1, x_2) \leq C_1 \cdot D(x_1, x_2)
\end{equation}
where $C_1$ are positive constants, then the following inequality holds true for the convex function $f(x) = -\log \sigma(x)$, as in the settings of DPO, SimPO, and ORPO:
\begin{align}
    &\mathcal{L}_{\eta, \gamma}(\mathcal{D}_{long},\mathcal{D}_{long}; \mathcal{D}_{y\mid x_{long}},\mathcal{D}_{y\mid x_{long}}) \\
    &\leq \frac{1}{3}\mathcal{L}_{3\eta, \frac{\gamma}{3}}(\mathcal{D}_{short},\mathcal{D}_{short}; \mathcal{D}_{y\mid x_{short}},\mathcal{D}_{y\mid x_{short}}) + \eta\cdot C_1 \mathbb{E}_{x_{\cdot} \sim \mathcal{D}_{x_\cdot}}[ D(x_{short}, x_{long}) ] + C_2
\end{align}
where $C_2 = \frac{2}{3}(\log(1 + e^{\eta\cdot\gamma}))$.
\end{theorem}
Theorem \ref{new_distance_shs} guarantees that any new distance satisfying the inequality \ref{ineq_distance_shs} can substitute the absolute distance.
\begin{proof}
    This theorem can be proved by directly using the proposition \ref{prop:specific}.
\end{proof}

\subsection{The General Formula of \ours Loss Function}
\begin{tcolorbox}[
    colframe=white,       
    colback=white,      
    sharp corners,      
    boxrule=1.5pt,      
    top=4pt,            
    bottom=4pt,         
    left=4pt,           
    right=4pt,          
    ]
{The general formula of Short-to-Long Preference Optimization (\ours) loss function:} 

\begin{flalign*}
    & \mathcal{L}_{\ours} =  \mathbb{E}_{\substack{x \sim \mathcal{D}_{x_{short}} ; \\ y_w, y_l \sim \mathcal{D}_{y}; y_w\succ y_l} } [ f(\eta \cdot[r_\phi(x, y_w) - r_\phi(x, y_l) - \gamma])] \\
    & \qquad\qquad\qquad\qquad\qquad\qquad + \alpha \cdot \mathbb{E}_{x_{\cdot} \sim \mathcal{D}_{x_\cdot}; y \sim \mathcal{D}_{y}}s( |r_\phi(x_{short}, y) - r_\phi(x_{long}, y)|)
\end{flalign*}
where $f$ is a convex function, and $f(x+\gamma) + f(-x + \gamma) \leq s(|x|)$ for some function $s$. $\gamma,\alpha, \eta$ are hyperparameters.

\end{tcolorbox}

    

\subsection{{Empirical evidence for Assumption~\ref{assumption1}}}
\label{sec:empirical_evidence_for_assumption_1}
In this section, we verify Assumption~\ref{assumption1}:

\textit{since $x_{\text{long}}$ contains more task-irrelevant information than $x_{\text{short}}$, making it harder for the model to distinguish preferences when conditioned on $x_{\text{long}}$}:
\begin{align}\label{eq_hy_appendix}
    p(y_w\succ y_l | x_{long}) \leq p(y_w\succ y_l | x_{short})
\end{align}
\textit{where $p(y_w \succ y_l \mid x) = \sigma\!\left(r^*(x,y_w) - r^*(x,y_l)\right)$,
$r^*$ denotes the optimal reward model, and $\sigma$ is the sigmoid function.}

In the absence of an optimal reward model, we employ a model $\pi_{final}$ trained via DPO to estimate the reward margin for $(y_w, y_l)$, following the reward computation formulation defined in DPO:
\begin{align}\label{eq:dpo_reward_margin_appendix}
    r_{DPO}(x,y_w) - r_{DPO}(x,y_l) = \left(
\beta \log \frac{\pi_{final}(y_w \mid x)}{\pi_{\mathrm{ref}}(y_w \mid x)}
- \beta \log \frac{\pi_{final}(y_l \mid x)}{\pi_{\mathrm{ref}}(y_l \mid x)}
\right)
\end{align}

Specifically, we adopt Qwen2.5-7B-Instruct as the reference policy $\pi_{\text{ref}}$, and perform \textbf{DPO training} on data with a context length of $1$K to obtain the final policy $\pi_{\text{final}}$. We set the short-context length to $4$K, from which we sample preference pairs $y_w \succ y_l \sim \pi_{\text{ref}}(y \mid x_{\text{short}})$, and subsequently expand them to lengths ranging from $8$K to $32$K to obtain $x_{\text{long}}$. \textbf{This design aims to emulate a reward model with a non-trivial scoring capacity that is nonetheless susceptible to noise, in order to meet the preconditions of our assumption}. For each length, we compute the proportion satisfying Eq.~(\ref{eq_hy_appendix}) based on Eq.~(\ref{eq:dpo_reward_margin_appendix}).

\begin{table}[h]
    \centering
    \caption{Proportion of cases in which Assumption~\ref{eq_hy_appendix} holds for preference pairs ($y_w \succ y_l$) sampled from $x_{\text{short}}$ (length $4\mathrm{K}$) and evaluated on $x_{\text{long}}$ with varying lengths ($8$K--$32$K), using a model trained via DPO on a $1$K context length. Longer contexts introduce additional task-irrelevant information, leading to a gradual increase in the satisfaction rate and stabilizing at approximately $95\%$}
    \begin{adjustbox}{max width=\textwidth}
        \begin{tabular}{lccccccc}
    \toprule
         \textbf{Length of $x_{long}$}& $8\mathrm{K}$&$12\mathrm{K}$& $16\mathrm{K}$&$20\mathrm{K}$&$24\mathrm{K}$&$28\mathrm{K}$&$32\mathrm{K}$  \\
    \midrule
    \textbf{Hypothesis Validity Proportion}&80.58\%&86.41\%&91.26\%&90.29\%&96.12\%&95.15\%&94.17\%\\
    \bottomrule
    \end{tabular}
    \end{adjustbox}
    
    \label{tab:hy_holds}
\end{table}

 Results are shown in Table~\ref{tab:hy_holds}. As the context length increases, the amount of task-irrelevant information grows, and the proportion satisfying the hypothesis gradually rises, stabilizing at around $95\%$. Considering potential inaccuracies in reward estimation, such consistently high proportions provide substantial support for Assumption~\ref{assumption1}.

\subsection{{The impact of the tightness of $s(\cdot)$ on SoLoPO's performance}}
\label{sec:exp_for_s}
Theoretically, a tighter $s$ corresponds to a stronger upper bound, which may improve empirical performance. In contrast, a loose upper bound often results in vague optimization directions and may cause the optimization to converge to a local optimum. For example, consider minimizing $x^4-x^2=x^2(x^2-2)$, whose global minima are at  $x= ±1$. If we use an upper bound $x^4$, the global minimum of this upper bound is at $x=0$. Therefore, if the upper bound is too loose, the obtained solution may deviate from the true optimum.

We compare a \textit{tighter} setting $s_t(x) = |x|$ (used in the paper) with a \textit{looser} setting $s_l(x) = |x| + \sin(x) + 1 \ge s_t(x)$ on Qwen2.5-7B-Instruct, while keeping all other training configurations identical. As shown in the table below, the tighter $s(\cdot)$ generally achieves better performance, except in the MD-QA scenario of LongBenchV1 where SoLo-DPO($s_t$) is slightly worse than SoLo-DPO($s_l$). Since $s_t \leq s_l$ in our experiments, we can infer that a looser $s(\cdot)$ tends to yield worse performance.
\begin{table}[h]
    \centering
    \caption{Impact of $s(\cdot
    )$ tightness on performance. Tighter $s(\cdot
    )$ generally improves performance.}
    \adjustbox{max width=\textwidth}{
    \begin{tabular}{lccccccccc}
    \toprule
         \multirow{2}{*}{\textbf{Model}}&\multicolumn{3}{c}{\textbf{QAs-LongBenchV1}}&&\multicolumn{5}{c}{\textbf{QAs-RULER}}\\
         \cline{2-4} \cline{6-10}
         &\textbf{S-Doc QA}&\textbf{M-Doc QA}&\textbf{Avg.}&&\textbf{4k}&\textbf{8k}&\textbf{16k}&\textbf{32k}&\textbf{Avg.}\\
    \midrule
    $M^{DPO}_{SoLo(s_l)}$ &35.3 & \textbf{57.8}& 46.5&&64.2&62.6 &61.2 &57.4&61.4\\
    $M^{DPO}_{SoLo(s_t)}$ &\textbf{38.0}& 57.6 &\textbf{47.8} &&\textbf{66.4}& \textbf{64.5} &\textbf{62.7} &\textbf{57.7} &\textbf{62.8} \\
    \midrule
    $M^{SimPO}_{SoLo(s_l)}$ &36.8 &56.7& 46.7&&67.8&65.0 &61.9&57.2&63.0\\
    $M^{SimPO}_{SoLo(s_t)}$ & \textbf{38.1} &\textbf{58.6}& \textbf{48.4}&& \textbf{69.2} &\textbf{66.0} &\textbf{62.7} &\textbf{57.8} &\textbf{63.9}\\
    \midrule
    $M^{ORPO}_{SoLo(s_l)}$ &36.1& 57.8&  46.9&&68.1&65.4 &60.5 &56.2&62.6\\
    $M^{ORPO}_{SoLo(s_t)}$ &\textbf{37.6}& \textbf{61.4} &\textbf{49.5}&&\textbf{70.8}&\textbf{68.3} &\textbf{64.0}& \textbf{57.3} &\textbf{65.1}\\
    \bottomrule
    \end{tabular}
    }
    \label{tab:placeholder}
\end{table}
\subsection{Discussion on the Modeling of SoLoPO}
\label{sec:on_modeling_solopo}
\paragraph{a. Two key abilities in long-context scenarios}
Unlike short-context tasks such as mathematics~\citep{li2025pspoeffectiveprocesssupervisedpolicy}, which can directly leverage the model's inherent \textbf{(contextual knowledge) reasoning} ability, long-context tasks---such as question answering~\citep{musique-dataset} or information extraction~\citep{xu2024large}---also require the model to possess \textbf{contextual knowledge localization} skills, i.e., the ability to identify task-relevant information $c_{rel}$ from a long context $c_{long}$ while ignoring irrelevant content $c_{irr}$. In other words, the model needs to first identify the key task-relevant information $c_{rel}$ from the context $c_{long}$ and subsequently perform reasoning upon it~\citep{li-etal-2024-fundamental}. 
\paragraph{b. Explicit modeling of two key abilities in SoLoPO}
Recall that the optimization objective of SoLoPO is defined as follows:
{\small
\begin{align}
     \mathcal{L}_{SoLoPO} =\mathbb{E}_{\substack{x \sim \mathcal{D}_{x_{short}};y_w, y_l \sim \mathcal{D}_{y}\\y_w\succ y_l\sim \pi_\theta(y\mid x_{short})}}&\underbrace{ \left[ f\left(3\eta \cdot[r_\phi(x, y_w) - r_\phi(x, y_l) - \frac{\gamma}{3}]\right)\right]}_{\text{short-context preference optimization}}\label{short_po_agin} \\
    + \alpha \cdot \mathbb{E}_{\substack{x_{\cdot} \sim \mathcal{D}_{x_\cdot}; y. \sim \mathcal{D}_{(y_w,y_l)}\\y_w\succ y_l\sim \pi_\theta(y\mid x_{short})}}& \underbrace{\left[s(3\eta\cdot|r_\phi(x_{short}, y) - r_\phi(x_{long}, y)|)\right]}_{\text{short-to-long reward alignment}}\label{solo_ra_agin}.
\end{align}
}
SoLoPO explicitly models the two abilities in a decoupled manner:
\begin{itemize}[leftmargin=0.5cm]
    \item \textbf{Contextual knowledge reasoning:} Since $x_{short}$ consists of the task instruction $I$ and the task-relevant content $c_{\mathrm{rel}}$, i.e., $x_{short} := [c_{\mathrm{rel}}; I]$, SoLoPO directly enhances the model's inherent reasoning ability via short-context preference optimization (Eq.~(\ref{short_po_agin})).
    \item \textbf{Contextual knowledge localization:} SoLo-RA (Eq.~\ref{solo_ra_agin}) encourages the reward model $r_\phi$ to implicitly predict $\hat{x}_{short} \sim \hat{p}(x_{short} \mid x_{long})$ by minimizing the divergence between $\hat{x}_{short}\coloneqq[\hat c_{rel};I]$ and the ground-truth $x_{short}\coloneqq[c_{rel};I]$. In preference optimization, where the reward model $r_\phi$ and the policy model $\pi_\theta$ coincide, this process also strengthens the policy model's ability to locate relevant knowledge $c_{rel}$ for task $I$ within a long context $c_{long}$.

    Taking SimPO as an example, where $s(|x|) = |x| + C$ and the reward is defined as $r_\theta(x, y) = \frac{\beta}{|y|}\log\pi_\theta(y|x)$. For brevity, we set $\eta = \frac{1}{3}$ and omit constant $C$, the SoLo-RA loss becomes:
    \begin{align}
        \text{SoLo-RA}_{SimPO} =\mathbb{E}_{\substack{x_{\cdot} \sim \mathcal{D}_{x_\cdot}; y \sim \mathcal{D}_{y}\\y\sim \pi_\theta(y\mid x_{short})}} \left[\frac{\beta}{|y|}|\log \pi_\theta(y \mid x_{short}) - \log \pi_\theta(y\mid  x_{long})|\right],
    \end{align}
    which encourages $\pi_\theta$ to produce an output $y$ with the same likelihood whether it is conditioned on $x_{short}$ or on the full input $x_{long}$. Under this objective, SoLo-RA \textbf{implicitly} guides the model to extract from $x_{long}$ only the minimal sufficient information needed to behave as if conditioned on $x_{short}$. That is, it enforces:
    \begin{align}
        \pi_\theta(y \mid x_{long}) \approx \pi_\theta(y \mid x_{short}) \overset{implicitly}{\Longrightarrow} &g_{\theta'}(x_{long}) = x_{short}
    \end{align}
    where $g_{\theta'}$ can be interpreted as an internal attention mechanism or latent projection that compresses $x_{long}$ into a representation functionally equivalent to $x_{short}$. 
    
    This learned behavior aligns with the principle of contextual knowledge localization.
\end{itemize}

\subsection{Supporting Analysis for Chosen-only SoLo-RA}
\subsubsection{The original objective of SoLoPO is theoretical and empirical supported} From an theoretical perspective, SoLoPO decomposes long-context preference optimization (PO) into short-context PO and short-to-long reward alignment (SoLo-RA) (\S~\ref{sec:method_loss_analysis}). The experimental analysis in Section~\ref{sec:depth_analysis} and Figure~\ref{fig:perfor_both_chosen} shows that directly using the SoLoPO objective---applying SoLo-RA jointly to both the chosen and rejected responses (\textit{both} SoLo-RA)---achieves superior performance to Long-PO across most settings. These results offer both theoretical proof and empirical validation for the original optimization objective of \textbf{SoLoPO with \textit{both} SoLo-RA}, as defined in Eq. (\ref{solo_ra_agin}).

\subsubsection{Supporting Analysis for chosen-only SoLo-RA}
\label{sec:supporting_analysis_for_chosen_only_solo_RA}
In practical scenarios, we further consider a variant, \textbf{SoLoPO with \textit{chosen-only} SoLo-RA}, which is motivated by two factors:
\begin{enumerate}[leftmargin=1cm]
    \item $y_l$ may not always exploit the key information in the context. For example, $y_l$ might simply respond, ``No task-relevant content can be found in the context, so the question cannot be answered.'' In such cases, enforcing SoLo-RA may not improve, and could even degrade, the model’s ability to locate or reason over relevant long-context information.
    \item We aim to reduce the amount of long-text processing during training, thereby improving training efficiency.
\end{enumerate}

In addition, we further examine the theoretical validity of the first motivation from a \textbf{data-sampling perspective}.
\paragraph{a. Relation $\pi(y \mid x_{short}) \ge \pi(y \mid x_{long})$ typically holds owing to the practical data sampling strategy.} 

Recall that preference pairs in \textsc{SoLoPO} are obtained based on short contexts, as defined in Eq. (\ref{solo_ra_agin}):
\begin{equation}
    y_w\succ y_l\sim \pi_\theta(y|x_{short}).
\end{equation}
For arbitrary $x_{short}$ and $x_{long}$, the Kullback–Leibler divergence~\citep{wiki:kl_divergence} satisfies:
{\small
\begin{align}
    D_{KL}(\pi(y|x_{short}) \parallel \pi(y|x_{long})) = \int \big[ \log \pi(y|x_{short}) - \log \pi(y|x_{long}) \big] \pi(y|x_{short}) \ dy \geq 0.
\end{align}
}

This implies that, when sampling $y \sim \pi_\theta(y \mid x_{short})$, the quantity
$
\log \pi(y \mid x_{short}) - \log \pi(y \mid x_{long})
$
is more likely than not to be non-negative. Since the logarithm is strictly monotonic, the following inequality \textbf{tends to hold} (more accurately, holds in expectation; see Table~\ref{tab:short_logp_bigger_exp} for empirical evidence):
\begin{equation}\label{eq:prb_short>long}
    \pi_\theta(y \mid x_{short}) \ge \pi_\theta(y \mid x_{long}), \quad \text{for} \quad y \sim \pi_\theta(y|x_{short}).
\end{equation}

\paragraph{b. Applying SoLo-RA to $y_l$ may harm long-context capability.}
Referring to Table~\ref{tab:s2l_app_obj}, the SoLo-RA loss becomes:
{\small
\begin{align}
    \text{SoLo-RA}_{SimPO} =&\mathbb{E}_{\substack{x_{\cdot} \sim \mathcal{D}_{x_\cdot}; y \sim \mathcal{D}_{y}\\y\sim \pi_\theta(y\mid x_{short})}} \left[ \frac{\beta}{|y|}\big|\log \frac{\pi_\theta(y|x_{short})}{\pi_\theta(y|x_{long})}\big| \right]\\
    \text{SoLo-RA}_{DPO} =&\mathbb{E}_{\substack{x_{\cdot} \sim \mathcal{D}_{x_\cdot}; y \sim \mathcal{D}_{y}\\y\sim \pi_\theta(y\mid x_{short})}} \left[ \beta\big|\log \frac{\pi_\theta(y|x_{short})}{\pi_\theta(y|x_{long})} + \log \frac{\pi_{ref}(y|x_{long})}{\pi_{ref}(y|x_{short})} + \log \frac{Z(x_{short})}{Z(x_{long})} \big| \right]\\
    \text{SoLo-RA}_{ORPO} =&\mathbb{E}_{\substack{x_{\cdot} \sim \mathcal{D}_{x_\cdot}; y \sim \mathcal{D}_{y}\\y\sim \pi_\theta(y\mid x_{short})}} \left[ \big|\log \frac{\pi_\theta(y|x_{short})}{\pi_\theta(y|x_{long})} + \log \frac{ 1 - \pi_{\theta}(y|x_{long})}{1 - \pi_{\theta}(y|x_{short})} \big| \right]
\end{align}
}
Since the latter two terms in $\text{SoLo-RA}_{DPO}$  are independent of the learnable parameters $\theta$, and the reward coefficient does not affect the subsequent analysis, we treat the SoLo-RA in DPO as equivalent to that in SimPO. Expanding the expectation, we separate \textbf{cases} based on the likelihood ratio:
{\small
\begin{align}
    \text{SoLo-RA}_{DPO/SimPO} = &\mathbb{E}_{\substack{x_{\cdot} \sim \mathcal{D}_{x_\cdot}; y \sim \mathcal{D}_{y}\\y\sim \pi_\theta(y\mid x_{short})}} \left[ \mathbf{1}[\frac{\pi_\theta(y|x_{short})}{\pi_\theta(y|x_{long})} \geq 1]\log \frac{\pi_\theta(y|x_{short})}{\pi_\theta(y|x_{long})}\right]\\
    - & \mathbb{E}_{\substack{x_{\cdot} \sim \mathcal{D}_{x_\cdot}; y \sim \mathcal{D}_{y}\\y\sim \pi_\theta(y\mid x_{short})}} \left[\mathbf{1}[\frac{\pi_\theta(y|x_{short})}{\pi_\theta(y|x_{long})} < 1] \log \frac{\pi_\theta(y|x_{short})}{\pi_\theta(y|x_{long})} \right]
\end{align}
}
{\small
\begin{align}
    \text{SoLo-RA}_{ORPO} =& \mathbb{E}_{\substack{x_{\cdot} \sim \mathcal{D}_{x_\cdot}; y \sim \mathcal{D}_{y}\\y\sim \pi_\theta(y\mid x_{short})}} \left[ \mathbf{1}[\frac{\pi_\theta(y|x_{short})}{\pi_\theta(y|x_{long})} \geq 1] (\log \frac{\pi_\theta(y|x_{short})}{\pi_\theta(y|x_{long})}+ \log \frac{ 1 - \pi_{\theta}(y|x_{long})}{1 - \pi_{\theta}(y|x_{short})})\right] \\
    -&\mathbb{E}_{\substack{x_{\cdot} \sim \mathcal{D}_{x_\cdot}; y \sim \mathcal{D}_{y}\\y\sim \pi_\theta(y\mid x_{short})}} \left[ \mathbf{1}[\frac{\pi_\theta(y|x_{short})}{\pi_\theta(y|x_{long})} < 1] (\log \frac{\pi_\theta(y|x_{short})}{\pi_\theta(y|x_{long})} + \log \frac{ 1 - \pi_{\theta}(y|x_{long})}{1 - \pi_{\theta}(y|x_{short})}) \right]
\end{align}
}

We restrict our analysis to the cases in $\text{SoLo-RA}_{\mathrm{DPO/SimPO}}$, as $\text{SoLo-RA}_{\mathrm{ORPO}}$ can be examined in the same manner. We next consider two cases:
\begin{itemize}[leftmargin=0.5cm]
    \item \textbf{Case 1: Winning Responses ($y = y_w$)} Based on Eq. (\ref{eq:prb_short>long}), we have $\pi_\theta(y_w|x_{long}) \leq \pi_\theta(y_w|x_{short})$. Here, SoLo-RA minimizes $\log \frac{\pi_\theta(y_w|x_{short})}{\pi_\theta(y_w|x_{long})}$, which:
    \begin{itemize}
        \item \textbf{Decrease} $\pi_\theta(y_w|x_{short})$ (nominator) but is counterbalanced by first loss term (short-context preference optimization);
        \item \textbf{Increases} $\pi_\theta(y_w|x_{long})$ (denominator), aligning the objective of long-context alignment.
    \end{itemize}
    Since SoLoPO also includes short-context PO (Eq. (\ref{short_po_agin})), this term dominates and can prevent $\pi_\theta(y_w|x_{short})$ from decreasing, thereby mitigating the negative impact on the model's short-context performance.
    \item \textbf{Case 2: Losing Responses ($y = y_l$)} Based on Eq. (\ref{eq:prb_short>long}), we have $\pi_\theta(y_l|x_{long}) \leq \pi_\theta(y_l|x_{short})$. SoLo-RA again minimizes $\log \frac{\pi_\theta(y_l|x_{short})}{\pi_\theta(y_l|x_{long})}$, leading to:
    \begin{itemize}
        \item \textbf{Decrease} in $\pi_\theta(y_l|x_{short})$ (desirable for reducing poor generation)
        \item \textbf{Increase} in $\pi_\theta(y_l|x_{long})$ (undesirable, as it promotes $y_l$ in long contexts).
    \end{itemize}
    However, there is no other loss here that can reduce the prediction probability of $\pi_\theta(y_l|x_{long})$, therefore, using \textbf{Case 2} would bring certain negative impacts to the model's long-text capabilities.
\end{itemize}

Based on the above analysis and our goal of improving training efficiency, we adopt the \textit{chosen-only} SoLo-RA in practical applications of SoLoPO. The experimental results in Section~\ref{sec:depth_analysis} and Figure~\ref{fig:perfor_both_chosen} further demonstrate the effectiveness of the \textit{chosen-only} SoLo-RA.

\paragraph{c. No conflict with the original objective of SoLoPO}
It is important to note that \textbf{the above analysis does not conflict with the original optimization objective of \textsc{SoLoPO} (with \textit{both} SoLo-RA)}. As long as the objective is perfectly optimized---e.g., the SoLo-RA loss (Eq. (\ref{solo_ra_agin})) reaches zero---the short-context PO (Eq. (\ref{short_po_agin})) will simultaneously reduce $\pi_\theta(y_l|x_{short})$ and $\pi_\theta(y_l|x_{long})$, thereby resolving the issue in \textbf{Case 2} and aligning with the objective of preference optimization.
\paragraph{d. Empirical validation of relationship $\pi_\theta(y \mid x_{\text{short}}) \ge \pi_\theta(y \mid x_{\text{long}})$}
We further empirically validate the relationship 
$\pi_\theta(y \mid x_{\text{short}}) \ge \pi_\theta(y \mid x_{\text{long}})$ for $y \sim \pi_\theta(y|x_{short})$ (Eq. (\ref{eq:prb_short>long})). 
Specifically, 100 preference pairs are sampled from $x_{\text{short}}$ 
(length $1\mathrm{K}$) using Qwen-2.5-7B-Instruct, and the corresponding contexts are then extended to lengths from $4\mathrm{K}$ to $32\mathrm{K}$ in $4\mathrm{K}$ increments to form $x_{\text{long}}$. For each 
length, we measure the proportion of instances in which the relationship 
holds. As shown in Table~\ref{tab:short_logp_bigger_exp}, the relationship
is preserved in $100\%$ of the cases across all context lengths tested.
\begin{table}[ht]
    \centering
    \caption{Proportion of cases where relationship $\pi(y \mid x_{\text{short}}) \ge \pi(y \mid x_{\text{long}})$ holds for preference pairs ($y_w\succ y_l$) sampled from $x_{\text{short}}$ (length $1\mathrm{K}$) evaluated on $x_{\text{long}}$ with varying lengths.}
    
    \begin{adjustbox}{max width=\textwidth}
        \begin{tabular}{lcccccccc}
    \toprule
         \textbf{Length of $x_{long}$}&$4\mathrm{K}$& $8\mathrm{K}$&$12\mathrm{K}$& $16\mathrm{K}$&$20\mathrm{K}$&$24\mathrm{K}$&$28\mathrm{K}$&$32\mathrm{K}$  \\
    \midrule
         \textbf{Ratio of $\pi_\theta(y_w \mid x_{\text{short}}) \ge \pi_\theta(y_w \mid x_{\text{long}})$}& 100\% &100\% &100\% &100\% & 100\% &100\% &100\% &100\%\\
         \textbf{Ratio of $\pi_\theta(y_l \mid x_{\text{short}}) \ge \pi_\theta(y_l \mid x_{\text{long}})$}& 100\% &100\% &100\% &100\% & 100\% &100\% &100\% &100\%\\
    \bottomrule
    \end{tabular}
    \end{adjustbox}
    
    \label{tab:short_logp_bigger_exp}
\end{table}

\subsection{Supporting Analysis of SoLoPO’s Stability in Short-Context Capability}
\label{sec:short_context_stability}
Appendix~\ref{sec:supporting_analysis_for_chosen_only_solo_RA} analyzes the rationale of the \textit{chosen-only} SoLo-RA. In \textbf{Case 1}, it is noted that during optimization, 
$\pi_\theta(y_w|x_{short})$ may decrease. However, since SoLoPO also incorporates short-context PO (Eq.~(\ref{short_po_agin})), this term dominates and prevents $\pi_\theta(y_w|x_{short})$ from dropping, thereby alleviating distributional shift in short contexts~\cite{longred} and mitigating its adverse impact on short-context performance.

\subsection{Applications of Short-to-Long Preference Optimization Loss}
\label{sec:full_express_of_SoLoPO_app}
\begin{table}[ht]
    \centering
    \caption{Variants of Convex function $f(x)$ and upper bound function $s(x)$.}
    
    \begin{tabular}{cccc}
        \toprule
         \textbf{Methods} & $f(x)$& $s(x)$ & $r_\phi(x,y)$\\
         \midrule
         DPO\cite{Rafailov2023DirectPO} &  $\log(1+e^{-x})$ & $|x| + 2\log(1 + e^{3\gamma})$ & $\beta\log \frac{\pi_\theta(y|x)}{\pi_{ref}(y|x)} + \beta\log Z(x)$\\
         SimPO\cite{meng2024simpo}  & $\log(1+e^{-x})$ & $|x| + 2\log(1 + e^{3\gamma})$ & $\frac{\beta}{|y|} \log\pi_{\theta}(y|x)$ \\
         ORPO\cite{hong-etal-2024-orpo}  & $\log(1+e^{-x})$ & $|x| + 2\log(1 + e^{3\gamma})$ & $\log \frac{\pi_\theta(y|x)}{1 - \pi_\theta(y|x)}$ \\
         IPO \cite{azar2024general} & $x^2$ & $2x^2 + 2\gamma^2$ & $\log \frac{\pi_\theta(y|x)}{\pi_{ref}(y|x)}$\\
         SLiC \cite{zhao2023slic} & $max(0, -x)$ & $|x| - \gamma$ & $ \log\pi_{\theta}(y|x)$\\
         - & $e^{-x}$ & $2e^{-\gamma}cosh(|x|)$ & - \\
         - & $(max(0, -x))^2$ & $x^2 - \gamma^2$ & - \\
         \bottomrule
    \end{tabular}
    
    \label{tab:SLA_variants}
\end{table}

In Section~\ref{sec:method_app}, we apply \textsc{SoLoPO} to several mainstream preference optimization (PO) algorithms, including DPO, SimPO, and ORPO. More generally, \textsc{SoLoPO} is compatible with any PO algorithm for which there exists an upper-bound function $s(x)$ of its convergence function $f(x)$ such that 
\begin{equation}
    f(x+\gamma) + f(-x+\gamma) \leq s(x).
\end{equation}
Table~\ref{tab:SLA_variants} lists each PO algorithm with its corresponding $f(x)$ and $s(x)$. It is worth noting that if the inequality is tight, the performance would be further enhanced theoretically. If $s(x)$ is perfectly tight, \textit{i.e.}  $f(x+\gamma) + f(-x+\gamma) = s(x)$, then $s(x)$ should satisfy the following properties:
\begin{itemize}
    \item $s(x)$ is an even function
    \item $\forall x, \; s(x) \ge 2 \cdot f(\gamma) = s(0)$
\end{itemize}

We subsequently present the complete theoretical objectives for all considered PO algorithms.

\textbf{DPO Setting:} 
\begin{align*}
    \mathcal{L}^{DPO}_{\ours} &= - \mathbb{E}_{\substack{x \sim \mathcal{D}_{x_{short}} ; \\ y_w, y_l \sim \mathcal{D}_{y}; y_w\succ y_l} }[ \log \sigma(3\cdot [ \beta\log \frac{\pi_\theta(y_w|x)}{\pi_{ref}(y_w|x)} - \beta\log \frac{\pi_\theta(y_l|x)}{\pi_{ref}(y_l|x)} - \gamma])] \\
    &\qquad + 3 \cdot \mathbb{E}_{x_{\cdot} \sim \mathcal{D}_{x_\cdot}; y \sim \mathcal{D}_{y}}|\beta\log \frac{\pi_\theta(y|x_{short})}{\pi_{ref}(y|x_{short})} - \beta\log \frac{\pi_\theta(y|x_{long})}{\pi_{ref}(y|x_{long})}|
\end{align*}

\textbf{SimPO Setting:} 
\begin{align*}
    \mathcal{L}^{SimPO}_{\ours} &= - \mathbb{E}_{\substack{x \sim \mathcal{D}_{x_{short}} ; \\ y_w, y_l \sim \mathcal{D}_{y}; y_w\succ y_l} }[ \log \sigma(3\cdot [ \frac{\beta}{|y_w|} \log\pi_{\theta}(y_w|x) - \frac{\beta}{|y_l|} \log\pi_{\theta}(y_l|x) - \gamma])] \\
    &\qquad + 3 \cdot \mathbb{E}_{x_{\cdot} \sim \mathcal{D}_{x_\cdot}; y \sim \mathcal{D}_{y}}|\frac{\beta}{|y|} \log\pi_{\theta}(y|x_{short}) - \frac{\beta}{|y|} \log\pi_{\theta}(y|x_{long})| \\
\end{align*}


\textbf{ORPO Setting:}

To maintain consistency with the vanilla ORPO, we add the conventional causal language modeling negative log-likelihood (NLL) loss.
\begin{align*}
    \mathcal{L}^{ORPO}_{\ours} &= - \mathbb{E}_{\substack{x \sim \mathcal{D}_{x_{short}} ; \\ y_w, y_l \sim \mathcal{D}_{y}; y_w\succ y_l} }[ \log \sigma(3\cdot [ \log \frac{\pi_\theta(y_w|x)}{1 - \pi_\theta(y_w|x)} - \log \frac{\pi_\theta(y_l|x)}{1 - \pi_\theta(y_l|x)} - \gamma])] \\
    &\qquad + 3 \cdot \mathbb{E}_{x_{\cdot} \sim \mathcal{D}_{x_\cdot}; y \sim \mathcal{D}_{y}}| \log \frac{\pi_\theta(y|x_{short})}{1 - \pi_\theta(y|x_{short})} - \log \frac{\pi_\theta(y|x_{long})}{1 - \pi_\theta(y|x_{long})}| \\ 
    &\qquad + \mathbb{E}_{\substack{x \sim \mathcal{D}_{x_{short}} ; \\ y_w \sim \mathcal{D}_{y}} }\mathcal{L}_{NLL}(\pi_\theta;x;y_{w})
\end{align*}

\textbf{IPO Setting:}


\begin{align*}
    \mathcal{L}^{IPO}_{\ours} &= \mathbb{E}_{\substack{x \sim \mathcal{D}_{x_{short}} ; \\ y_w, y_l \sim \mathcal{D}_{y}; y_w\succ y_l} }[3\cdot [ \log \frac{\pi_\theta(y_w|x)}{\pi_{ref}(y_w|x)} - \log \frac{\pi_\theta(y_l|x)}{\pi_{ref}(y_l|x)} - \gamma]]^2 \\
    &\qquad + 18 \cdot \mathbb{E}_{x_{\cdot} \sim \mathcal{D}_{x_\cdot}; y \sim \mathcal{D}_{y}}[ \log \frac{\pi_\theta(y|x_{short})}{\pi_{ref}(y|x_{short})} - \log \frac{\pi_\theta(y|x_{long})}{\pi_{ref}(y|x_{long})}]^2
\end{align*}



\textbf{SLiC Setting:} 
\begin{align*}
    \mathcal{L}^{SLiC}_{\ours} &= \mathbb{E}_{\substack{x \sim \mathcal{D}_{x_{short}} ; \\ y_w, y_l \sim \mathcal{D}_{y}; y_w\succ y_l} }[ \max(0, - \log\pi_{\theta}(y_w|x) + \log\pi_{\theta}(y_l|x) + \gamma)] \\
    &\qquad + \mathbb{E}_{x_{\cdot} \sim \mathcal{D}_{x_\cdot}; y \sim \mathcal{D}_{y}} | \log\pi_{\theta}(y|x_{short}) - \log\pi_{\theta}(y|x_{long})|
\end{align*}

%% file: iclr2026_conference.bib
@inproceedings{bai-etal-2024-longbench,
    title = "{L}ong{B}ench: A Bilingual, Multitask Benchmark for Long Context Understanding",
    author = "Bai, Yushi  and
      Lv, Xin  and
      Zhang, Jiajie  and
      Lyu, Hongchang  and
      Tang, Jiankai  and
      Huang, Zhidian  and
      Du, Zhengxiao  and
      Liu, Xiao  and
      Zeng, Aohan  and
      Hou, Lei  and
      Dong, Yuxiao  and
      Tang, Jie  and
      Li, Juanzi",
    editor = "Ku, Lun-Wei  and
      Martins, Andre  and
      Srikumar, Vivek",
    booktitle = "Proceedings of the 62nd Annual Meeting of the Association for Computational Linguistics (Volume 1: Long Papers)",
    month = aug,
    year = "2024",
    address = "Bangkok, Thailand",
    publisher = "Association for Computational Linguistics",
    url = "https://aclanthology.org/2024.acl-long.172/",
    doi = "10.18653/v1/2024.acl-long.172",
    pages = "3119--3137"
}

@article{cot_matters,
  author       = {Dawei Zhu and
                  Xiyu Wei and
                  Guangxiang Zhao and
                  Wenhao Wu and
                  Haosheng Zou and
                  Junfeng Ran and
                  Xun Wang and
                  Lin Sun and
                  Xiangzheng Zhang and
                  Sujian Li},
  title        = {Chain-of-Thought Matters: Improving Long-Context Language Models with
                  Reasoning Path Supervision},
  journal      = {CoRR},
  volume       = {abs/2502.20790},
  year         = {2025},
  url          = {https://doi.org/10.48550/arXiv.2502.20790},
  doi          = {10.48550/ARXIV.2502.20790},
  eprinttype    = {arXiv},
  eprint       = {2502.20790},
  bibsource    = {dblp computer science bibliography, https://dblp.org}
}

@inproceedings{
hsieh2024ruler,
title={{RULER}: What{\textquoteright}s the Real Context Size of Your Long-Context Language Models?},
author={Cheng-Ping Hsieh and Simeng Sun and Samuel Kriman and Shantanu Acharya and Dima Rekesh and Fei Jia and Boris Ginsburg},
booktitle={First Conference on Language Modeling},
year={2024},
url={https://openreview.net/forum?id=kIoBbc76Sy}
}

@inproceedings{qasper-dataset,
    title = "A Dataset of Information-Seeking Questions and Answers Anchored in Research Papers",
    author = "Dasigi, Pradeep  and
      Lo, Kyle  and
      Beltagy, Iz  and
      Cohan, Arman  and
      Smith, Noah A.  and
      Gardner, Matt",
    editor = "Toutanova, Kristina  and
      Rumshisky, Anna  and
      Zettlemoyer, Luke  and
      Hakkani-Tur, Dilek  and
      Beltagy, Iz  and
      Bethard, Steven  and
      Cotterell, Ryan  and
      Chakraborty, Tanmoy  and
      Zhou, Yichao",
    booktitle = "Proceedings of the 2021 Conference of the North American Chapter of the Association for Computational Linguistics: Human Language Technologies",
    month = jun,
    year = "2021",
    address = "Online",
    publisher = "Association for Computational Linguistics",
    url = "https://aclanthology.org/2021.naacl-main.365/",
    doi = "10.18653/v1/2021.naacl-main.365",
    pages = "4599--4610"
}

@inproceedings{hotpotqa-dataset,
    title = "{H}otpot{QA}: A Dataset for Diverse, Explainable Multi-hop Question Answering",
    author = "Yang, Zhilin  and
      Qi, Peng  and
      Zhang, Saizheng  and
      Bengio, Yoshua  and
      Cohen, William  and
      Salakhutdinov, Ruslan  and
      Manning, Christopher D.",
    editor = "Riloff, Ellen  and
      Chiang, David  and
      Hockenmaier, Julia  and
      Tsujii, Jun{'}ichi",
    booktitle = "Proceedings of the 2018 Conference on Empirical Methods in Natural Language Processing",
    month = oct # "-" # nov,
    year = "2018",
    address = "Brussels, Belgium",
    publisher = "Association for Computational Linguistics",
    url = "https://aclanthology.org/D18-1259/",
    doi = "10.18653/v1/D18-1259",
    pages = "2369--2380"
}

@inproceedings{2wikimqa_dataset,
    title = "Constructing A Multi-hop {QA} Dataset for Comprehensive Evaluation of Reasoning Steps",
    author = "Ho, Xanh  and
      Duong Nguyen, Anh-Khoa  and
      Sugawara, Saku  and
      Aizawa, Akiko",
    editor = "Scott, Donia  and
      Bel, Nuria  and
      Zong, Chengqing",
    booktitle = "Proceedings of the 28th International Conference on Computational Linguistics",
    month = dec,
    year = "2020",
    address = "Barcelona, Spain (Online)",
    publisher = "International Committee on Computational Linguistics",
    url = "https://aclanthology.org/2020.coling-main.580/",
    doi = "10.18653/v1/2020.coling-main.580",
    pages = "6609--6625"
}

@article{musique-dataset,
    title = "MuSiQue: Multihop Questions via Single-hop Question Composition",
    author = "Trivedi, Harsh  and
      Balasubramanian, Niranjan  and
      Khot, Tushar  and
      Sabharwal, Ashish",
    editor = "Roark, Brian  and
      Nenkova, Ani",
    journal = "Transactions of the Association for Computational Linguistics",
    volume = "10",
    year = "2022",
    address = "Cambridge, MA",
    publisher = "MIT Press",
    url = "https://aclanthology.org/2022.tacl-1.31/",
    doi = "10.1162/tacl_a_00475",
    pages = "539--554"
}

@inproceedings{squad-dataset,
    title = "{SQ}u{AD}: 100,000+ Questions for Machine Comprehension of Text",
    author = "Rajpurkar, Pranav  and
      Zhang, Jian  and
      Lopyrev, Konstantin  and
      Liang, Percy",
    editor = "Su, Jian  and
      Duh, Kevin  and
      Carreras, Xavier",
    booktitle = "Proceedings of the 2016 Conference on Empirical Methods in Natural Language Processing",
    month = nov,
    year = "2016",
    address = "Austin, Texas",
    publisher = "Association for Computational Linguistics",
    url = "https://aclanthology.org/D16-1264/",
    doi = "10.18653/v1/D16-1264",
    pages = "2383--2392"
}

@article{narrativeqa-dataset,
    title = "The {N}arrative{QA} Reading Comprehension Challenge",
    author = "Ko{\v{c}}isk{\'y}, Tom{\'a}{\v{s}}  and
      Schwarz, Jonathan  and
      Blunsom, Phil  and
      Dyer, Chris  and
      Hermann, Karl Moritz  and
      Melis, G{\'a}bor  and
      Grefenstette, Edward",
    editor = "Lee, Lillian  and
      Johnson, Mark  and
      Toutanova, Kristina  and
      Roark, Brian",
    journal = "Transactions of the Association for Computational Linguistics",
    volume = "6",
    year = "2018",
    address = "Cambridge, MA",
    publisher = "MIT Press",
    url = "https://aclanthology.org/Q18-1023/",
    doi = "10.1162/tacl_a_00023",
    pages = "317--328"
}

@misc{ifeval,
      title={Instruction-Following Evaluation for Large Language Models}, 
      author={Jeffrey Zhou and Tianjian Lu and Swaroop Mishra and Siddhartha Brahma and Sujoy Basu and Yi Luan and Denny Zhou and Le Hou},
      year={2023},
      eprint={2311.07911},
      archivePrefix={arXiv},
      primaryClass={cs.CL},
      url={https://arxiv.org/abs/2311.07911}, 
}

@inproceedings{
wei2022chain_of_thought,
title={Chain of Thought Prompting Elicits Reasoning in Large Language Models},
author={Jason Wei and Xuezhi Wang and Dale Schuurmans and Maarten Bosma and brian ichter and Fei Xia and Ed H. Chi and Quoc V Le and Denny Zhou},
booktitle={Advances in Neural Information Processing Systems},
editor={Alice H. Oh and Alekh Agarwal and Danielle Belgrave and Kyunghyun Cho},
year={2022},
url={https://openreview.net/forum?id=_VjQlMeSB_J}
}

@article{Survey_Dong2023ASO,
  title={A Survey on Long Text Modeling with Transformers},
  author={Zican Dong and Tianyi Tang and Lunyi Li and Wayne Xin Zhao},
  journal={ArXiv},
  year={2023},
  volume={abs/2302.14502},
  url={https://api.semanticscholar.org/CorpusID:257232619}
}

@inproceedings{Survey_Liu2025ACS,
  title={A Comprehensive Survey on Long Context Language Modeling},
  author={Jiaheng Liu and Dawei Zhu and Zhiqi Bai and Yancheng He and Huanxuan Liao and Haoran Que and Zekun Moore Wang and Chenchen Zhang and Ge Zhang and Jiebin Zhang and Yuanxing Zhang and Zhuo Chen and Hangyu Guo and Shilong Li and Ziqiang Liu and Yong Shan and Yifan Song and Jiayi Tian and Wenhao Wu and Zhejian Zhou and Ruijie Zhu and Junlan Feng and Yang Gao and Shizhu He and Zhoujun Li and Tianyu Liu and Fanyu Meng and Wenbo Su and Ying Tan and Zili Wang and Jian Yang and Wei Ye and Bo Zheng and Wangchunshu Zhou and Wenhao Huang and Sujian Li and Zhaoxiang Zhang},
  year={2025},
  url={https://api.semanticscholar.org/CorpusID:277271533}
}

@article{LLama_3,
  title={The Llama 3 Herd of Models},
  author={Abhimanyu Dubey and Abhinav Jauhri and Abhinav Pandey and Abhishek Kadian and Ahmad Al-Dahle, et al.},
  journal={ArXiv},
  year={2024},
  volume={abs/2407.21783},
  url={https://api.semanticscholar.org/CorpusID:271571434}
}

@article{Zhu2025GeneralizingFromShort2Long,
  title={Generalizing From Short to Long: Effective Data Synthesis for Long-Context Instruction Tuning},
  author={Wenhao Zhu and Pinzhen Chen and Hanxu Hu and Shujian Huang and Fei Yuan and Jiajun Chen and Alexandra Birch},
  journal={ArXiv},
  year={2025},
  volume={abs/2502.15592},
  url={https://api.semanticscholar.org/CorpusID:276557686}
}

@inproceedings{BART_Lewis2019,
  title={BART: Denoising Sequence-to-Sequence Pre-training for Natural Language Generation, Translation, and Comprehension},
  author={Mike Lewis and Yinhan Liu and Naman Goyal and Marjan Ghazvininejad and Abdel-rahman Mohamed and Omer Levy and Veselin Stoyanov and Luke Zettlemoyer},
  booktitle={Annual Meeting of the Association for Computational Linguistics},
  year={2019},
  url={https://api.semanticscholar.org/CorpusID:204960716}
}

@inproceedings{LongCEloss,
title={What is Wrong with Perplexity for Long-context Language Modeling?},
author={Lizhe Fang and Yifei Wang and Zhaoyang Liu and Chenheng Zhang and Stefanie Jegelka and Jinyang Gao and Bolin Ding and Yisen Wang},
booktitle={The Thirteenth International Conference on Learning Representations},
year={2025},
url={https://openreview.net/forum?id=fL4qWkSmtM}
}

@article{SFT,
  publtype={informal},
  author={Jason Wei and Maarten Bosma and Vincent Y. Zhao and Kelvin Guu and Adams Wei Yu and Brian Lester and Nan Du and Andrew M. Dai and Quoc V. Le},
  title={Finetuned Language Models Are Zero-Shot Learners},
  year={2021},
  cdate={1609459200000},
  journal={CoRR},
  volume={abs/2109.01652},
  url={https://arxiv.org/abs/2109.01652}
}

@article{Rafailov2023DirectPO,
  title={Direct Preference Optimization: Your Language Model is Secretly a Reward Model},
  author={Rafael Rafailov and Archit Sharma and Eric Mitchell and Stefano Ermon and Christopher D. Manning and Chelsea Finn},
  journal={ArXiv},
  year={2023},
  volume={abs/2305.18290},
  url={https://api.semanticscholar.org/CorpusID:258959321}
}

@inproceedings{hong-etal-2024-orpo,
    title = "{ORPO}: Monolithic Preference Optimization without Reference Model",
    author = "Hong, Jiwoo  and
      Lee, Noah  and
      Thorne, James",
    editor = "Al-Onaizan, Yaser  and
      Bansal, Mohit  and
      Chen, Yun-Nung",
    booktitle = "Proceedings of the 2024 Conference on Empirical Methods in Natural Language Processing",
    month = nov,
    year = "2024",
    address = "Miami, Florida, USA",
    publisher = "Association for Computational Linguistics",
    url = "https://aclanthology.org/2024.emnlp-main.626/",
    doi = "10.18653/v1/2024.emnlp-main.626",
    pages = "11170--11189"
}

@inproceedings{
meng2024simpo,
title={Sim{PO}: Simple Preference Optimization with a Reference-Free Reward},
author={Yu Meng and Mengzhou Xia and Danqi Chen},
booktitle={The Thirty-eighth Annual Conference on Neural Information Processing Systems},
year={2024},
url={https://openreview.net/forum?id=3Tzcot1LKb}
}

@inproceedings{
chen2025longpo,
title={Long{PO}: Long Context Self-Evolution of Large Language Models through Short-to-Long Preference Optimization},
author={Guanzheng Chen and Xin Li and Michael Shieh and Lidong Bing},
booktitle={The Thirteenth International Conference on Learning Representations},
year={2025},
url={https://openreview.net/forum?id=qTrEq31Shm}
}

@article{Yang2024Qwen25TR,
  title={Qwen2.5 Technical Report},
  author={Qwen An Yang and Baosong Yang and Beichen Zhang and Binyuan Hui and Bo Zheng and Bowen Yu and Chengyuan Li and Dayiheng Liu and Fei Huang and Guanting Dong and Haoran Wei and Huan Lin and Jian Yang and Jianhong Tu and Jianwei Zhang and Jianxin Yang and Jiaxin Yang and Jingren Zhou and Junyang Lin and Kai Dang and Keming Lu and Keqin Bao and Kexin Yang and Le Yu and Mei Li and Mingfeng Xue and Pei Zhang and Qin Zhu and Rui Men and Runji Lin and Tianhao Li and Tingyu Xia and Xingzhang Ren and Xuancheng Ren and Yang Fan and Yang Su and Yi-Chao Zhang and Yunyang Wan and Yuqi Liu and Zeyu Cui and Zhenru Zhang and Zihan Qiu and Shanghaoran Quan and Zekun Wang},
  journal={ArXiv},
  year={2024},
  volume={abs/2412.15115},
  url={https://api.semanticscholar.org/CorpusID:274859421}
}

@article{Yang2024MindLLMLL,
  title={MindLLM: Lightweight large language model pre-training, evaluation and domain application},
  author={Yizhe Yang and Huashan Sun and Jiawei Li and Runheng Liu and Yinghao Li and Yuhang Liu and Yang Gao and Heyan Huang},
  journal={AI Open},
  year={2024},
  volume={5},
  pages={1-26},
  url={https://api.semanticscholar.org/CorpusID:271818589}
}

@article{li-etal-2024-fundamental,
author = {Li, Jiawei and Gao, Yang and Yang, Yizhe and Bai, Yu and Zhou, Xiaofeng and Li, Yinghao and Sun, Huashan and Liu, Yuhang and Si, Xingpeng and Ye, Yuhao and Wu, Yixiao and Lin, Yiguan and Xu, Bin and Ren, Bowen and Feng, Chong and Huang, Heyan},
title = {Fundamental Capabilities and Applications of Large Language Models: A Survey},
year = {2025},
issue_date = {January 2026},
publisher = {Association for Computing Machinery},
address = {New York, NY, USA},
volume = {58},
number = {2},
issn = {0360-0300},
url = {https://doi.org/10.1145/3735632},
doi = {10.1145/3735632},
journal = {ACM Comput. Surv.},
month = sep,
articleno = {38},
numpages = {42}
}

@misc{jiang2023mistral7b,
      title={Mistral 7B}, 
      author={Albert Q. Jiang and Alexandre Sablayrolles and Arthur Mensch and Chris Bamford and Devendra Singh Chaplot and Diego de las Casas and Florian Bressand and Gianna Lengyel and Guillaume Lample and Lucile Saulnier and Lélio Renard Lavaud and Marie-Anne Lachaux and Pierre Stock and Teven Le Scao and Thibaut Lavril and Thomas Wang and Timothée Lacroix and William El Sayed},
      year={2023},
      eprint={2310.06825},
      archivePrefix={arXiv},
      primaryClass={cs.CL},
      url={https://arxiv.org/abs/2310.06825}, 
}

@inproceedings{
gpo_paper,
title={Generalized Preference Optimization: A Unified Approach to Offline Alignment},
author={Yunhao Tang and Zhaohan Daniel Guo and Zeyu Zheng and Daniele Calandriello and Remi Munos and Mark Rowland and Pierre Harvey Richemond and Michal Valko and Bernardo Avila Pires and Bilal Piot},
booktitle={Forty-first International Conference on Machine Learning},
year={2024},
url={https://openreview.net/forum?id=gu3nacA9AH}
}

@inproceedings{
rlhf,
title={Training language models to follow instructions with human feedback},
author={Long Ouyang and Jeffrey Wu and Xu Jiang and Diogo Almeida and Carroll Wainwright and Pamela Mishkin and Chong Zhang and Sandhini Agarwal and Katarina Slama and Alex Gray and John Schulman and Jacob Hilton and Fraser Kelton and Luke Miller and Maddie Simens and Amanda Askell and Peter Welinder and Paul Christiano and Jan Leike and Ryan Lowe},
booktitle={Advances in Neural Information Processing Systems},
editor={Alice H. Oh and Alekh Agarwal and Danielle Belgrave and Kyunghyun Cho},
year={2022},
url={https://openreview.net/forum?id=TG8KACxEON}
}

@inproceedings{Sloane1951PredictionAE,
  title={Prediction and Entropy of Printed English},
  author={Neil J. A. Sloane and Aaron D. Wyner},
  year={1951},
  url={https://api.semanticscholar.org/CorpusID:9101213}
}

@inproceedings{Wit2013WhatIL,
  title={What is Linguistic Redundancy},
  author={Ernst C. Wit and Marie Gillette},
  year={2013},
  url={https://api.semanticscholar.org/CorpusID:1425655}
}

@inproceedings{
li2024snapkv,
title={Snap{KV}: {LLM} Knows What You are Looking for Before Generation},
author={Yuhong Li and Yingbing Huang and Bowen Yang and Bharat Venkitesh and Acyr Locatelli and Hanchen Ye and Tianle Cai and Patrick Lewis and Deming Chen},
booktitle={The Thirty-eighth Annual Conference on Neural Information Processing Systems},
year={2024},
url={https://openreview.net/forum?id=poE54GOq2l}
}

@inproceedings{huang-etal-2024-fewer,
    title = "Fewer is More: Boosting Math Reasoning with Reinforced Context Pruning",
    author = "Huang, Xijie  and
      Zhang, Li Lyna  and
      Cheng, Kwang-Ting  and
      Yang, Fan  and
      Yang, Mao",
    editor = "Al-Onaizan, Yaser  and
      Bansal, Mohit  and
      Chen, Yun-Nung",
    booktitle = "Proceedings of the 2024 Conference on Empirical Methods in Natural Language Processing",
    month = nov,
    year = "2024",
    address = "Miami, Florida, USA",
    publisher = "Association for Computational Linguistics",
    url = "https://aclanthology.org/2024.emnlp-main.758/",
    doi = "10.18653/v1/2024.emnlp-main.758",
    pages = "13674--13695"
}

@inproceedings{pan-etal-2024-llmlingua,
    title = "{LLML}ingua-2: Data Distillation for Efficient and Faithful Task-Agnostic Prompt Compression",
    author = {Pan, Zhuoshi  and
      Wu, Qianhui  and
      Jiang, Huiqiang  and
      Xia, Menglin  and
      Luo, Xufang  and
      Zhang, Jue  and
      Lin, Qingwei  and
      R{\"u}hle, Victor  and
      Yang, Yuqing  and
      Lin, Chin-Yew  and
      Zhao, H. Vicky  and
      Qiu, Lili  and
      Zhang, Dongmei},
    editor = "Ku, Lun-Wei  and
      Martins, Andre  and
      Srikumar, Vivek",
    booktitle = "Findings of the Association for Computational Linguistics: ACL 2024",
    month = aug,
    year = "2024",
    address = "Bangkok, Thailand",
    publisher = "Association for Computational Linguistics",
    url = "https://aclanthology.org/2024.findings-acl.57/",
    doi = "10.18653/v1/2024.findings-acl.57",
    pages = "963--981"
}

@inproceedings{li-etal-2023-compressing,
    title = "Compressing Context to Enhance Inference Efficiency of Large Language Models",
    author = "Li, Yucheng  and
      Dong, Bo  and
      Guerin, Frank  and
      Lin, Chenghua",
    editor = "Bouamor, Houda  and
      Pino, Juan  and
      Bali, Kalika",
    booktitle = "Proceedings of the 2023 Conference on Empirical Methods in Natural Language Processing",
    month = dec,
    year = "2023",
    address = "Singapore",
    publisher = "Association for Computational Linguistics",
    url = "https://aclanthology.org/2023.emnlp-main.391/",
    doi = "10.18653/v1/2023.emnlp-main.391",
    pages = "6342--6353"
}

@inproceedings{
xu2024recomp,
title={{RECOMP}: Improving Retrieval-Augmented {LM}s with Context Compression and Selective Augmentation},
author={Fangyuan Xu and Weijia Shi and Eunsol Choi},
booktitle={The Twelfth International Conference on Learning Representations},
year={2024},
url={https://openreview.net/forum?id=mlJLVigNHp}
}

@inproceedings{bai-etal-2024-citrus,
    title = "{CI}tru{S}: Chunked Instruction-aware State Eviction for Long Sequence Modeling",
    author = "Bai, Yu  and
      Zou, Xiyuan  and
      Huang, Heyan  and
      Chen, Sanxing  and
      Rondeau, Marc-Antoine  and
      Gao, Yang  and
      Cheung, Jackie CK",
    editor = "Al-Onaizan, Yaser  and
      Bansal, Mohit  and
      Chen, Yun-Nung",
    booktitle = "Proceedings of the 2024 Conference on Empirical Methods in Natural Language Processing",
    month = nov,
    year = "2024",
    address = "Miami, Florida, USA",
    publisher = "Association for Computational Linguistics",
    url = "https://aclanthology.org/2024.emnlp-main.338/",
    doi = "10.18653/v1/2024.emnlp-main.338",
    pages = "5908--5930"
}

@inproceedings{compression_rate,
title = "Unifying Cross-lingual Summarization and Machine Translation with Compression Rate",
author = "Yu Bai and Heyan Huang and Kai Fan and Yang Gao and Yiming Zhu and Jiaao Zhan and Zewen Chi and Boxing Chen",
note = "Publisher Copyright: {\textcopyright} 2022 ACM.; 45th Annual International ACM SIGIR Conference on Research and Development in Information Retrieval, SIGIR 2022 ; Conference date: 11-07-2022 Through 15-07-2022",
year = "2022",
month = jul,
day = "6",
doi = "10.1145/3477495.3532071",
language = "English",
series = "SIGIR 2022 - Proceedings of the 45th International ACM SIGIR Conference on Research and Development in Information Retrieval",
publisher = "Association for Computing Machinery, Inc",
pages = "1087--1097",
booktitle = "SIGIR 2022 - Proceedings of the 45th International ACM SIGIR Conference on Research and Development in Information Retrieval",
}

@techreport{long_doc_mt_survey,
  TITLE = {{Handling Very Long Contexts in Neural Machine Translation: a Survey}},
  AUTHOR = {Peng, Ziqian and Bawden, Rachel and Yvon, Fran{\c c}ois},
  URL = {https://inria.hal.science/hal-04652584},
  NUMBER = {Livrable D3-2.1},
  PAGES = {50},
  INSTITUTION = {{Projet ANR MaTOS}},
  YEAR = {2024},
  MONTH = Jun,
  HAL_ID = {hal-04652584},
  HAL_VERSION = {v1},
}

@inproceedings{long_doc_mt,
    title = "Improving Long Context Document-Level Machine Translation",
    author = "Herold, Christian  and
      Ney, Hermann",
    editor = "Strube, Michael  and
      Braud, Chloe  and
      Hardmeier, Christian  and
      Li, Junyi Jessy  and
      Loaiciga, Sharid  and
      Zeldes, Amir",
    booktitle = "Proceedings of the 4th Workshop on Computational Approaches to Discourse (CODI 2023)",
    month = jul,
    year = "2023",
    address = "Toronto, Canada",
    publisher = "Association for Computational Linguistics",
    url = "https://aclanthology.org/2023.codi-1.15/",
    doi = "10.18653/v1/2023.codi-1.15",
    pages = "112--125"
}

@inproceedings{
cream,
title={An Efficient Recipe for Long Context Extension via Middle-Focused Positional Encoding},
author={Tong Wu and Yanpeng Zhao and Zilong Zheng},
booktitle={The Thirty-eighth Annual Conference on Neural Information Processing Systems},
year={2024},
url={https://openreview.net/forum?id=aNHEqFMS0N}
}

@inproceedings{
zhu2024pose,
title={Po{SE}: Efficient Context Window Extension of {LLM}s via Positional Skip-wise Training},
author={Dawei Zhu and Nan Yang and Liang Wang and Yifan Song and Wenhao Wu and Furu Wei and Sujian Li},
booktitle={The Twelfth International Conference on Learning Representations},
year={2024},
url={https://openreview.net/forum?id=3Z1gxuAQrA}
}

@inproceedings{
peng2024yarn,
title={Ya{RN}: Efficient Context Window Extension of Large Language Models},
author={Bowen Peng and Jeffrey Quesnelle and Honglu Fan and Enrico Shippole},
booktitle={The Twelfth International Conference on Learning Representations},
year={2024},
url={https://openreview.net/forum?id=wHBfxhZu1u}
}

@inproceedings{BiPE,
author = {He, Zhenyu and Feng, Guhao and Luo, Shengjie and Yang, Kai and Wang, Liwei and Xu, Jingjing and Zhang, Zhi and Yang, Hongxia and He, Di},
title = {Two stones hit one bird: bilevel positional encoding for better length extrapolation},
year = {2024},
publisher = {JMLR.org},
booktitle = {Proceedings of the 41st International Conference on Machine Learning},
articleno = {715},
numpages = {19},
location = {Vienna, Austria},
series = {ICML'24}
}

@inproceedings{belyi-etal-2025-luna,
    title = "{L}una: A Lightweight Evaluation Model to Catch Language Model Hallucinations with High Accuracy and Low Cost",
    author = "Belyi, Masha  and
      Friel, Robert  and
      Shao, Shuai  and
      Sanyal, Atindriyo",
    editor = "Rambow, Owen  and
      Wanner, Leo  and
      Apidianaki, Marianna  and
      Al-Khalifa, Hend  and
      Eugenio, Barbara Di  and
      Schockaert, Steven  and
      Darwish, Kareem  and
      Agarwal, Apoorv",
    booktitle = "Proceedings of the 31st International Conference on Computational Linguistics: Industry Track",
    month = jan,
    year = "2025",
    address = "Abu Dhabi, UAE",
    publisher = "Association for Computational Linguistics",
    url = "https://aclanthology.org/2025.coling-industry.34/",
    pages = "398--409"
}

@misc{
zhang2024longcite,
title={LongCite: Enabling {LLM}s to Generate Fine-grained Citations in Long-context {QA}},
author={Jiajie Zhang and Yushi Bai and Xin Lv and Wanjun Gu and Danqing Liu and Minhao Zou and Shulin Cao and Lei Hou and Yuxiao Dong and Ling Feng and Juanzi Li},
year={2024},
url={https://openreview.net/forum?id=mMXdHyBcHh}
}

@inproceedings{
gu2022efficiently,
title={Efficiently Modeling Long Sequences with Structured State Spaces},
author={Albert Gu and Karan Goel and Christopher Re},
booktitle={International Conference on Learning Representations},
year={2022},
url={https://openreview.net/forum?id=uYLFoz1vlAC}
}

@inproceedings{
gu2024mamba,
title={Mamba: Linear-Time Sequence Modeling with Selective State Spaces},
author={Albert Gu and Tri Dao},
booktitle={First Conference on Language Modeling},
year={2024},
url={https://openreview.net/forum?id=tEYskw1VY2}
}

@misc{lu2025mobamixtureblockattention,
      title={MoBA: Mixture of Block Attention for Long-Context LLMs}, 
      author={Enzhe Lu and Zhejun Jiang and Jingyuan Liu and Yulun Du and Tao Jiang and Chao Hong and Shaowei Liu and Weiran He and Enming Yuan and Yuzhi Wang and Zhiqi Huang and Huan Yuan and Suting Xu and Xinran Xu and Guokun Lai and Yanru Chen and Huabin Zheng and Junjie Yan and Jianlin Su and Yuxin Wu and Neo Y. Zhang and Zhilin Yang and Xinyu Zhou and Mingxing Zhang and Jiezhong Qiu},
      year={2025},
      eprint={2502.13189},
      archivePrefix={arXiv},
      primaryClass={cs.LG},
      url={https://arxiv.org/abs/2502.13189}, 
}

@inproceedings{
chen2024longlora,
title={LongLo{RA}: Efficient Fine-tuning of Long-Context Large Language Models},
author={Yukang Chen and Shengju Qian and Haotian Tang and Xin Lai and Zhijian Liu and Song Han and Jiaya Jia},
booktitle={The Twelfth International Conference on Learning Representations},
year={2024},
url={https://openreview.net/forum?id=6PmJoRfdaK}
}

@inproceedings{DE_128k,
  author       = {Yao Fu and
                  Rameswar Panda and
                  Xinyao Niu and
                  Xiang Yue and
                  Hannaneh Hajishirzi and
                  Yoon Kim and
                  Hao Peng},
  title        = {Data Engineering for Scaling Language Models to 128K Context},
  booktitle    = {Forty-first International Conference on Machine Learning, {ICML} 2024,
                  Vienna, Austria, July 21-27, 2024},
  publisher    = {OpenReview.net},
  year         = {2024},
  url          = {https://openreview.net/forum?id=TaAqeo7lUh},
  bibsource    = {dblp computer science bibliography, https://dblp.org}
}

@article{Extending_llama3,
  author       = {Peitian Zhang and
                  Ninglu Shao and
                  Zheng Liu and
                  Shitao Xiao and
                  Hongjin Qian and
                  Qiwei Ye and
                  Zhicheng Dou},
  title        = {Extending Llama-3's Context Ten-Fold Overnight},
  journal      = {CoRR},
  volume       = {abs/2404.19553},
  year         = {2024},
  url          = {https://doi.org/10.48550/arXiv.2404.19553},
  doi          = {10.48550/ARXIV.2404.19553},
  eprinttype    = {arXiv},
  eprint       = {2404.19553},
  bibsource    = {dblp computer science bibliography, https://dblp.org}
}

@misc{open-llm-leaderboard-v2,
  author = {Clémentine Fourrier and Nathan Habib and Alina Lozovskaya and Konrad Szafer and Thomas Wolf},
  title = {Open LLM Leaderboard v2},
  year = {2024},
  publisher = {Hugging Face},
  howpublished = "\url{https://huggingface.co/spaces/open-llm-leaderboard/open_llm_leaderboard}",
}

@misc{suzgun2022challengingbigbenchtaskschainofthought,
  title={Challenging BIG-Bench Tasks and Whether Chain-of-Thought Can Solve Them},
  author={Mirac Suzgun and Nathan Scales and Nathanael Schärli and Sebastian Gehrmann and Yi Tay and Hyung Won Chung and Aakanksha Chowdhery and Quoc V. Le and Ed H. Chi and Denny Zhou and Jason Wei},
  year={2022},
  eprint={2210.09261},
  archivePrefix={arXiv},
  primaryClass={cs.CL},
  url={https://arxiv.org/abs/2210.09261},
}

@misc{hendrycks2021measuringmathematicalproblemsolving,
  title={Measuring Mathematical Problem Solving With the MATH Dataset},
  author={Dan Hendrycks and Collin Burns and Saurav Kadavath and Akul Arora and Steven Basart and Eric Tang and Dawn Song and Jacob Steinhardt},
  year={2021},
  eprint={2103.03874},
  archivePrefix={arXiv},
  primaryClass={cs.LG},
  url={https://arxiv.org/abs/2103.03874},
}

@misc{rein2023gpqagraduatelevelgoogleproofqa,
  title={GPQA: A Graduate-Level Google-Proof Q\&A Benchmark},
  author={David Rein and Betty Li Hou and Asa Cooper Stickland and Jackson Petty and Richard Yuanzhe Pang and Julien Dirani and Julian Michael and Samuel R. Bowman},
  year={2023},
  eprint={2311.12022},
  archivePrefix={arXiv},
  primaryClass={cs.AI},
  url={https://arxiv.org/abs/2311.12022},
}

@misc{wang2024mmluprorobustchallengingmultitask,
  title={MMLU-Pro: A More Robust and Challenging Multi-Task Language Understanding Benchmark},
  author={Yubo Wang and Xueguang Ma and Ge Zhang and Yuansheng Ni and Abhranil Chandra and Shiguang Guo and Weiming Ren and Aaran Arulraj and Xuan He and Ziyan Jiang and Tianle Li and Max Ku and Kai Wang and Alex Zhuang and Rongqi Fan and Xiang Yue and Wenhu Chen},
  year={2024},
  eprint={2406.01574},
  archivePrefix={arXiv},
  primaryClass={cs.CL},
  url={https://arxiv.org/abs/2406.01574},
}

@article{MDCure,
  author       = {Gabrielle Kaili{-}May Liu and
                  Bowen Shi and
                  Avi Caciularu and
                  Idan Szpektor and
                  Arman Cohan},
  title        = {MDCure: {A} Scalable Pipeline for Multi-Document Instruction-Following},
  journal      = {CoRR},
  volume       = {abs/2410.23463},
  year         = {2024},
  url          = {https://doi.org/10.48550/arXiv.2410.23463},
  doi          = {10.48550/ARXIV.2410.23463},
  eprinttype    = {arXiv},
  eprint       = {2410.23463},
  bibsource    = {dblp computer science bibliography, https://dblp.org}
}

@article{LongReward,
  author       = {Jiajie Zhang and
                  Zhongni Hou and
                  Xin Lv and
                  Shulin Cao and
                  Zhenyu Hou and
                  Yilin Niu and
                  Lei Hou and
                  Yuxiao Dong and
                  Ling Feng and
                  Juanzi Li},
  title        = {LongReward: Improving Long-context Large Language Models with {AI}
                  Feedback},
  journal      = {CoRR},
  volume       = {abs/2410.21252},
  year         = {2024},
  url          = {https://doi.org/10.48550/arXiv.2410.21252},
  doi          = {10.48550/ARXIV.2410.21252},
  eprinttype    = {arXiv},
  eprint       = {2410.21252},
  bibsource    = {dblp computer science bibliography, https://dblp.org}
}

@article{LongFaith,
  author       = {Cehao Yang and
                  Xueyuan Lin and
                  Chengjin Xu and
                  Xuhui Jiang and
                  Shengjie Ma and
                  Aofan Liu and
                  Hui Xiong and
                  Jian Guo},
  title        = {LongFaith: Enhancing Long-Context Reasoning in LLMs with Faithful
                  Synthetic Data},
  journal      = {CoRR},
  volume       = {abs/2502.12583},
  year         = {2025},
  url          = {https://doi.org/10.48550/arXiv.2502.12583},
  doi          = {10.48550/ARXIV.2502.12583},
  eprinttype    = {arXiv},
  eprint       = {2502.12583},
  bibsource    = {dblp computer science bibliography, https://dblp.org}
}

@misc{Generalizing_Short2Long,
      title={Generalizing From Short to Long: Effective Data Synthesis for Long-Context Instruction Tuning}, 
      author={Wenhao Zhu and Pinzhen Chen and Hanxu Hu and Shujian Huang and Fei Yuan and Jiajun Chen and Alexandra Birch},
      year={2025},
      eprint={2502.15592},
      archivePrefix={arXiv},
      primaryClass={cs.CL},
      url={https://arxiv.org/abs/2502.15592}, 
}

@article{logo,
  author       = {Zecheng Tang and
                  Zechen Sun and
                  Juntao Li and
                  Qiaoming Zhu and
                  Min Zhang},
  title        = {{LOGO} - Long cOntext aliGnment via efficient preference Optimization},
  journal      = {CoRR},
  volume       = {abs/2410.18533},
  year         = {2024},
  url          = {https://doi.org/10.48550/arXiv.2410.18533},
  doi          = {10.48550/ARXIV.2410.18533},
  eprinttype    = {arXiv},
  eprint       = {2410.18533},
  bibsource    = {dblp computer science bibliography, https://dblp.org}
}

@inproceedings{bai-etal-2024-longalign,
    title = "{L}ong{A}lign: A Recipe for Long Context Alignment of Large Language Models",
    author = "Bai, Yushi  and
      Lv, Xin  and
      Zhang, Jiajie  and
      He, Yuze  and
      Qi, Ji  and
      Hou, Lei  and
      Tang, Jie  and
      Dong, Yuxiao  and
      Li, Juanzi",
    editor = "Al-Onaizan, Yaser  and
      Bansal, Mohit  and
      Chen, Yun-Nung",
    booktitle = "Findings of the Association for Computational Linguistics: EMNLP 2024",
    month = nov,
    year = "2024",
    address = "Miami, Florida, USA",
    publisher = "Association for Computational Linguistics",
    url = "https://aclanthology.org/2024.findings-emnlp.74/",
    doi = "10.18653/v1/2024.findings-emnlp.74",
    pages = "1376--1395"
}

@inproceedings{zhao-etal-2024-longagent,
    title = "{LONGAGENT}: Achieving Question Answering for 128k-Token-Long Documents through Multi-Agent Collaboration",
    author = "Zhao, Jun  and
      Zu, Can  and
      Hao, Xu  and
      Lu, Yi  and
      He, Wei  and
      Ding, Yiwen  and
      Gui, Tao  and
      Zhang, Qi  and
      Huang, Xuanjing",
    editor = "Al-Onaizan, Yaser  and
      Bansal, Mohit  and
      Chen, Yun-Nung",
    booktitle = "Proceedings of the 2024 Conference on Empirical Methods in Natural Language Processing",
    month = nov,
    year = "2024",
    address = "Miami, Florida, USA",
    publisher = "Association for Computational Linguistics",
    url = "https://aclanthology.org/2024.emnlp-main.912/",
    doi = "10.18653/v1/2024.emnlp-main.912",
    pages = "16310--16324"
}

@inproceedings{BABILong,
 author = {Kuratov, Yuri and Bulatov, Aydar and Anokhin, Petr and Rodkin, Ivan and Sorokin, Dmitry and Sorokin, Artyom and Burtsev, Mikhail},
 booktitle = {Advances in Neural Information Processing Systems},
 editor = {A. Globerson and L. Mackey and D. Belgrave and A. Fan and U. Paquet and J. Tomczak and C. Zhang},
 pages = {106519--106554},
 publisher = {Curran Associates, Inc.},
 title = {BABILong: Testing the Limits of LLMs with Long Context Reasoning-in-a-Haystack},
 url = {https://proceedings.neurips.cc/paper_files/paper/2024/file/c0d62e70dbc659cc9bd44cbcf1cb652f-Paper-Datasets_and_Benchmarks_Track.pdf},
 volume = {37},
 year = {2024}
}

@misc{bai2025longbenchv2deeperunderstanding,
      title={LongBench v2: Towards Deeper Understanding and Reasoning on Realistic Long-context Multitasks}, 
      author={Yushi Bai and Shangqing Tu and Jiajie Zhang and Hao Peng and Xiaozhi Wang and Xin Lv and Shulin Cao and Jiazheng Xu and Lei Hou and Yuxiao Dong and Jie Tang and Juanzi Li},
      year={2025},
      eprint={2412.15204},
      archivePrefix={arXiv},
      primaryClass={cs.CL},
      url={https://arxiv.org/abs/2412.15204}, 
}

@misc{
gao2025how,
title={How to Train Long-Context Language Models (Effectively)},
author={Tianyu Gao and Alexander Wettig and Howard Yen and Danqi Chen},
year={2025},
url={https://openreview.net/forum?id=nwZHFKrYTB}
}

@inproceedings{coc,
    title = "Making Long-Context Language Models Better Multi-Hop Reasoners",
    author = "Li, Yanyang  and
      Liang, Shuo  and
      Lyu, Michael  and
      Wang, Liwei",
    editor = "Ku, Lun-Wei  and
      Martins, Andre  and
      Srikumar, Vivek",
    booktitle = "Proceedings of the 62nd Annual Meeting of the Association for Computational Linguistics (Volume 1: Long Papers)",
    month = aug,
    year = "2024",
    address = "Bangkok, Thailand",
    publisher = "Association for Computational Linguistics",
    url = "https://aclanthology.org/2024.acl-long.135/",
    doi = "10.18653/v1/2024.acl-long.135",
    pages = "2462--2475"
}

@misc{dao2023flashattention2fasterattentionbetter,
      title={FlashAttention-2: Faster Attention with Better Parallelism and Work Partitioning}, 
      author={Tri Dao},
      year={2023},
      eprint={2307.08691},
      archivePrefix={arXiv},
      primaryClass={cs.LG},
      url={https://arxiv.org/abs/2307.08691}, 
}

@misc{rajbhandari2020zeromemoryoptimizationstraining,
      title={ZeRO: Memory Optimizations Toward Training Trillion Parameter Models}, 
      author={Samyam Rajbhandari and Jeff Rasley and Olatunji Ruwase and Yuxiong He},
      year={2020},
      eprint={1910.02054},
      archivePrefix={arXiv},
      primaryClass={cs.LG},
      url={https://arxiv.org/abs/1910.02054}, 
}

@inproceedings{
adamw,
title={Decoupled Weight Decay Regularization},
author={Ilya Loshchilov and Frank Hutter},
booktitle={International Conference on Learning Representations},
year={2019},
url={https://openreview.net/forum?id=Bkg6RiCqY7},
}

@inproceedings{azar2024general,
  title={A general theoretical paradigm to understand learning from human preferences},
  author={Azar, Mohammad Gheshlaghi and Guo, Zhaohan Daniel and Piot, Bilal and Munos, Remi and Rowland, Mark and Valko, Michal and Calandriello, Daniele},
  booktitle={International Conference on Artificial Intelligence and Statistics},
  pages={4447--4455},
  year={2024},
  organization={PMLR}
}

@article{zhao2023slic,
  title={Slic-hf: Sequence likelihood calibration with human feedback},
  author={Zhao, Yao and Joshi, Rishabh and Liu, Tianqi and Khalman, Misha and Saleh, Mohammad and Liu, Peter J},
  journal={arXiv preprint arXiv:2305.10425},
  year={2023}
}

@misc{longred,
      title={LongReD: Mitigating Short-Text Degradation of Long-Context Large Language Models via Restoration Distillation}, 
      author={Zican Dong and Junyi Li and Jinhao Jiang and Mingyu Xu and Wayne Xin Zhao and Bingning Wang and Weipeng Chen},
      year={2025},
      eprint={2502.07365},
      archivePrefix={arXiv},
      primaryClass={cs.CL},
      url={https://arxiv.org/abs/2502.07365}, 
}

@inproceedings{short_forgetting,
    title = "Effective Long-Context Scaling of Foundation Models",
    author = "Xiong, Wenhan  and
      Liu, Jingyu  and
      Molybog, Igor  and
      Zhang, Hejia  and
      Bhargava, Prajjwal  and
      Hou, Rui  and
      Martin, Louis  and
      Rungta, Rashi  and
      Sankararaman, Karthik Abinav  and
      Oguz, Barlas  and
      Khabsa, Madian  and
      Fang, Han  and
      Mehdad, Yashar  and
      Narang, Sharan  and
      Malik, Kshitiz  and
      Fan, Angela  and
      Bhosale, Shruti  and
      Edunov, Sergey  and
      Lewis, Mike  and
      Wang, Sinong  and
      Ma, Hao",
    editor = "Duh, Kevin  and
      Gomez, Helena  and
      Bethard, Steven",
    booktitle = "Proceedings of the 2024 Conference of the North American Chapter of the Association for Computational Linguistics: Human Language Technologies (Volume 1: Long Papers)",
    month = jun,
    year = "2024",
    address = "Mexico City, Mexico",
    publisher = "Association for Computational Linguistics",
    url = "https://aclanthology.org/2024.naacl-long.260/",
    doi = "10.18653/v1/2024.naacl-long.260",
    pages = "4643--4663"
}

@inproceedings{zheng-etal-2024-llamafactory,
    title = "{L}lama{F}actory: Unified Efficient Fine-Tuning of 100+ Language Models",
    author = "Zheng, Yaowei  and
      Zhang, Richong  and
      Zhang, Junhao  and
      Ye, Yanhan  and
      Luo, Zheyan",
    editor = "Cao, Yixin  and
      Feng, Yang  and
      Xiong, Deyi",
    booktitle = "Proceedings of the 62nd Annual Meeting of the Association for Computational Linguistics (Volume 3: System Demonstrations)",
    month = aug,
    year = "2024",
    address = "Bangkok, Thailand",
    publisher = "Association for Computational Linguistics",
    url = "https://aclanthology.org/2024.acl-demos.38/",
    doi = "10.18653/v1/2024.acl-demos.38",
    pages = "400--410"
}

@misc{jacobs2023deepspeedulyssesoptimizationsenabling,
      title={DeepSpeed Ulysses: System Optimizations for Enabling Training of Extreme Long Sequence Transformer Models}, 
      author={Sam Ade Jacobs and Masahiro Tanaka and Chengming Zhang and Minjia Zhang and Shuaiwen Leon Song and Samyam Rajbhandari and Yuxiong He},
      year={2023},
      eprint={2309.14509},
      archivePrefix={arXiv},
      primaryClass={cs.LG},
      url={https://arxiv.org/abs/2309.14509}, 
}

@inproceedings{
qin2024infobatch,
title={InfoBatch: Lossless Training Speed Up by Unbiased Dynamic Data Pruning},
author={Ziheng Qin and Kai Wang and Zangwei Zheng and Jianyang Gu and Xiangyu Peng and xu Zhao Pan and Daquan Zhou and Lei Shang and Baigui Sun and Xuansong Xie and Yang You},
booktitle={The Twelfth International Conference on Learning Representations},
year={2024},
url={https://openreview.net/forum?id=C61sk5LsK6}
}

@misc{marion2023moreinvestigatingdatapruning,
      title={When Less is More: Investigating Data Pruning for Pretraining LLMs at Scale}, 
      author={Max Marion and Ahmet Üstün and Luiza Pozzobon and Alex Wang and Marzieh Fadaee and Sara Hooker},
      year={2023},
      eprint={2309.04564},
      archivePrefix={arXiv},
      primaryClass={cs.CL},
      url={https://arxiv.org/abs/2309.04564}, 
}

@article{BT_ranking_loss,
 ISSN = {00063444, 14643510},
 URL = {http://www.jstor.org/stable/2334029},
 author = {Ralph Allan Bradley and Milton E. Terry},
 journal = {Biometrika},
 number = {3/4},
 pages = {324--345},
 publisher = {[Oxford University Press, Biometrika Trust]},
 title = {Rank Analysis of Incomplete Block Designs: I. The Method of Paired Comparisons},
 urldate = {2025-05-06},
 volume = {39},
 year = {1952}
}

@inproceedings{
bai2025longwriter,
title={LongWriter: Unleashing 10,000+ Word Generation from Long Context {LLM}s},
author={Yushi Bai and Jiajie Zhang and Xin Lv and Linzhi Zheng and Siqi Zhu and Lei Hou and Yuxiao Dong and Jie Tang and Juanzi Li},
booktitle={The Thirteenth International Conference on Learning Representations},
year={2025},
url={https://openreview.net/forum?id=kQ5s9Yh0WI}
}

@article{context_align,
  publtype={informal},
  author={Baolong Bi and Shaohan Huang and Yiwei Wang and Tianchi Yang and Zihan Zhang and Haizhen Huang and Lingrui Mei and Junfeng Fang and Zehao Li and Furu Wei and Weiwei Deng and Feng Sun and Qi Zhang and Shenghua Liu},
  title={Context-DPO: Aligning Language Models for Context-Faithfulness},
  year={2024},
  cdate={1704067200000},
  journal={CoRR},
  volume={abs/2412.15280},
  url={https://doi.org/10.48550/arXiv.2412.15280}
}

@misc{IOPO,
      title={IOPO: Empowering LLMs with Complex Instruction Following via Input-Output Preference Optimization}, 
      author={Xinghua Zhang and Haiyang Yu and Cheng Fu and Fei Huang and Yongbin Li},
      year={2024},
      eprint={2411.06208},
      archivePrefix={arXiv},
      primaryClass={cs.CL},
      url={https://arxiv.org/abs/2411.06208}, 
}

@inproceedings{loong,
    title = "Leave No Document Behind: Benchmarking Long-Context {LLM}s with Extended Multi-Doc {QA}",
    author = "Wang, Minzheng  and
      Chen, Longze  and
      Cheng, Fu  and
      Liao, Shengyi  and
      Zhang, Xinghua  and
      Wu, Bingli  and
      Yu, Haiyang  and
      Xu, Nan  and
      Zhang, Lei  and
      Luo, Run  and
      Li, Yunshui  and
      Yang, Min  and
      Huang, Fei  and
      Li, Yongbin",
    editor = "Al-Onaizan, Yaser  and
      Bansal, Mohit  and
      Chen, Yun-Nung",
    booktitle = "Proceedings of the 2024 Conference on Empirical Methods in Natural Language Processing",
    month = nov,
    year = "2024",
    address = "Miami, Florida, USA",
    publisher = "Association for Computational Linguistics",
    url = "https://aclanthology.org/2024.emnlp-main.322/",
    doi = "10.18653/v1/2024.emnlp-main.322",
    pages = "5627--5646"
}

@article{xu2024large,
  title={Large language models for generative information extraction: A survey},
  author={Xu, Derong and Chen, Wei and Peng, Wenjun and Zhang, Chao and Xu, Tong and Zhao, Xiangyu and Wu, Xian and Zheng, Yefeng and Wang, Yang and Chen, Enhong},
  journal={Frontiers of Computer Science},
  volume={18},
  number={6},
  pages={186357},
  year={2024},
  publisher={Springer}
}

@misc{wiki:kl_divergence,
  author       = "{Wikipedia contributors}",
  title        = "{Kullback--Leibler divergence}",
  year         = "2024",
  howpublished = "\url{https://en.wikipedia.org/wiki/Kullback%E2%80%93Leibler_divergence#cite_note-Csiszar-1}",
  note         = "Accessed: 2024-06-10"
}

@misc{li2025pspoeffectiveprocesssupervisedpolicy,
      title={PSPO*: An Effective Process-supervised Policy Optimization for Reasoning Alignment}, 
      author={Jiawei Li and Xinyue Liang and Junlong Zhang and Yizhe Yang and Chong Feng and Yang Gao},
      year={2025},
      eprint={2411.11681},
      archivePrefix={arXiv},
      primaryClass={cs.AI},
      url={https://arxiv.org/abs/2411.11681}, 
}

@inproceedings{gao-etal-2025-train,
    title = "How to Train Long-Context Language Models (Effectively)",
    author = "Gao, Tianyu  and
      Wettig, Alexander  and
      Yen, Howard  and
      Chen, Danqi",
    editor = "Che, Wanxiang  and
      Nabende, Joyce  and
      Shutova, Ekaterina  and
      Pilehvar, Mohammad Taher",
    booktitle = "Proceedings of the 63rd Annual Meeting of the Association for Computational Linguistics (Volume 1: Long Papers)",
    month = jul,
    year = "2025",
    address = "Vienna, Austria",
    publisher = "Association for Computational Linguistics",
    url = "https://aclanthology.org/2025.acl-long.366/",
    doi = "10.18653/v1/2025.acl-long.366",
    pages = "7376--7399",
    ISBN = "979-8-89176-251-0"
}

@article{From_Short_to_Long,
  publtype={informal},
  author={Wenhao Zhu and Pinzhen Chen and Hanxu Hu and Shujian Huang and Fei Yuan and Jiajun Chen and Alexandra Birch},
  title={Generalizing From Short to Long: Effective Data Synthesis for Long-Context Instruction Tuning},
  year={2025},
  month={February},
  cdate={1738368000000},
  journal={CoRR},
  volume={abs/2502.15592},
  url={https://doi.org/10.48550/arXiv.2502.15592}
}

@inproceedings{
xiao2024efficient,
title={Efficient Streaming Language Models with Attention Sinks},
author={Guangxuan Xiao and Yuandong Tian and Beidi Chen and Song Han and Mike Lewis},
booktitle={The Twelfth International Conference on Learning Representations},
year={2024},
url={https://openreview.net/forum?id=NG7sS51zVF}
}

@inproceedings{
zhang2025long,
title={Long Context Compression with Activation Beacon},
author={Peitian Zhang and Zheng Liu and Shitao Xiao and Ninglu Shao and Qiwei Ye and Zhicheng Dou},
booktitle={The Thirteenth International Conference on Learning Representations},
year={2025},
url={https://openreview.net/forum?id=1eQT9OzfNQ}
}

@inproceedings{
gurung2025learning,
title={Learning to Reason for Long-Form Story Generation},
author={Alexander Gurung and Mirella Lapata},
booktitle={Second Conference on Language Modeling},
year={2025},
url={https://openreview.net/forum?id=dr3eg5ehR2}
}

@inproceedings{tang-etal-2025-unlocking,
    title = "Unlocking General Long Chain-of-Thought Reasoning Capabilities of Large Language Models via Representation Engineering",
    author = "Tang, Xinyu  and
      Wang, Xiaolei  and
      Lv, Zhihao  and
      Min, Yingqian  and
      Zhao, Xin  and
      Hu, Binbin  and
      Liu, Ziqi  and
      Zhang, Zhiqiang",
    editor = "Che, Wanxiang  and
      Nabende, Joyce  and
      Shutova, Ekaterina  and
      Pilehvar, Mohammad Taher",
    booktitle = "Proceedings of the 63rd Annual Meeting of the Association for Computational Linguistics (Volume 1: Long Papers)",
    month = jul,
    year = "2025",
    address = "Vienna, Austria",
    publisher = "Association for Computational Linguistics",
    url = "https://aclanthology.org/2025.acl-long.339/",
    doi = "10.18653/v1/2025.acl-long.339",
    pages = "6832--6849",
    ISBN = "979-8-89176-251-0"
}

@misc{shao2024deepseekmathpushinglimitsmathematical,
      title={DeepSeekMath: Pushing the Limits of Mathematical Reasoning in Open Language Models}, 
      author={Zhihong Shao and Peiyi Wang and Qihao Zhu and Runxin Xu and Junxiao Song and Xiao Bi and Haowei Zhang and Mingchuan Zhang and Y. K. Li and Y. Wu and Daya Guo},
      year={2024},
      eprint={2402.03300},
      archivePrefix={arXiv},
      primaryClass={cs.CL},
      url={https://arxiv.org/abs/2402.03300}, 
}
